%% file: neurips_2020.tex
\documentclass{article}

    \usepackage[final, nonatbib]{neurips_2020}

\usepackage[utf8]{inputenc} 
\usepackage[T1]{fontenc}    
\usepackage{hyperref}       
\usepackage{url}            
\usepackage{booktabs}       
\usepackage{amsfonts}       
\usepackage{nicefrac}       
\usepackage{microtype}      

\usepackage{amsmath}
\usepackage{amssymb}
\usepackage{mathtools}
\usepackage[authoryear, round]{natbib} 
\usepackage{stmaryrd}
\usepackage{tgtermes}

\title{Analysis and Design of Thompson Sampling \\ for Stochastic Partial Monitoring}

\input{tex_src/preamble_common.tex}

\author{
  Taira Tsuchiya \\
  The University of Tokyo \\
  RIKEN AIP \\
  \texttt{tsuchiya@ms.k.u-tokyo.ac.jp} \\
  \And
  Junya Honda \\
  The University of Tokyo \\
  RIKEN AIP \\
  \texttt{honda@edu.k.u-tokyo.ac.jp} \\
  \And
  Masashi Sugiyama \\
  RIKEN AIP \\
  The University of Tokyo \\
  \texttt{sugi@k.u-tokyo.ac.jp} 
}

\begin{document}

\maketitle
\input{tex_src/abstract.tex}
\input{tex_src/introduction.tex}
\input{tex_src/problem_setup.tex}
\input{tex_src/tspm.tex}

\input{tex_src/theoretical_analysis.tex}
\input{tex_src/experiments.tex}
\input{tex_src/conclusion.tex}

\input{tex_src/broader_impact.tex}

\input{tex_src/acknowledgements.tex}


\input{neurips_2020.bbl}
\bibliographystyle{plainnat}

\newpage
\input{tex_src/appendix.tex}

\end{document}

%% file: tex_src/preamble_common.tex
\usepackage[ruled,vlined,linesnumbered]{algorithm2e}
\usepackage{amsmath}
\usepackage{amssymb}
\usepackage{amsthm}
\usepackage{bbm}
\usepackage{bm}
\usepackage[font=footnotesize,labelfont=bf]{caption}
\usepackage{eucal}  
\usepackage{mathrsfs}
\usepackage{tabularx}

\hypersetup{
  colorlinks=true,
  linkcolor=blue,
  citecolor=blue,
}

\newtheorem{theorem}{Theorem}
\newtheorem*{theorem*}{Theorem}
\newtheorem*{proof*}{proof}

\newtheorem*{assumption*}{Assumption}
\newtheorem{lemma}[theorem]{Lemma}
\newtheorem*{lemma*}{Lemma}
\newtheorem{proposition}[theorem]{Proposition}
\newtheorem*{proposition*}{Proposition}
\newtheorem{corollary}[theorem]{Corollary}
\newtheorem*{corollary*}{Corollary}
\newtheorem{fact}[theorem]{Fact}
\newtheorem*{fact*}{Fact}

\theoremstyle{definition} 
\newtheorem{definition}{Definition}
\newtheorem*{definition*}{Definition}

\theoremstyle{remark}  
\newtheorem*{remark}{Remark}

\SetArgSty{textnormal}
\SetKwInOut{Input}{Input}
\SetKwInOut{Output}{Output}
\SetKwComment{Comment}{$\triangleright$\ }{}

\def\e{\mathrm{e}}
\def\E{\mathbb{E}}
\def\P{\mathbb{P}}
\def\R{\mathbb{R}}
\def\N{\mathbb{N}}

\def\calA{\mathcal{A}}
\def\calB{\mathcal{B}}
\def\calC{\mathcal{C}}

\def\calF{\mathcal{F}}
\def\calG{\mathcal{G}}

\def\calN{\mathcal{N}}

\def\calP{\mathcal{P}}
\def\calS{\mathcal{S}}
\def\calT{\mathcal{T}}
\def\calU{\mathcal{U}}
\def\calV{\mathcal{V}}

\def\calX{\mathcal{X}}
\def\calY{\mathcal{Y}}

\newcommand{\one}{\mbox{(i)}}
\newcommand{\two}{\mbox{(i\hspace{-.1em}i)}}

\DeclareMathOperator*{\argmax}{arg\,max}
\DeclareMathOperator*{\argmin}{arg\,min}

\newcommand{\ep}{\epsilon}

\newcommand{\de}{\delta}
\newcommand{\la}{\lambda}

\newcommand{\ind}[1]{\mathbbm{1}{\left[#1\right]}}
\newcommand{\pr}[1]{\P{\left\{#1\right\}}}
\newcommand{\pipr}[1]{\pi{\left(#1\right)}}  
\newcommand{\ev}[1]{{\left\{#1\right\}}}  
\newcommand{\expect}{\operatornamewithlimits{\mathbb{E}}}
\newcommand{\Expect}[1]{\expect{\left[#1\right]}}
\newcommand{\ExpectX}[2]{\E_{#1}{\left[#2\right]}}  
\newcommand{\Ex}[1]{\mathbb{E}\left[#1\right]}  

\newcommand{\relmiddle}[1]{\mathrel{}\middle#1\mathrel{}}  
\newcommand{\KL}[2]{\mathcal{D}_{\mathrm{KL}}\left({#1} \middle\| {#2}\right)}

\newcommand{\diff}[2]{\frac{\mathrm{d}{#1}}{\mathrm{d}{#2}}}

\renewcommand{\d}{\mathrm{d}}
\newcommand{\dx}{\d{x}}

\newcommand{\By}[1]{\quad \left(\text{by {#1}} \right) }
\newcommand{\Bym}[1]{\quad \left(\text{by } {#1} \right) }  
\newcommand{\Since}[1]{\quad \left(\text{{#1}} \right) }
\newcommand{\Sincem}[1]{\quad \left( {#1} \right) }  
\newcommand{\since}[1]{\quad\left(\mbox{#1}\right)}
\newcommand{\nn}{\nonumber\\}
\newcommand{\n}{\nonumber}
\newcommand{\per}{\,.}
\newcommand{\com}{\,,}

\def\iid{\stackrel{\mathrm{i.i.d.}}{\sim}}

\DeclarePairedDelimiter{\abs}{\lvert}{\rvert} %

\DeclarePairedDelimiter{\crl}{\{}{\}}
\DeclarePairedDelimiter{\prn}{(}{)}
\DeclarePairedDelimiter{\nrm}{\|}{\|}

\DeclareMathOperator{\Order}{O}
\DeclareMathOperator{\tilOrder}{\tilde{O}}
\DeclareMathOperator{\order}{o}
\DeclareMathOperator{\tilTheta}{\tilde{\Theta}}

\newcommand{\normx}[1]{\left\Vert#1\right\Vert}
\newcommand{\pax}[1]{\left(#1\right)}

\newcommand{\brx}[1]{\left\{#1\right\}}

\newcommand{\Regret}{\mathrm{Reg}}

\newcommand{\lossmat}{\mathbf{L}}
\newcommand{\fbmat}{\mathbf{H}}

\newcommand{\unif}[1]{\calU({#1})}

\newcommand{\psub}{{p^{(\alpha)}}}
\newcommand{\hsub}{{h^{(\alpha)}}}
\newcommand{\bsub}{{b^{(\alpha)}}}

\newcommand{\onemat}{\mathbf{1}}

\newcommand{\sumT}{\sum_{t=1}^T}

\newcommand{\qit}{q_i^{(t)}}
\newcommand{\qkt}{q_k^{(t)}}
\newcommand{\qin}{q_{i,n}}
\newcommand{\qini}{q_{i,n_i}}
\newcommand{\qknk}{q_{k,n_k}}

\newcommand{\phat}{\hat{p}}
\newcommand{\phatt}{\hat{p}_t}
\newcommand{\pbar}{\bar{p}}
\newcommand{\pact}{\bar{p}_t}  
\newcommand{\ptilt}{\tilde{p}_t}

\newcommand{\Ci}{\calC_i}
\newcommand{\Cj}{\calC_j}
\newcommand{\Ck}{\calC_k}
\newcommand{\Nplus}[2]{N_{{#1}, {#2}}^+}

\newcommand{\Ai}{\calA_i(t)}

\newcommand{\Aicomp}{\calA_i^c(t)}
\newcommand{\Aini}{\calA_{i,n_i}}  
\newcommand{\Atili}{\tilde{\calA}_i(t)}
\newcommand{\Atilone}{\tilde{\calA}_1(t)}
\newcommand{\Atilicomp}{\tilde{\calA}_i^c(t)}
\newcommand{\Atilonecomp}{\tilde{\calA}_1^c(t)}
\newcommand{\Aaci}{\bar{\calA}_i(t)}
\newcommand{\Aacicomp}{\bar{\calA}_i^c(t)}

\newcommand{\Aack}{\bar{\calA}_k(t)}
\newcommand{\Aackcomp}{\bar{\calA}_k^c(t)}

\newcommand{\Znk}{Z_{n_k}}

\DeclareMathOperator{\Img}{Im}

\usepackage{subfigure}
\newlength{\subfigwidth}
\newlength{\subfigcolsep}
\newcommand{\ie}{\textit{i.e., }}
\newcommand{\eg}{\textit{e.g., }}

\newcommand{\idmat}{I}  

\newcommand{\gbart}{\bar{g}_t}

\newcommand{\Zbsli}{Z_{\backslash i}}

\newcommand{\BPM}{\mathrm{BPM}}
\newcommand{\BtBPM}{{B_t^{\BPM}}}
\newcommand{\btBPM}{{b_t^{\BPM}}}

\newcommand{\x}{,\,}

\newcommand{\eppr}{\epsilon^\prime}

\newcommand{\pjk}{p^{(j,k)}}
\newcommand{\Lx}{L^{(x)}}

%% file: tex_src/abstract.tex
\begin{abstract}
We investigate finite stochastic partial monitoring,
which is a general model for sequential learning with limited feedback.
While Thompson sampling is one of the most promising algorithms on a variety of online decision-making problems,
its properties for stochastic partial monitoring have not been theoretically investigated,
and
the existing algorithm relies on a heuristic approximation of the posterior distribution.
To mitigate these problems,
we present a novel Thompson-sampling-based algorithm,
which enables us to exactly sample the target parameter from the posterior distribution.
Besides,
we prove that the new algorithm achieves the logarithmic \emph{problem-dependent expected pseudo-regret} $\Order(\log T)$ 
for a linearized variant of the problem with local observability.
This result is the first regret bound of Thompson sampling for partial monitoring, which also becomes the first logarithmic regret bound of Thompson sampling for linear bandits.
\end{abstract}

%% file: tex_src/introduction.tex
\section{Introduction}
Partial monitoring (PM) is a general sequential decision-making problem with limited feedback~\citep{Rustichini99general, Piccolboni01FeedExp3}.
PM is attracting broad interest because
it includes a wide range of problems such as
the multi-armed bandit problem~\citep{LaiRobbins85},
a linear optimization problem with full or bandit feedback~\citep{Zinkevich03linearfull, Dani08linearbandit},
dynamic pricing~\citep{Kleinberg03dp},
and label efficient prediction~\citep{CesaBianchi04lep}.

A PM game can be seen as a sequential game that is played by two players: a learner and an opponent.
At every round, the learner chooses an action, while the opponent chooses an outcome.
Then, the learner suffers an unobserved loss and receives a feedback symbol, 
both of which are determined from the selected action and outcome.
The main characteristic of this game is that the learner cannot directly observe the outcome and loss.
The goal of the learner is to minimize his/her cumulative loss over all rounds.
The performance of the learner is evaluated by the regret,
which is defined as the difference between the cumulative losses of the learner and the optimal action 
(\ie the action whose expected loss is the smallest).

There are mainly two types of PM games, which are the \emph{stochastic} and \emph{adversarial} settings~\citep{Piccolboni01FeedExp3, Bartok11minimax}.
In the stochastic setting, the outcome at each round is determined from the \emph{opponent's strategy},
which is a probability vector over the opponent's possible choices.
On the other hand, in the adversarial setting, the outcomes are arbitrarily decided by the opponent.
We refer to the PM game with finite actions and finite outcomes as a \emph{finite} PM game.
In this paper, we focus on the stochastic finite game.

One of the first algorithms for PM was considered by~\citet{Piccolboni01FeedExp3}. 
They proposed the FeedExp3 algorithm, the key idea of which is to use an unbiased estimator of the losses.
They showed that the FeedExp3 algorithm attains $\tilOrder(T^{3/4})$ minimax regret for a certain class of PM games, and showed that
any algorithm suffers linear minimax regret $\Omega(T)$ for the other class.
Here $T$ is the time horizon and the notation $\tilOrder(\cdot)$ hides polylogarithmic factors.
The upper bound $\tilOrder(T^{3/4})$ is later improved by~\citet{CesaBianchi06regret} to $\Order(T^{2/3})$,
and they also provided a game with a matching lower bound.

In the seminal paper by~\citet{Bartok11minimax}, they classified PM games into four classes based on their minimax regrets.
To be more specific, they classified games into trivial, easy, hard, and hopeless games,
where their minimax regrets are $0$, $\tilTheta(\sqrt{T})$, $\Theta(T^{2/3})$, and $\Theta(T)$, respectively.
Note that the easy game is also called a~\emph{locally observable} game.
After their work,
several algorithms have been proposed for the finite PM problem~\citep{Bartok12CBP, Vanchinathan14BPM, Komiyama15PMDEMD}.
For the problem-dependent regret analysis, \citet{Komiyama15PMDEMD} proposed an algorithm
that achieves $\Order(\log T)$ regret with the optimal constant factor.
However, it requires to solve a time-consuming optimization problem with infinitely many constraints at each round.
In addition, this algorithm relies on the forced exploration to achieve the optimality, which makes the empirical performance near-optimal only after prohibitively many rounds, say, $10^5$ or $10^6$.

Thompson sampling (TS,~\citealp{Thompson1933likelihood}) is one of the most promising algorithms on a variety of online decision-making problems
such as the multi-armed bandit~\citep{LaiRobbins85} and the linear bandit~\citep{Agrawal13payoff},
and the effectiveness of TS has been investigated both empirically~\citep{Chapelle11empirical} and theoretically~\citep{Kaufmann12finitetime, Agrawal13further, Honda14prior}.
In the literature on PM,
\citet{Vanchinathan14BPM} proposed a TS-based algorithm called BPM-TS (Bayes-update for PM based on TS) for stochastic PM,
which empirically achieved state-of-the-art performance.
Their algorithm uses Gaussian approximation to handle the complicated posterior distribution of the opponent's strategy.
However, this approximation is somewhat heuristic and
can degrade the empirical performance due to the discrepancy from the exact posterior distribution.
Furthermore, no theoretical guarantee is provided for BPM-TS. 

Our goals are to establish a new TS-based algorithm for stochastic PM,
which allows us to sample the opponent's strategy parameter from the exact posterior distribution,
and investigate whether the TS-based algorithm can achive sub-linear regret in stochastic PM.
Using the accept-reject sampling,
we propose a new TS-based algorithm for PM (TSPM),
which is equipped with a numerical scheme to obtain a posterior sample from the complicated posterior distribution.
We derive a logarithmic regret upper bound $\Order(\log{T})$ for the proposed algorithm on the locally observable game under a linearized variant of the problem.
This is the first regret bound for TS on the locally observable game.
Moreover, our setting includes the linear bandit problem and our result is also the first logarithmic expected regret bound of TS for the linear bandit, whereas a high-probability bound was provided, for example, in~\citet{Agrawal13payoff}.
Finally, we compare the performance of TSPM with existing algorithms in numerical experiments,
and show that TSPM outperforms existing algorithms.

%% file: tex_src/problem_setup.tex
\section{Preliminaries}

This paper studies finite stochastic PM games~\citep{Bartok11minimax}.
A PM game with $N$ actions and $M$ outcomes is defined by
a pair of a loss matrix $\lossmat = (\ell_{i,j}) \in \R^{N \times M}$ and feedback matrix $\fbmat = (h_{i,j}) \in [A]^{N \times M}$,
where $A$ is the number of feedback symbols
and $[A] = \ev{1,2,\ldots,A}$.

A PM game can be seen as a sequential game that is played by two players: the learner and the opponent.
At each round $t = 1,2,\dots,T$, the learner selects action $i(t) \in [N]$,
and at the same time the opponent selects an outcome based on the opponent's strategy $p^* \in \calP_M$,
where $\calP_n = \ev{p\in\R^n : p_k \geq 0, \sum_{k=1}^n p_k = 1}$ is the $(n-1)$-dimensional probability simplex.
The outcome $j(t)$ of each round is an independent and identically distributed sample from $p^*$,
and then, the learner suffers loss $\ell_{i(t), j(t)}$ at time $t$.
The learner cannot directly observe the value of this loss,
but instead observes the \emph{feedback symbol} $y(t)=h_{i(t),j(t)} \in [A]$.
The setting explained above has been widely studied in the literature of stochastic PM~\citep{Bartok11minimax, Komiyama15PMDEMD},
and we call this the \emph{discrete} setting.
In Section~\ref{sec:theory}, we also introduce a \emph{linear} setting for theoretical analysis, 
which is a slightly different setting from the discrete one.

The learner aims to minimize the cumulative loss 
over $T$ rounds.
The expected loss of action $i$ is given by $L_i^\top p^*$,
where $L_i$ is the $i$-th column of $\lossmat^\top$.
We say action $i$ is~\emph{optimal} under strategy $p^*$
if $(L_i - L_j)^\top p^* \leq 0$ for any $j \neq i$.
We assume that the optimal action is unique, and
without loss of generality that the optimal action is action $1$.
Let $\Delta_i = (L_i - L_1)^\top p^* \ge 0$ for $i \in [N]$
and $N_i(t)$ be the number of times action $i$ is selected before the $t$-th round.
When the time step is clear from the context, we use $n_i$ instead of $N_i(t)$.
We adopt the pseudo-regret to measure the performance:
$\Regret(T) = \sumT \Delta_{i(t)} = \sum_{i \in [N]} \Delta_{i} N_i(T+1)$.
This is the relative performance of the algorithm against the~\emph{oracle}, which knows the optimal action $1$ before the game starts.

We introduce the following definitions to clarify the class of PM games,
for which we develop an algorithm and derive a regret upper bound.
The following cell decomposition is the concept to divide the simplex $\calP_M$ based on the loss matrix
to identify the optimal action, which depends on the opponent's strategy $p^*$.
\begin{definition}[Cell decomposition and Pareto-optimality~\citep{Bartok11minimax}]
    For every action $i\in[N]$, 
    \emph{cell} 
    $\Ci \coloneqq \{p \in \calP_M : (L_i - L_j)^\top p \leq 0,\,\forall j \neq i\}$ 
    is the set of opponent's strategies for which action $i$ is optimal.
    Action $i$ is \emph{Pareto-optimal} if there exists an opponent's strategy $p^*$ under which action $i$ is optimal.
\end{definition}
Each cell is a convex closed polytope.
Next, we define~\emph{neighbors} between two Pareto-optimal actions,
which intuitively means that the two actions ``touch'' each other in their surfaces.
\begin{definition}[Neighbors and neighborhood action~\citep{Bartok11minimax}]
    Two Pareto-optimal actions $i$ and $j$ are \emph{neighbors} if $\Ci \cap \Cj$ is an $(M-2)$-dimensional polytope.
    For two neighboring actions $i,j\in[N]$, 
    the \emph{neighborhood action set} is defined as $\Nplus{i}{j} = \{ k\in[N] : \Ci \cap \Cj \subseteq \Ck \}$.
\end{definition}
Note that the
neighborhood action set
$\Nplus{i}{j}$ includes actions $i$ and $j$ from its definition.
Next, we define the \emph{signal matrix}, which encodes the information of the feedback matrix $\fbmat$ so that we can utilize the feedback information.
\begin{definition}[Signal matrix~\citep{Komiyama15PMDEMD}]
    The signal matrix $S_i \in \ev{0,1}^{A \times M}$ of action $i$ is defined as $(S_i)_{y,j} = \ind{h_{i,j} = y}$, 
    where $\ind{X} = 1$ if the event $X$ is true and $0$ otherwise.
\end{definition}
Note that if we define the signal matrix as above,
$S_i p^* \in \R^A$ is a probability vector over feedback symbols of action $i$.
The following~\emph{local observability} condition separates easy and hard games,
this condition intuitively means that the information obtained by taking actions in the neighborhood action set $\Nplus{i}{j}$ 
is sufficient to distinguish the loss difference between actions $i$ and $j$.
\begin{definition}[Local observability~\citep{Bartok11minimax}]
    A partial monitoring game is said to be~\emph{locally observable} if
    for all pairs $i,j$ of neighboring actions,
    $L_i - L_j \in \oplus_{k\in\Nplus{i}{j}} \Img S_k^\top$,
    where $\Img V$ is the image of the linear map $V$, and $V \oplus W$ is the direct sum between the vector spaces $V$ and $W$.
\end{definition}
We also consider 
the concept of the \emph{strong local observability} condition, which implies the above local observability condition.
\begin{definition}[Strong local observability]\label{def_strong}
    A partial monitoring game is said to be~\emph{strongly locally observable} if
    for all pairs $i,j \in [N]$,
    $L_i - L_j \in \Img S_i^\top \oplus \Img S_j^\top$.
\end{definition}
This condition was assumed in the theoretical analysis in~\citet{Vanchinathan14BPM},
and we also assume this condition in theoretical analysis in Section~\ref{sec:theory}.
Note that the strong local observability means that,
for any $j \neq k$, there exists $z_{j,k} \neq 0 \in \R^{2A}$ such that
$ L_j - L_k = (S_j^\top, S_k^\top)\,z_{j,k}$.

\noindent \textbf{Notation.}
Let $\nrm{\cdot}$ and $\nrm{\cdot}_p$ be the Euclidian norm and $p$-norm,
and let $\nrm{x}_A = \sqrt{x^\top A x}$ be the norm induced by the positive semidefinite matrix $A \succeq 0$.
Let $\KL{p}{q} = \sum_{a=1}^A p_a \log (p_a/q_a)$ be the Kullback-Leibler divergence of $p$ from $q$.
The vector $e_{y} \in \R^M$ is the $y$-th orthonormal basis of $\R^M$,
and $\onemat_{n} = [1,\ldots,1]^\top$ is the $n$-dimensional all-one vector.
Let $\qit$ be the empirical feedback distribution of action $i$ at time $t$,
\ie $\qit = [n_{i1}/n_i, \ldots, n_{iA}/n_i]^\top \in \calP_A$,
where
$n_{iy} = \sum_{s=1}^t \ind{i(s) = i, y(s) = y}$ and $n_i = \sum_{y=1}^{A} n_{iy}$.
The notation is summarized in Appendix~\ref{sec:notation}.

\noindent \textbf{Methods for Sampling from Posterior Distribution.}
We briefly review the methods to draw a sample from the posterior distribution.
While TS is one of the most promising algorithms,
the posterior distribution can be in a quite complicated form,
which makes obtaining a sample from it computationally hard.
To overcome this issue, a variety of approximate posterior sampling methods have been considered,
such as Gibbs sampling, Langevin Monte Carlo, Laplace approximation, and the bootstrap~\cite[Section~5]{Russo18tutorial}.
Recent work~\citep{Lu17ensamble} proposed a flexible approximation method, which can even efficiently be applied to quite complex models such as neural networks.
However, more recent work revealed that algorithms based on such an approximation procedure~\emph{can} suffer a linear regret~\citep{Phan19inference}, even if the approximation error in terms of the $\alpha$-divergence is small enough.

Although BPM-TS is one of the best methods for stochastic PM,
it approximates the posterior by a Gaussian distribution in a heuristic way, 
which can degrade the empirical performance due to the distributional discrepancy from the exact posterior distribution.
Furthermore, no theoretical guarantee is provided for BPM-TS.
In this paper, we mitigate these problems by providing a new algorithm for stochastic PM, 
which allows us to exactly draw samples from the posterior distribution.
We also give theoretical analysis for the proposed algorithm.

%% file: tex_src/tspm.tex
\section{Thompson-sampling-based Algorithm for Partial Monitoring}
\label{sec:tspm}

In this section, we present a new algorithm for stochastic PM games,
where we name the algorithm TSPM (TS-based algorithm for PM).
The algorithm is given in Algorithm~\ref{alg:TSPM}, and we will explain the subroutines in
the following.

\begin{algorithm}[t!]
    \KwIn{prior parameter $\lambda > 0$}
    Set $B_0 \leftarrow \lambda \idmat_M, b_0 \leftarrow 0$. \\
    Take each action for $n\geq1$ times. \\
    \For{$t = 1, 2, \ldots, T$}{
        Sample $\ptilt \sim \pi(p \mid \ev{i(s), y(s)}_{s=1}^t)$ based on the accept-reject sampling (Algorithm~\ref{alg:rejection_sampling}).  \\
        Take action $i(t)=\argmax_{i\in[N]} L_i^\top \ptilt$ and observe feedback $y(t)$.  \\
        Update $B_t \leftarrow B_{t-1} + S_{i(t)}^\top S_{i(t)},\; b_t \leftarrow b_{t-1} + S_{i(t)}^\top e_{y(t)}$.   \\
    }
\caption{TSPM Algorithm}\label{alg:TSPM}
\end{algorithm}

\begin{figure}
\vspace{-3mm}
    \begin{minipage}[t!]{.48\linewidth}
    \begin{algorithm}[H]
        \KwIn{constant $R \in[0,1]$}
        \While{true}{
            Sample $\ptilt \sim g_t(p)$ (Algorithm~\ref{alg:sample_p_exact}). \\
            Sample $\tilde{u} \sim \unif{[0, 1]}$. \\
            \If{$R\tilde{u} < {F_t(\ptilt)}/{ G_t(\ptilt)}$ \label{line:rejection_sampling:cond}}{
                \Return $\ptilt$.  \\
            }
        }
    \caption{Accept-Reject Sampling}\label{alg:rejection_sampling}
    \end{algorithm}
    \end{minipage}
    \hfill
    \begin{minipage}[t!]{.48\linewidth}
    \begin{algorithm}[H]
        Compute $\tilde{B}_t, \tilde{b}_t$ from $B_t, b_t$.  \\
        \Repeat{$\psub \in \calP_{M-1}$ \label{line:sample_p_exact:cond}}{
            Sample $\psub \sim \calN(\tilde{B}_t^{-1} \tilde{b}_t, \tilde{B}_t^{-1})$.
        }
        \Return $\tilde{p} = [\psub^\top\x 1 - \sum_{i=1}^{M-1} (\psub)_i]^\top$. \\

        \caption{Sampling from $g_t(p)$}\label{alg:sample_p_exact}
    \end{algorithm}
    \end{minipage}
\end{figure}

\subsection{Accept-Reject Sampling}\label{subsec_rejection}
We adopt the accept-reject sampling~\citep{Casella04acceptreject}
to~\emph{exactly} draw samples from the posterior distribution.
The accept-reject sampling is a technique to draw samples from a specific distribution $f$,
and a key feature is to use a \emph{proposal distribution} $g$,
from which we can easily draw a sample
and whose ratio to $f$, that is $f/g$, is bounded by a constant value $R$.
To obtain samples from $f$,
$\one$ we generate samples $X \sim g$;
$\two$ accept $X$ with probability $f(X)/R g(X)$.
Note that $f$ and $g$ do not have to be normalized
when the acceptance probability is calculated.

Let $\pipr{p}$ be a prior distribution for $p$.
Then 
an unnormalized density of
the posterior distribution for $p$ can be expressed as
\begin{align}
    F_t(p)
    &=
    \pipr{p}
    \prod_{i=1}^N \exp \crl[\Big]{ -n_i \KL{\qit}{S_i p} } \com
    \label{f_unnormalized}
\end{align}
the detailed derivation of which is given in
Appendix~\ref{sec:proof_of_rejection_sampling}.
We use the proposal distribution
with unnormalized density
\begin{align}
    G_t(p)
    &=
    \pipr{p}
    \prod_{i=1}^N \exp \crl[\Big]{ - \frac12 n_i \nrm{ \qit - S_i p }^2 } \per
    \label{g_unnormalized}
\end{align}
Based on these distributions,
we use Algorithm~\ref{alg:rejection_sampling}
for exact sampling from the posterior distribution,
where
$\unif{[0, 1]}$ is the uniform distribution over $[0,1]$ and
$g_t(p)$ is the distribution corresponding to the unnormalized density $G_t(p)$ in~\eqref{g_unnormalized}.
The following proposition shows that setting $R=1$ realizes the exact sampling.
\begin{proposition}
    \label{prop:rejection_sampling_edited}
    Let $f_t(p)$ be the distribution corresponding to the unnormalized density $F_t(p)$ in~\eqref{f_unnormalized}.
    Then, the output of Algorithm~\ref{alg:rejection_sampling} 
    with $R=1$ follows $f_t(p)$.
\end{proposition}
This proposition can easily be proved by Pinsker's inequality, which is detailed in Appendix~\ref{sec:proof_of_rejection_sampling}.

In practice, 
$R\in[0,1]$ is a parameter to balance the amount of over-exploration and the computational efficiency.
As $R$ decreases from 1, the algorithm tends to accept a point $p$ far from the mode.
The case $R=0$ corresponds 
the TSPM algorithm where the proposal distribution is used without the accept-reject sampling, 
which
we call \emph{TSPM-Gaussian}.
As we will see in Section~\ref{sec:theory},
TSPM-Gaussian corresponds to exact sampling of the posterior distribution when the feedback follows a Gaussian distribution rather than a multinomial distribution.

TSPM-Gaussian can be related to BPM-TS~\citep{Vanchinathan14BPM} in the sense that both of them use samples from Gaussian distributions.
Nevertheless, they use different Gaussians and TSPM-Gaussian performs much better than BPM-TS as we will see in the experiments.
Details on the relation between TSPM-Gaussian and BPM-TS are
described in Appendix~\ref{sec:relation_tspm_bpm}.

In general, we can realize efficient sampling with a small number of rejections if the proposal distribution and the target distribution are close to each other.
On the other hand, in our problem, the densities in~\eqref{f_unnormalized} and~\eqref{g_unnormalized} for each fixed point $p$ exponentially decay with the number of samples $n_i$ if the empirical feedback distribution $\qit$ converges.
This means that $F_t(p)$ and $G_t(p)$ have an exponentially large relative gap in most rounds.
Nevertheless, the number of rejections does not increase with $t$ as we will see in the experiments, which suggests that the proposal distribution approximates the target distribution well with high probability.

\subsection{Sampling from Proposal Distribution}\label{subsec_proposal}

When we consider Gaussian density $\calN(0,\lambda I_M)$ truncated over $\calP_M$ as a prior,
the proposal distribution also has the Gaussian density $\calN(B_t^{-1}b_t\x B_t^{-1})$
over $\calP_M$,
where
    \begin{align}
        \label{eq:update_rule}
            B_t = \lambda I_M+\sum_{i=1}^N n_i S_i^\top S_i = B_{t-1} + S_{i(t)}^\top S_{i(t)} \com
            \quad
            b_t = \sum_{i=1}^N n_i S_i^\top \qit = b_{t-1} + S_{i(t)}^\top e_{y(t)} \per
    \end{align}
Here note that the probability simplex
$\calP_M$ is in an $(M-1)$-dimensional space and a sample from $\calN(0,\lambda I_M)$ is not contained in $\calP_M$ with probability one.
In the literature, \eg \citet{Altmann14simplex}, 
sampling methods for Gaussian distributions truncated on a simplex have been discussed.
We use one of these procedures summarized in Algorithm~\ref{alg:sample_p_exact}, 
where we first sample $M-1$ elements of $p$ from another Gaussian distribution and determine the remaining element by the constraint
$\sum_{i=1}^M p_i = 1$.
\begin{proposition}
    \label{prop:sample_p_exact}
    Sampling from $g_t(p)$ is equivalent to Algorithm~\ref{alg:sample_p_exact} with
    \begin{align*}
        \tilde{B}_t = C_t - 2 D_t + f_t \onemat_{M-1} \onemat_{M-1}^\top \com
        \quad
        \tilde{b}_t = f_t \onemat_{M-1} - d_t + b^{(\alpha)}_t - b^{(M)} \onemat_{M-1} \com
    \end{align*}
    where
    $ B_t =
    \begin{bmatrix}
        C_t & d_t \\
        d_t^\top & f_t \\
    \end{bmatrix}
    $
    for
    $C_t \in \R^{{M-1}\times{M-1}}$, $d_t \in \R^{M-1}$, $f_t \in \R$,
    $b_t = [{b^{(\alpha)}_t}^\top, b^{(M)}_t]^\top \in \R^{M-1} \times \R$,
    and
    $D_t = \frac12 \prn{d_t \onemat_{M-1}^\top + \onemat_{M-1} d_t^\top }$.
\end{proposition}
We give the proof of this proposition
for self-containedness in Appendix~\ref{sec:proof_of_sample_p_exact}.

%% file: tex_src/theoretical_analysis.tex
\section{Theoretical Analysis}
\label{sec:theory}
This section considers a regret upper bound of the TSPM algorithm.

In the theoretical analysis, we consider a \emph{linear} setting of PM.
In the linear PM,
the learner suffers the expected loss $L_{i(t)}^\top p^*$ as in the discrete setting,
and
receives feedback vector
$y(t) = S_i p^* + \epsilon_t$ for $\epsilon_t\sim\calN(0, \idmat_M)$
whereas the one-hot representation of $y(t)$ is distributed by the probability vector $S_ip^*$ in the discrete setting.
Therefore, if $\epsilon_t$ can be regarded as a sub-Gaussian random variable
as in~\citet{Kirschner20ids}
then the linear PM includes the discrete PM,
though our theoretical analysis requires $\epsilon_t$ to be Gaussian.
The relation between discrete and linear settings can also be seen from the observation that
bandit problems with Bernoulli and Gaussian rewards can be expressed as discrete and linear PM,
respectively.
The linear PM also includes the linear bandit problem, where
the feedback vector is expressed as $L_i^{\top}p^*+\epsilon_t$.

In the linear PM,
$G_t(p)$ in~\eqref{g_unnormalized} becomes the exact posterior distribution rather than a proposal distribution.
The definition of the cell decomposition for this setting is 
largely the same as that of discrete setting
and detailed in Appendix~\ref{sec:proof_of_regret_upper_bound}.
Therefore, TS with exact posterior sampling in the linear PM corresponds to TSPM-Gaussian.
In the linear PM, the unknown parameter $p^*$ is in $\R^M$ rather than in $\calP_M$, and therefore we consider the prior $\pipr{p}=\calN(0,\lambda I_M)$ over $\R^M$, where the posterior distribution becomes $\calN(B_t^{-1}b_t\x B_t^{-1})$.

There are a few works that analyze TS for the PM because of its difficulty.
For example in~\citet{Vanchinathan14BPM}, an analysis of the TS-based algorithm (BPM-TS) is not given 
despite the fact that its performance is better than the algorithm based on a confidence ellipsoid (BPM-LEAST).
\citet{Zimmert19mirror} considered the theoretical aspect of a variant of TS for the linear PM in view of the Bayes regret, but this algorithm is based on the knowledge on the time horizon and different from the family of TS used in practice. 
More specifically, their algorithm considers the posterior distribution for \textit{regret} (not pseudo-regret), and an action is chosen according to the posterior probability that each arm minimizes the \textit{cumulative} regret.
Thus, the time horizon also needs to be known.

\noindent \textbf{Types of Regret Bounds.}
We focus on the \emph{(a) problem-dependent (b) expected pseudo-regret}.
(a)~
In the literature, a \emph{minimax} (or \emph{problem-independent}) regret bound has mainly been considered, for example, to classify difficulties of the PM problem~\citep{Bartok10toward, Bartok11minimax}.
On the other hand, a \emph{problem-dependent} regret bound
often reflects the empirical performance more clearly than the minimax regret
\citep{Bartok12CBP, Vanchinathan14BPM, Komiyama15PMDEMD}.
For this reason,
we consider this problem-dependent regret bound.
(b)~
In complicated settings of bandit problems,
a \emph{high-probability regret bound} has mainly been considered~\citep{Abbasi11improved, Agrawal13payoff},
which bounds the pseudo-regret with high probability $1-\delta$.
Though such a bound can be transformed to an expected regret bound,
this type of analysis often sacrifices the tightness since a linear regret might be suffered with small probability $\delta$.
This is why the analysis in~\citet{Vanchinathan14BPM} for BPM-LEAST finally yielded an $\tilOrder(\sqrt{T})$ expected regret bound
whereas their high-probability bound is $\Order(\log T)$.

\subsection{Regret Upper Bound}
In the following theorem, we show that
logarithmic problem-dependent expected regret is achievable by the TSPM-Gaussian algorithm.
\begin{theorem}[Regret upper bound]
    \label{thm:regret_upper_bound}
    Consider any finite stochastic linear partial monitoring game.
    Assume that the game is strongly locally observable and
    $\Delta_{i} = (L_i - L_1)^\top p^* > 0$ for any $i \neq 1$.
    Then, the regret of TSPM-Gaussian satisfies for sufficiently large $T$ that
    \begin{align}
        \Expect{\Regret(T)} 
        = 
        \Order
        \prn[\bigg]{
                \frac{A N^2 M \max_{i\in[N]} \Delta_i}{\Lambda^2}
        \log T } \com
    \label{thm_bound}
    \end{align}
    where
    $\Lambda \coloneqq \min_{i \neq 1}\Lambda_i$ for
    $\Lambda_i=\Delta_{i}/{\nrm{z_{1,i}}}$ with
    $z_{1,i}$ defined after Definition~\ref{def_strong}.
\end{theorem}
\begin{remark}
In the proof of Theorem~\ref{thm:regret_upper_bound},
it is sufficient to assume that
    $L_1 - L_i \in \Img S_1^\top \oplus \Img S_i^\top$ for $i \in [N]$,
which is weaker than the strong local observability,
though it is still sometimes stronger than the local observability condition.
\end{remark}
The proof of Theorem~\ref{thm:regret_upper_bound} is given in Appendix~\ref{sec:proof_of_regret_upper_bound}.
This result is the first problem-dependent bound of TS for PM, which also becomes the first logarithmic regret bound of TS for linear bandits.

The norm of $z_{j,k}$ in $\Lambda$ intuitively indicates the difficulty of the problem. 
Whereas we can estimate $(S_jp, S_k p)$ with noise through taking actions $j$ and $k$, the actual interest is the gap of the losses $p^\top(L_j - L_k ) = (S_jp, S_k p)^\top z_{j,k}$. Thus, if $\Vert z_{j,k}\Vert$ is large, the gap estimation becomes difficult since the noise is enhanced through $z_{j,k}$.

Unfortunately, the derived bound in Theorem~\ref{thm:regret_upper_bound} has quadratic dependence on $N$,
which seems to be not tight.
This quadratic dependence comes from the difficulty of the \emph{expected} regret analysis.
In general, we evaluate the regret before and after the convergence of the statistics separately.
Whereas the latter one usually becomes dominant, the main difficulty comes from the analysis of the former one, which might become large with low probability~\citep{Agrawal12analysis, Kaufmann12finitetime, Agrawal13further}.

In our analysis, we were not able to bound the former one within
a non-dominant order,
though it is still logarithmic in $T$.
In fact,
our analysis shows that
the regret after convergence is 
$\Order( \sum_{i\neq1} \Delta_i \frac{A}{\Lambda^2} \log T)$
as shown in Lemma~\ref{lem:case_A} in Appendix~\ref{sec:proof_of_regret_upper_bound}, 
which will become the regret
with high probability.
In particular, if we consider the classic bandit problem as a PM game, we can confirm that the derived bound after convergence becomes
the best possible bound
\begin{align}
    \Order\prn[\bigg]{\sum_{i\neq 1} \frac{\log T}{\Delta_i} }\n
\end{align}
by considering $\Lambda_i$ depending on each suboptimal arm $i$ as the difficulty measure instead of $\Lambda$.
Still, deriving a regret bound for the term before convergence within an non-dominant order is an important future work.

\subsection{Technical Difficulties of the Analysis}
The main difficulty of this regret analysis is that PM requires to consider the statistics of \emph{all} actions when the number of selections $N_i(t)$ of some action $i$ is evaluated.
This is in stark contrast to the analysis of the classic bandit problems, 
where it becomes sufficient to evaluate statistics of action $i$ and the best action 1.
This makes the analysis remarkably complicated in TS, 
where we need to separately consider the randomness caused by the feedback and TS.

To overcome this difficulty, we handle the effect of actions of no interest in two different novel ways depending on each decomposed regret.
The first one is to evaluate the worst-case effect of these actions based on an argument (Lemma~\ref{lem:expectation_Z}) related to the law of the iterated logarithm (LIL), 
which is sometimes used in the best-arm identification literature to improve the performance~\citep{Jamieson14lil}.
The second one is to bound the action-selection probability of TS using an argument of (super-)martingale (Theorem~\ref{thm_stop}), which is of independent interest.
Whereas such a technique is often used for the construction of confidence bounds~\citep{Abbasi11improved}, we reveal that it is also useful for evaluation of the regret of TS.

We only focused on the Gaussian noise $\ep_t \sim \calN(0, I_M)$,
rather than the more general sub-Gaussian noise.
This restriction to the Gaussian noise comes from the essential difficulty of the problem-dependent analysis of TS, 
where lower bounds for some probabilities are needed whereas the sub-Gaussian assumption is suited for obtaining upper bounds.
To the best of our knowledge, 
the problem-dependent regret analysis for TS on the sub-Gaussian case has never been investigated even for the multi-armed bandit setting,
which is quite simple compared to that of PM.
In the literature of the problem-dependent regret analysis, 
the noise distribution is restricted to 
distributions with explicitly given forms, e.g., Bernoulli, Gaussian, or more generally a one-dimensional canonical exponential family~\citep{Kaufmann12finitetime, Agrawal13further, Korda13onedim}.
Their analysis relies on the specific characteristic of the distribution 
to bound the problem-dependent regret.

%% file: tex_src/experiments.tex
\section{Experiments}
\label{sec:experiments}

In this section, we numerically compare the performance of TSPM and TSPM-Gaussian
against existing methods, which are RandomPM (the algorithm which selects action randomly), FeedExp3~\citep{Piccolboni01FeedExp3}, and BPM-TS~\citep{Vanchinathan14BPM}.
Recently,~\citet{Lattimore19information} considered
the sampling-based algorithm called Mario sampling for easy games.
Mario sampling coincides with TS (except for the difference between pseudo-regret and regret with known time horizon) mentioned in the last section when any pair of actions is a neighbor.
As shown in Appendix~\ref{sec:prop_dp_easy},
this property is indeed satisfied for dp-easy games defined in the following.
Therefore, the performance is essentially the same between TSPM with $R=1$ and Mario sampling.
To compare the performance, we consider a dynamic pricing problem,
which is a typical example of PM games.
We conducted experiments on the discrete setting because 
the experiments for PM has been mainly focused on the discrete setting.

In the dynamic pricing game, the player corresponds to a seller, and the opponent corresponds to a buyer.
At each round, the seller sells an item for a specific price $i(t)$, and the buyer comes
with an evaluation price $j(t)$ for the item,
where the selling price and the evaluation price correspond to the action and outcome, respectively.
The buyer buys the item
if the selling price $i(t)$ is smaller than or equal to $j(t)$
and not otherwise.
The seller can
only know if the buyer bought the item (denoted as feedback $0$) or did not buy the item (denoted as $1$).
The seller aims to minimize the cumulative ``loss'', and there are two types of definitions for the loss,
where each induced game falls into the easy and hard games.
We call them \emph{dp-easy} and \emph{dp-hard} games, respectively.

In both cases, the seller incurs the constant loss $c > 0$ when the item is not bought due to the loss of opportunity to sell the item.
In contrast, when the item is not bought, the loss incurred to the seller is different between these settings.
The seller in the dp-easy game \emph{does not} take the buyer's evaluation price into account.
In other words, the seller gains the selling price $i(t)$ as a reward (equivalently incurs $-i(t)$ as a loss).
Therefore, the loss for the selling price $i(t)$ and the evaluation $j(t)$ is
\begin{align*}
    \ell_{i(t), j(t)} = - i(t) \ind{i(t) \leq j(t)} + c \ind{i(t) > j(t)} \per
\end{align*}
This setting can be regarded as a generalized version of the online posted price mechanism, 
which was addressed in, \eg~\citet{Blum04auctions} and~\citet{CesaBianchi06regret},
and an example of strongly locally observable games.

On the other hand, the seller in dp-hard game \emph{does} take the buyer's evaluation price into account when the item is bought.
In other words, the seller incurs the difference between the opponent evaluation and the selling price $j(t) - i(t)$ as a loss
because the seller could have made more profit if the seller had sold at the price $j(t)$.
Therefore, the loss incurred at time $t$ is
\begin{align*}
    \ell_{i(t), j(t)} = (j(t) - i(t)) \ind{i(t) \leq j(t)} + c \ind{i(t) > j(t)} \per
\end{align*}
This setting is also addressed in~\citet{CesaBianchi06regret},
and belongs to the class of hard games.
Note that our algorithm can also be applied to a hard game,
though there is no theoretical guarantee.

\setlength\abovecaptionskip{0.2pt}
\begin{figure}[t!]
    \begin{center}
    \setlength{\subfigwidth}{.32\linewidth}
    \addtolength{\subfigwidth}{-.32\subfigcolsep}
    \begin{minipage}[t]{\subfigwidth}
        \centering
        \subfigure[dp-easy, $N=M=3$]{\includegraphics[scale=0.4]{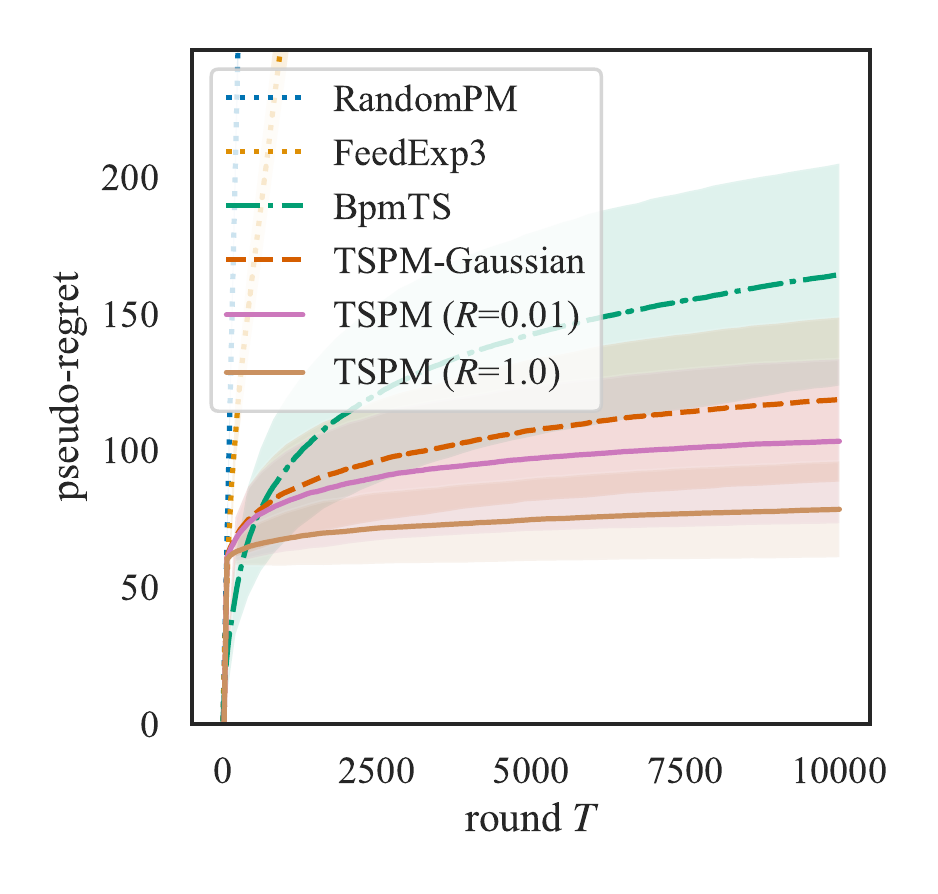}
        }
        \vspace{-0.501em}
    \end{minipage}\hfill
    \begin{minipage}[t]{\subfigwidth}
        \centering
        \subfigure[dp-easy, $N=M=5$]{\includegraphics[scale=0.4]{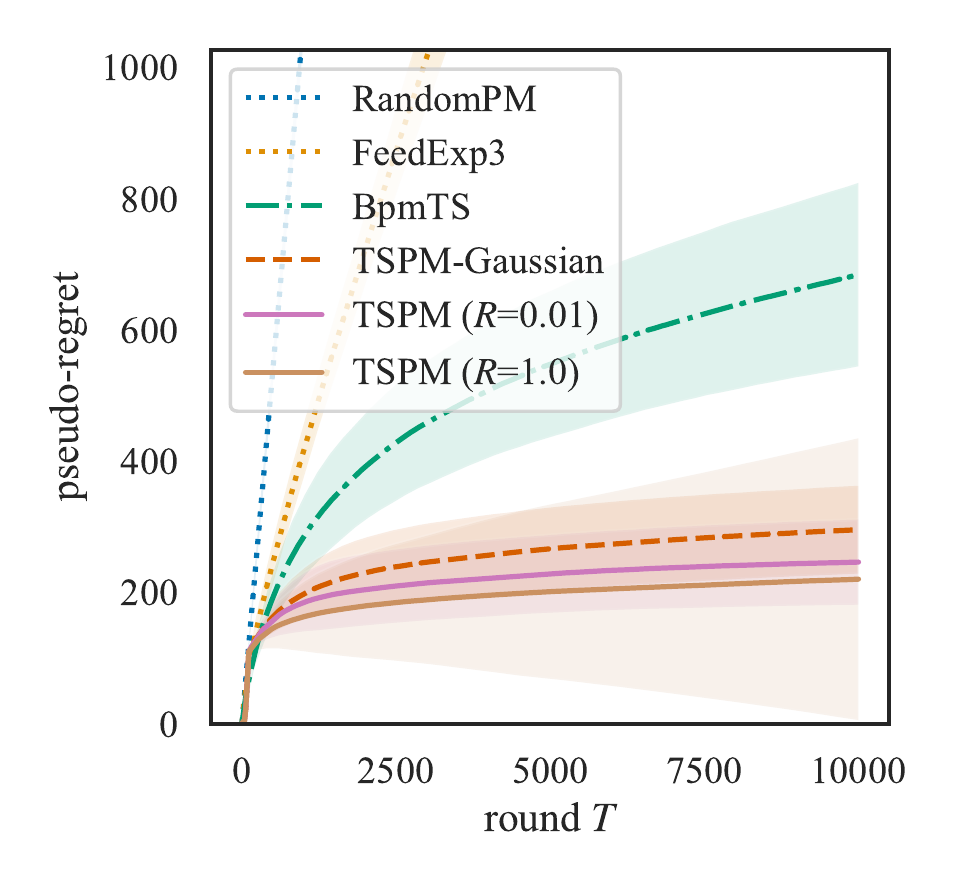}
        }
        \vspace{-0.501em}
    \end{minipage}\hfill
    \begin{minipage}[t]{\subfigwidth}
        \centering
        \subfigure[dp-easy, $N=M=7$]{\includegraphics[scale=0.4]{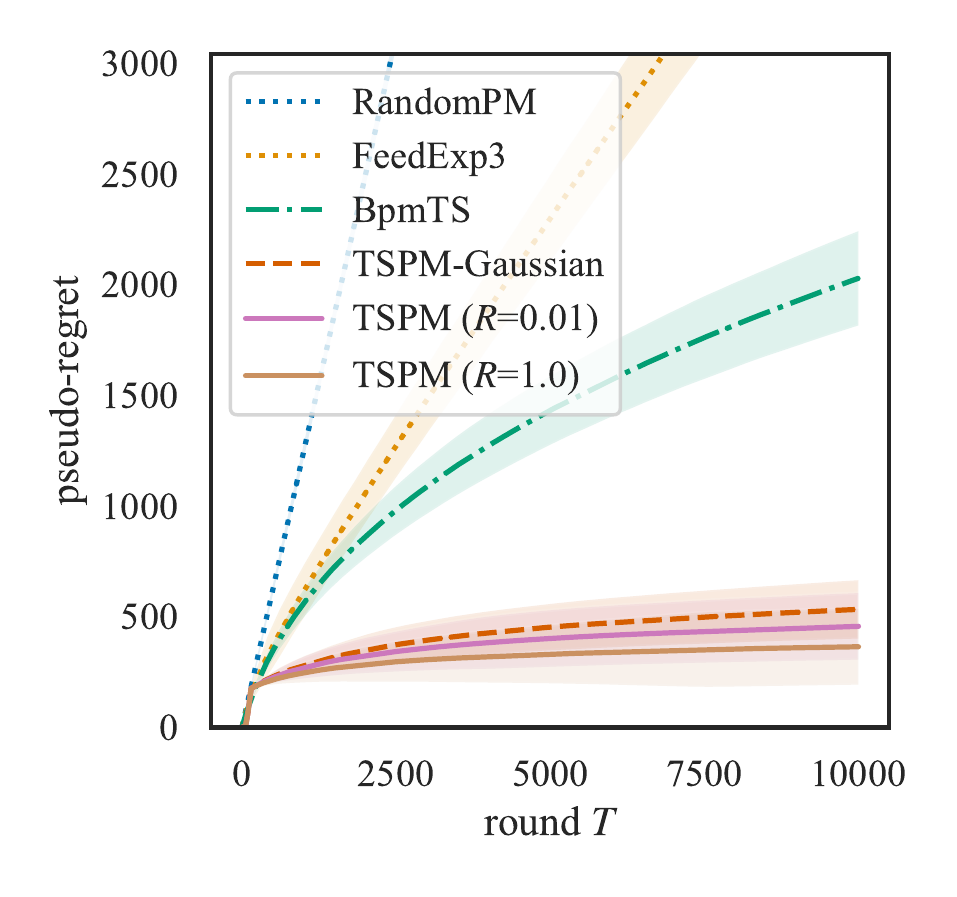}
        }
        \vspace{-0.501em}
    \end{minipage}
    \begin{minipage}[t]{\subfigwidth}
        \centering
        \subfigure[dp-hard, $N=M=3$]{\includegraphics[scale=0.4]{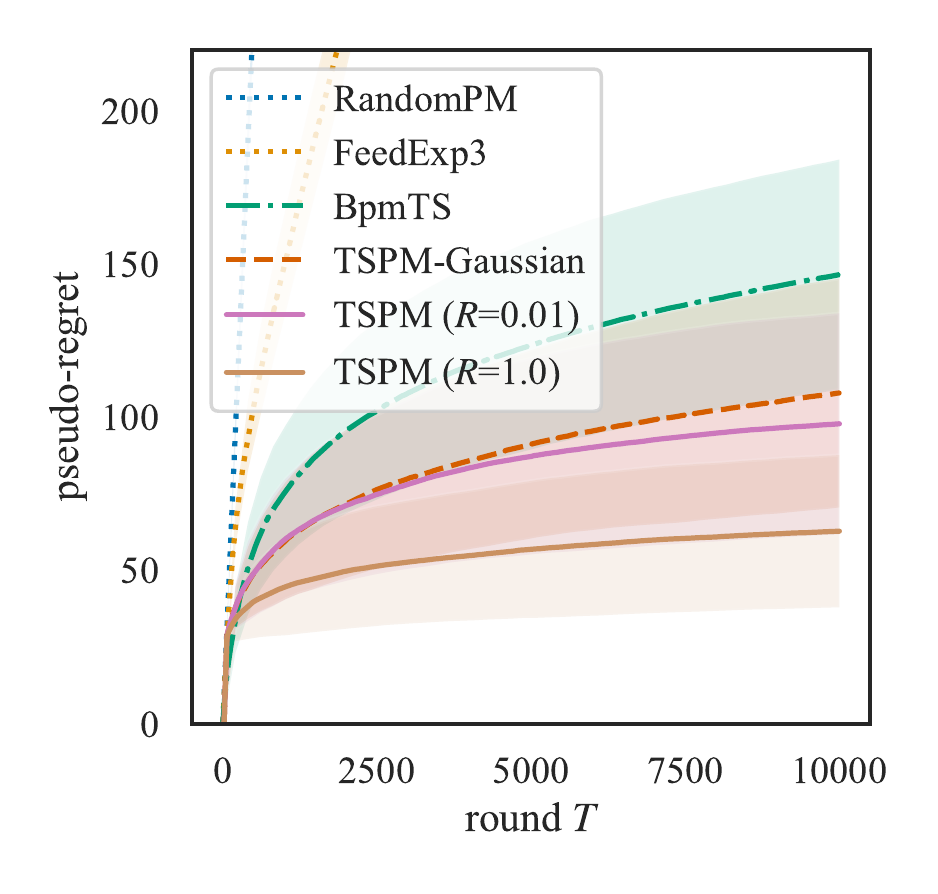}
        }
        \vspace{-0.501em}
    \end{minipage}\hfill
    \begin{minipage}[t]{\subfigwidth}
        \centering
        \subfigure[dp-hard, $N=M=5$]{\includegraphics[scale=0.4]{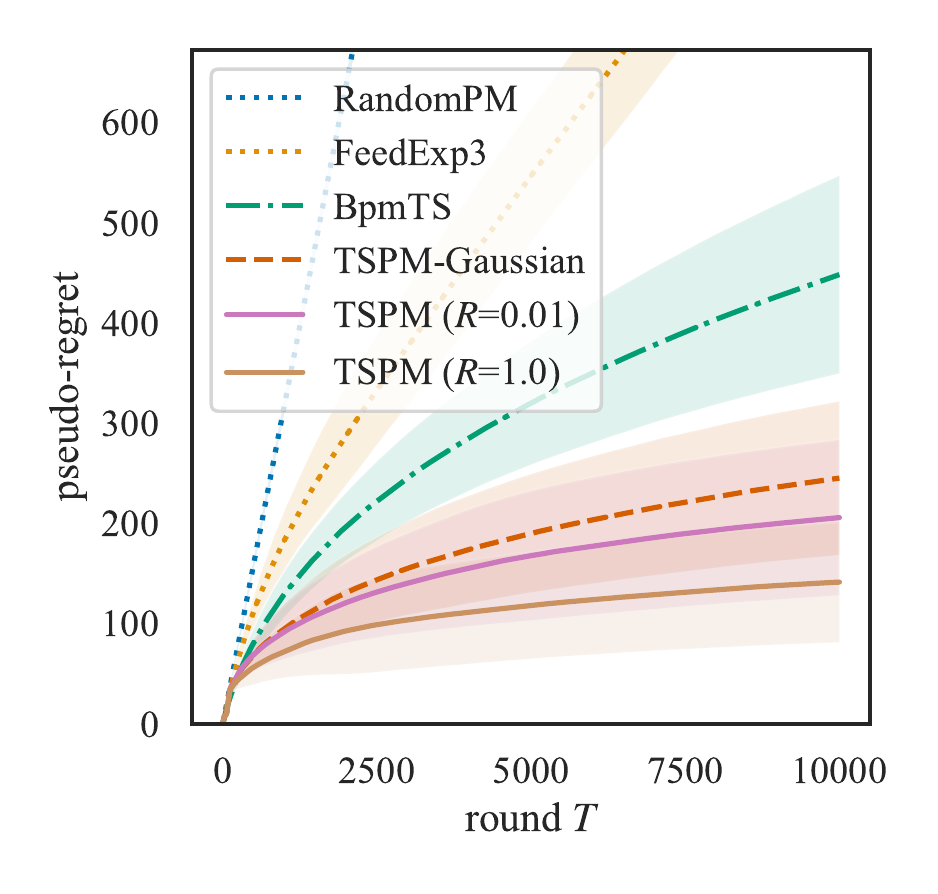}
        }
        \vspace{-0.501em}
    \end{minipage}\hfill
    \begin{minipage}[t]{\subfigwidth}
        \centering
        \subfigure[dp-hard, $N=M=7$]{\includegraphics[scale=0.4]{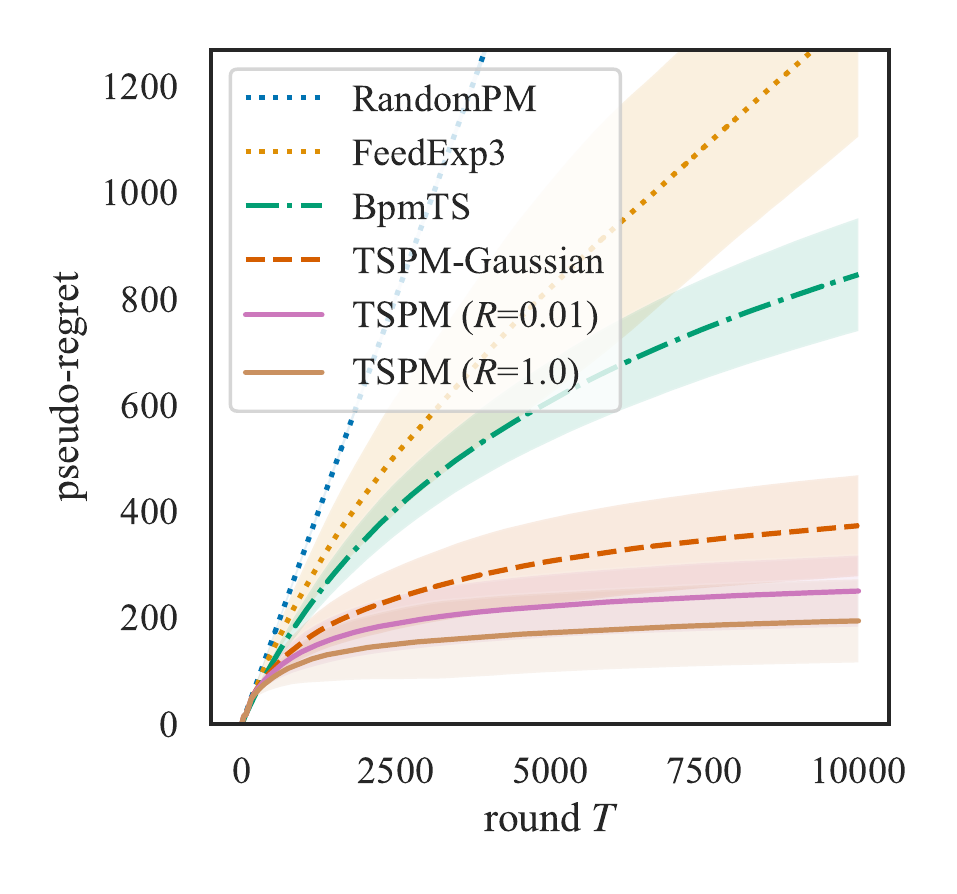}
        }
        \vspace{-0.501em}
    \end{minipage}\hfill
    \end{center}
    \caption{
        Regret-round plots of algorithms.
        The solid lines indicate the average over $100$ independent trials.
        The thin fillings are the standard error.
    }\label{fig:result_regret}
    \vspace{-1.0em}
\end{figure}%

\setlength\abovecaptionskip{0.2pt}
\begin{figure}[t!]
    \begin{center}
    \setlength{\subfigwidth}{.32\linewidth}
    \addtolength{\subfigwidth}{-.32\subfigcolsep}
    \begin{minipage}[t]{\subfigwidth}
        \centering
        \subfigure[dp-easy, $N=M=3$]{\includegraphics[scale=0.4]{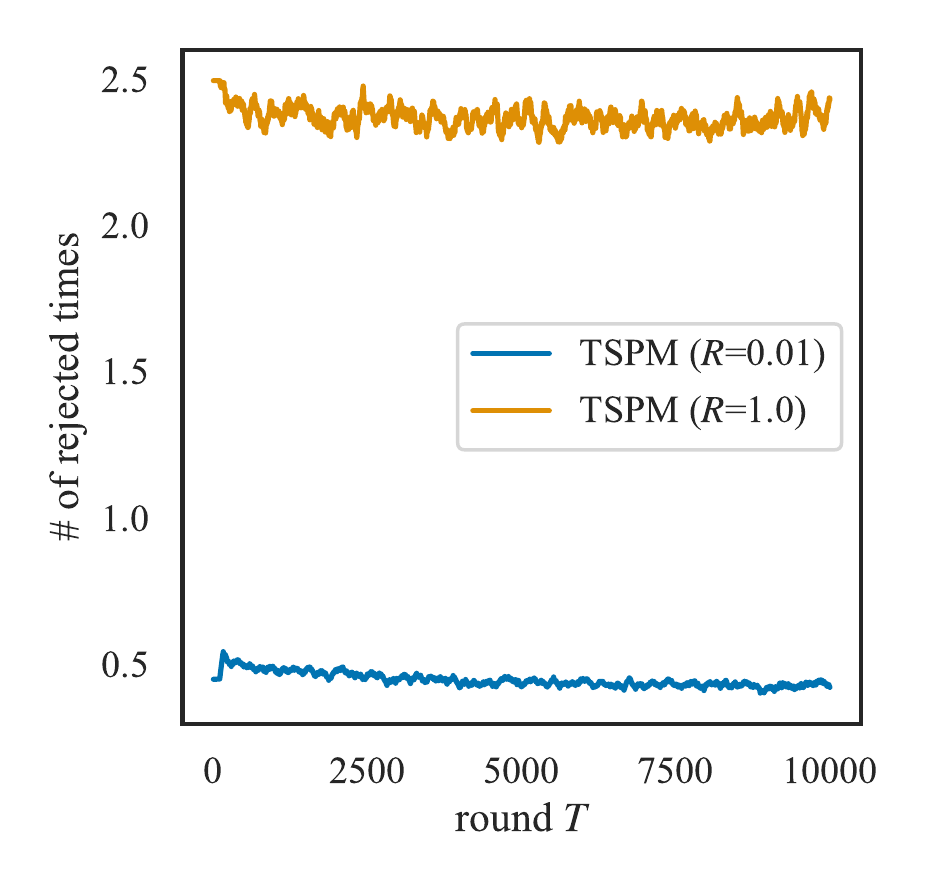}
        }
        \vspace{-0.501em}
    \end{minipage}\hfill
    \begin{minipage}[t]{\subfigwidth}
        \centering
        \subfigure[dp-easy, $N=M=5$]{\includegraphics[scale=0.4]{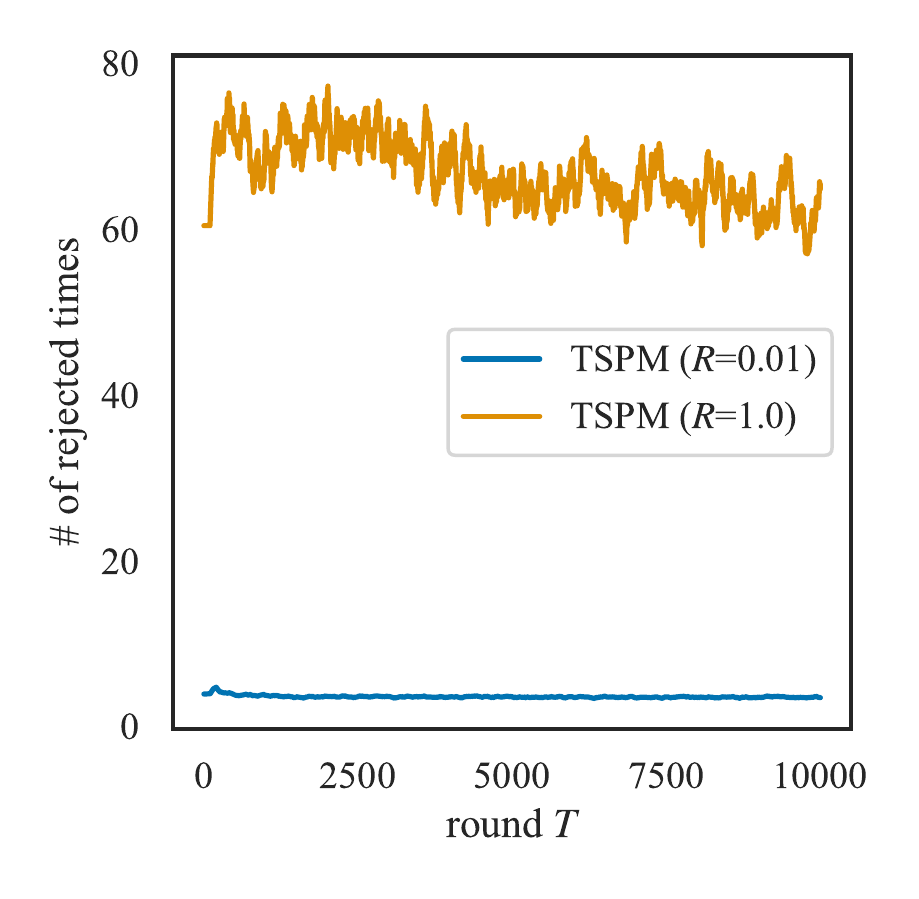}
        }
        \vspace{-0.501em}
    \end{minipage}\hfill
    \begin{minipage}[t]{\subfigwidth}
        \centering
        \subfigure[dp-easy, $N=M=7$]{\includegraphics[scale=0.4]{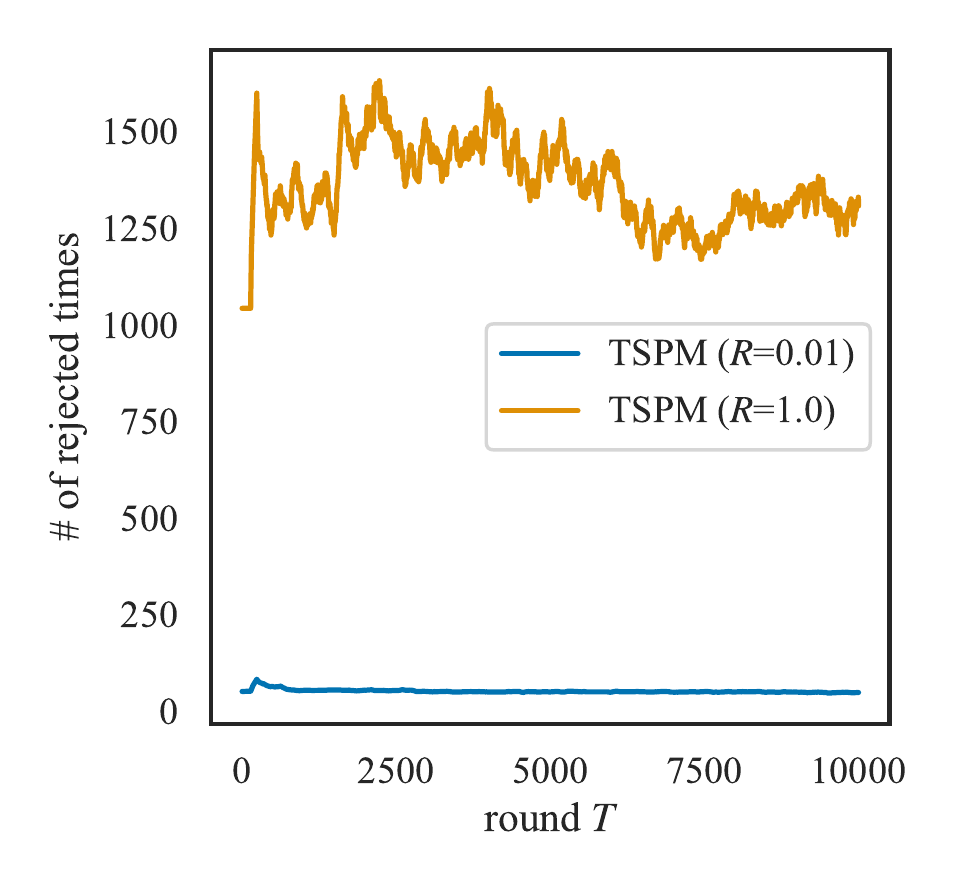}
        }
        \vspace{-0.501em}
    \end{minipage}
    \end{center}
    \caption{
        The number of rejected times by the accept-reject sampling.
        The solid lines indicate the average over $100$ independent trials after taking moving average with window size $100$.
    }\label{fig:result_n_rejected}
    \vspace{-1.0em}
\end{figure}%

\noindent\textbf{Setup.}
In the both dp-easy and dp-hard games, we fixed $N=M \in \ev{3,5,7}$ and $c=2$.
We fixed the time horizon $T$ to $10000$ and simulated $100$ times.
For FeedExp3 and BPM-TS, the setup of hyperparameters follows their original papers.
For TSPM, we set $\lambda=0.001$, and $R$ was selected from $\ev{0.01, 1.0}$.
Here, recall that 
TSPM with $R=1$ and $R=0$ correspond to the exact sampling and TSPM-Gaussian, respectively,
and a smaller value of $R$ gives the higher acceptance probability in the accept-reject sampling.
Therefore, using small $R$ makes the algorithm time-efficient,
although it can worsen the performance since
it over-explores the tail of the posterior distributions.
To stabilize sampling from the proposal distribution in Algorithm~\ref{alg:sample_p_exact}, 
we used an initialization that takes each action $n=10A$ times.
The detailed settings of the experiments with more results
are given in Appendix~\ref{sec:all_experiments}.

\noindent\textbf{Results.}
Figure~\ref{fig:result_regret} is
the empirical comparison of the proposed algorithms against the benchmark methods.
This result shows
that, in all cases, the TSPM with exact sampling gives the best performance.
TSPM-Gaussian also outperforms BPM-TS
even though both of them use Gaussian distributions as posteriors.
Besides, the experimental results suggest that our algorithm
performs reasonably well even for a hard game.
It can be observed that the proposed methods outperform BPM-TS more significantly for a larger number of outcomes. Further discussion for this observation is given in Appendix~\ref{sec:relation_tspm_bpm}.

Figure~\ref{fig:result_n_rejected} shows the number of rejections
at each time step in the accept-reject sampling.
We counted the number of times that
either Line~\ref{line:rejection_sampling:cond} in Algorithm~\ref{alg:rejection_sampling}
or Line~\ref{line:sample_p_exact:cond} in Algorithm~\ref{alg:sample_p_exact} was not satisfied.
In the accept-reject sampling, it is desirable that the frequency of rejection
does not increase as the time-step $t$
and 
does not increase rapidly with the number of outcomes.
We can see that the former one is indeed satisfied.
For the latter property,
the frequency of rejection becomes unfortunately large 
when exact sampling ($R=1$) is conducted.
Still, we can substantially improve this frequency by setting $R$ to be a small value or zero,
which still keeps regret tremendously better than that of BPM with almost the same time-efficiency as BPM-TS.

%% file: tex_src/conclusion.tex
\section{Conclusion and Discussion}
\label{sec:conclusion}

This paper investigated
Thompson sampling (TS) for stochastic partial monitoring from the algorithmic and theoretical viewpoints.
We provided a new algorithm that enables
exact sampling from the posterior distribution,
and
numerically showed that the proposed algorithm outperforms existing methods.
Besides, we provided an upper bound for the problem-dependent logarithmic expected pseudo-regret 
for the linearized version of the partial monitoring.
To our knowledge, this bound is the first logarithmic problem-dependent expected pseudo-regret bound of
a TS-based algorithm for linear bandit problems and strongly locally observable partial monitoring games.

There are several remaining questions.
As mentioned in Section~\ref{sec:theory},
\citet{Kirschner20ids} considered linear partial monitoring
with the feedback structure $y(t) = S_{i(t)} p^* + \epsilon_t$,
where $(\epsilon_t)_{t=1}^T$ is a sequence of independent sub-Gaussian noise vector in $\R^M$.
This setting is the generalization of our linear setting,
where $(\epsilon_t)_{t=1}^T$ are i.i.d.~Gaussian vectors.
Therefore, a natural question that arises is whether we can extend our analysis on TSPM-Gaussian to the sub-Gaussian case,
although we believe it would be not straightforward as discussed in Section~\ref{sec:theory}.
It is also an important open problem to derive a regret bound on TSPM using the exact posterior sampling for the discrete
partial monitoring.
Although we conjecture that the algorithm also achieves logarithmic regret for the setting,
there still remain some difficulties in the analysis.
In particular,
we have to handle the KL divergence in $f_t(p)$
and consider the restriction of the support of the opponent's strategy to $\calP_M$,
which make the analysis much more complicated.
Besides, it is worth noting that 
the theoretical analysis of TS for hard games has never been theoretically investigated.
We believe that in general TS suffers linear regret in the minimax sense 
due to its greediness.
However,
we conjecture that TS can achieve the sub-linear regret for some specific instances of hard games
in the sense of the problem-dependent regret,
as empirically observed in the experiments.
Finally, it is an important open problem to derive the minimax regret for anytime TS-based algorithms.
This needs more detailed analysis on $\order(\log T)$ terms in the regret bound, which were dropped in our main result.

%% file: tex_src/broader_impact.tex
\section*{Broader Impact}
\noindent \textbf{Application.}
Partial monitoring (PM) includes various online decision-making problems such as multi-armed bandits, linear bandits, dynamic pricing, and label efficient prediction.
Not only can PM handles them, 
the dueling bandits, combinatorial bandits, transductive bandits, and many other problems can be seen as a partial monitoring game,
as discussed in~\cite{Kirschner20ids}.
Therefore, our analysis of Thompson sampling (TS) for PM games pushes the application of TS to a more wide range of online decision-making problems forward.
Moreover, PM has the potential that 
novel online-decision making problems are newly discovered,
where we have to handle the limited feedback in an online fashion.

\noindent \textbf{Practical Use.}
The obvious advantage of using TS is that 
the users can easily apply the algorithm to their problems.
They do not have to solve 
mathematical optimization problems,
which are often required to solve when using non-sampling-based algorithms
~\citep{Bartok12CBP, Komiyama15PMDEMD}.
For the negative side, 
the theoretical analysis for the regret upper bound might make 
the users become overconfident when the users use their algorithms.
For example, they might use the TSPM algorithm to the linear PM game
with heavy-tailed noise, such as sub-exponential noise, without noticing it.
Nevertheless,
this is not an TS-specific problem, 
but one that can be found in many theoretical studies,
and TS is still one of the most promising policies.

%% file: tex_src/acknowledgements.tex
\section*{Acknowledgements}
The authors would like to thank the meta-reviewer and reviewers for a lot of helpful comments.
The authors would like to thank
Kento Nozawa and Ikko Yamane for maintaining servers for our experiments,
and Kenny Song for helpful discussion on the writing. 
TT was supported by Toyota-Dwango AI Scholarship,
and RIKEN Junior Research Associate Program for the final part of the project.
JH was supported by KAKENHI 18K17998, 
and MS was supported by KAKENHI 17H00757.

%% file: tex_src/appendix.tex
\newpage
\appendix

\section{Notation}
\label{sec:notation}
\input{tex_src/proof/notation.tex}

\section{Posterior Distribution and Proposal Distribution in Section~\ref{sec:tspm}}
\label{sec:proof_of_rejection_sampling}
\input{tex_src/proof/tspm_rejection_sampling.tex}

\section{Proof of Proposition~\ref{prop:sample_p_exact}}
\label{sec:proof_of_sample_p_exact}
\input{tex_src/proof/tspm_sample_p_exact.tex}

\section{Relation between TSPM-Gaussian and BPM-TS}
\label{sec:relation_tspm_bpm}
\input{tex_src/proof/relation_tspm_bpm}

\section{Preliminaries for Regret Analysis}
\label{sec:preliminaries_for_regret_analysis}
\input{tex_src/proof/preliminaries_for_regret_analysis.tex}

\section{Regret Analysis of TSPM Algorithm}
\label{sec:proof_of_regret_upper_bound}
\input{tex_src/proof/theory_regret_upper_bound.tex}

\section{Property of Dynamic Pricing Games}
\label{sec:prop_dp_easy}
\input{tex_src/proof/dp_easy.tex}

\section{Details and Additional Results of Experiments}
\label{sec:all_experiments}
\input{tex_src/proof/all_experiments.tex}

%% file: tex_src/proof/notation.tex
Table~\ref{tb:notation} summarizes the symbols used in this paper.
\begin{table}[h]
\centering
\small
\begin{center}
\captionof{table}{\label{tb:notation} List of symbols used in this paper.}
\begin{tabular}{lll}
    \toprule
    Symbol & Meaning \\
    \hline
    $\calP_n$ & $(n-1)$-dimensional probability simplex \\
    $\nrm{\cdot}$ & Euclidian norm for vector and operator norm for matrix\\
    $\nrm{\cdot}_p$ & $p$-norm \\
    $\nrm{\cdot}_A$ & norm induced by positive semidefinite matrix $A$ \\ 
    $\KL{p}{q}$ & KL divergence from $q$ to $p$ \\
    $B_r^n(p)$ & $n$-dimensional Euclidian ball of radius $r$ at point $p \in \R^N$ \\
    \hline
    $N, M \in \N$ & the number of actions and outcomes \\
    $\Sigma$ & set of feedback symbols \\
    $A$ & the number of feedback symbols \\
    $p^* \in \calP_M$ & opponent's strategy \\
    $T$ & time horizon \\
    $\lossmat = (\ell_{i,j}) \in \R^{N \times M}$ & loss matrix \\
    $\fbmat = (h_{i,j}) \in \Sigma^{N \times M}$ & feedback matrix \\
    $S_i \in \ev{0, 1}^{A \times M} \; (i=1,\dots,N)$ & signal matrix  \\
    $i(t)$ & action taken at time $t$  \\
    $N_i(t)$ & the number of times the action $i$ is taken before time $t\in[T]$  \\
    $j(t)$ & outcome taken by opponent at time $t$ \\
    $y(t)$ & feedback observed at time $t$  \\
    $F_t(p)$ & unnormalized posterior distribution in~\eqref{f_unnormalized} \\ 
    $f_t(p)$ & probability density function corresponding to $F_t(p)$ \\
    $G_t(p)$ & unnormalized proposal distribution for $F_t(p)$ in~\eqref{g_unnormalized} \\
    $g_t(p)$ & probability density function corresponding to $G_t(p)$ \\
    $\qit \in \calP_M$ & empirical feedback distribution of action $i$ by time $t$ \\
    $\qin \in \calP_M$ & empirical feedback distribution of action $i$ after the action is taken $n$ times \\
    $\Ci \subset \calP_M$ & cell of action $i$ \\
    \bottomrule
\end{tabular}
\end{center}
\end{table}


%% file: tex_src/proof/tspm_rejection_sampling.tex
In this appendix, we discuss representation of the posterior distribution
and its relation with the proposal distribution.
\begin{proposition}
    $F_t(p)$ in~\eqref{f_unnormalized} is proportional to the posterior distribution of the opponent's strategy,
    and $F_t(p) \le G_t(p)$ for all $p \in \calP_M$.
\end{proposition}
\begin{proof}
The posterior distribution of the opponent's strategy parameter $\pipr{ p \relmiddle| \ev{ i(s), y(s) }_{s=1}^t }$ 
is rewritten as
\begin{align}
    \pipr{ p \relmiddle| \ev{i(s), y(s)}_{s=1}^t }
    &\propto
    \pipr{p,  \ev{i(s), y(s)}_{s=1}^t } \nn
    &\propto
    \pipr{p} \prod_{s=1}^t \pr{y(s) \relmiddle| i(s), p}  \nn
    &=
    \pipr{p} \prod_{i=1}^N  \prod_{y=1}^{A} (S_{i,y} p)^{n_{iy}}  \nn
    &\propto
    \pipr{p} \prod_{i=1}^N \exp \crl[\Big]{ -n_i \KL{\qit}{S_i p} } \com
\end{align}
where $S_{i,y}$ is the $i$-th row of the signal matrix $S_i$,
and
note that $\qit$ is the empirical feedback distribution of action $i$ at time $t$,
that is,
$\qit = [n_{i1}/n_i, \ldots, n_{iA}/n_i]^\top \in \calP_A$
for 
$n_{iy} = \sum_{s=1}^t \ind{i(s) = i, y(s) = y}$ and $ n_i = \sum_{y=1}^{A} n_{iy}$.

Next, we show that $F_t(p) \le G_t(p)$ holds for all $p \in \calP_M$. 
Using the Pinsker's inequality,
the unnormalized posterior distribution $F_t(p)$ can be bounded from above as
\begin{align}
    F_t(p)
    &=
    \pipr{p}
        \prod_{i=1}^N \exp \crl[\Big]{ -n_i \KL{\qit}{S_i p} } \nn
    &\leq
    \pipr{p} \prod_{i=1}^N \exp \crl[\Big]{ - \frac12 n_i \nrm{ \qit - S_i p }_1^2 } 
    \By{Pinsker's ineq.} \nn
    &=
    \pipr{p} \exp \crl[\Big]{ - \frac12 \sum_{i=1}^N n_i \nrm{ \qit - S_i p }_1^2 }  \nn
    &\leq
    \pipr{p}
        \exp \crl[\Big]{ - \frac12 \sum_{i=1}^N n_i \nrm{ \qit - S_i p }^2 } 
    \Bym{\nrm{\qit - S_i p}_1 \geq \nrm{ \qit - S_i p }} \nn
    &=
    G_t(p)  \per
\end{align}
\end{proof}
\begin{remark}
The unnormalized density $G_t(p)$ is indeed Gaussian.
Recalling that $B_t$ and $b_t$ are defined in~\eqref{eq:update_rule} as
\begin{align}
    B_t = \sum_{i=1}^N n_i S_i^\top S_i = \sum_{s=1}^t S_{i(s)}^\top S_{i(s)} = B_{t-1} + S_{i(t)}^\top S_{i(t)},
    \quad
    b_t = \sum_{i=1}^N n_i S_i^\top \qit = b_{t-1} + S_{i(t)}^\top e_{y(t)} \com
\end{align}
we have
\begin{align}
    \sum_{i=1}^N n_i \nrm{ \qit - S_i p }^2
    &=
    \sum_{i=1}^N n_i (\qit - S_i p)^\top (\qit - S_i p)  \nn
    &=
    p^\top \underbrace{  \prn[\Big]{ \sum_{i=1}^N n_i S_i^\top S_i } }_{B_t}  p,
    -
    2 \underbrace{ \prn[\Big]{\sum_{i=1}^N n_i S_i^\top \qit }^\top }_{b_t}  p
    +
    \underbrace{\sum_{i=1}^N n_i \nrm{ \qit }^2}_{c_t}  \nn
    &=
    p^\top B_t p - 2b_t^\top p + c_t  \nn
    &=
    (p - B_t^{-1} b_t)^\top B_t (p - B_t^{-1}b_t) + c_t - b_t^\top B_t^{-1} b_t \per
\end{align}
Therefore, we have
\begin{align}
    \exp \crl[\Big]{ - \frac12 \sum_{i=1}^N n_i \nrm{ \qit - S_i p }^2 } 
    &\propto
    \exp \crl[\Big]{- \frac12 (p - B_t^{-1} b_t)^\top B_t (p - B_t^{-1}b_t) }  \per
\end{align}

\end{remark}

%% file: tex_src/proof/tspm_sample_p_exact.tex
We will see that the the procedure 
sampling $\ptilt$ from $g_t(p)$
and Algorithm~\ref{alg:sample_p_exact} are equivalent.
First, we derive the Gaussian density of $g_t(p)$ projected onto $\{p\in\R^M : \sum_{i=1}^M p_i = 1\}$.

For simplicity, we omit the subscript $t$ and write, \eg $B$ instead of $B_t$.
We define
$p = [\psub^\top, p_M]^\top \in \R^{M-1} \times \R$. 
Let $h = B^{-1} b$, and define
$h = [\hsub^\top, h_M]^\top \in \R^{M-1} \times \R$. 
Let
$ B =
\begin{bmatrix}
    C & d \\
    d^\top & f \\
\end{bmatrix}
$, where $C \in \R^{{M-1}\times{M-1}}, d \in \R^{M-1}$, and $f \in \R$.
Also, let $b = [\bsub^\top, b^{(M)}]^\top \in \R^{M-1} \times \R$. 

Using the decomposition
\begin{align}
    (p - B^{-1} b)^\top B (p - B^{-1}b)
    &=
    \underbrace{p^\top Bp}_{\text{(a)}}
    -
    2 \underbrace{h^\top B p}_{\text{(b)}}
    +
    h^\top B h \com
\end{align}
we rewrite each term by restricting the domain of $p$ so that
it satisfies the condition $\sum_{i=1}^M p_i = 1$.
Now the first term (a) is rewritten as
\begin{align}
    \text{(a)}
    &=
    \psub^\top C \psub + 2 \psub^\top d p_M + f p_M^2  \nn
    &=
    \psub^\top C \psub
    +
    2 \underbrace{\psub^\top d \prn[\Big]{1 - \sum_{i=1}^{M-1} p_i }}_{\text{(a1)}}
    +
    f \underbrace{\prn[\Big]{1 - \sum_{i=1}^{M-1} p_i}^2}_{\text{(a2)}} \per
\end{align}
The term (a1) is rewritten as
\begin{align}
    \text{(a1)}
    &=
    \psub^\top d - \psub^\top d \sum_{i=1}^{M-1} p_i  \nn
    &=
    \psub^\top d - \psub^\top d \onemat_{M-1}^\top \psub  \nn
    &=
    \psub^\top d - \psub^\top D \psub \quad \left(D = \frac{1}{2} \left(d \onemat_{M-1}^\top + \onemat_{M-1} d^\top \right) \right) \com
\end{align}
and the term (a2) is rewritten as
\begin{align}
    \text{(a2)}
    &=
    \prn[\Big]{1 - \sum_{i=1}^{M-1} p_i }^2  \nn
    &=
    1 - 2 \sum_{i=1}^{M-1} p_i + \prn[\Big]{\sum_{i=1}^{M-1} p_i}^2  \nn
    &=
    1 - 2 \onemat_{M-1}^\top \psub + \psub^\top \onemat_{M-1} \onemat_{M-1}^\top \psub \per
\end{align}
Therefore,
\begin{align}
    \text{(a)}
    &=
    \psub^\top \underbrace{(C - 2D + f \onemat_{M-1} \onemat_{M-1}^\top)}_{\tilde{B}} \psub 
    - 
    2 (f \onemat_{M-1} - d)^\top \psub + f \per
\end{align}
With regard to the term (b), we have
\begin{align}
    \text{(b)}
    &=
    b^\top p  \nn
    &=
    \bsub^\top \psub^\top + b^{(M)} p_M  \nn
    &=
    (\bsub - b^{(M)} \onemat_{M-1})^\top \psub + b^{(M)} \per
\end{align}
Therefore,
\begin{align}
    &
    (p - B^{-1} b)^\top B (p - B^{-1}b)  \nn
    &=
    \psub^\top \tilde{B} \psub
    - 2 (\underbrace{f \onemat_{M-1} - d + \bsub - b^{(M)} \onemat_{M-1}}_{\tilde{b}})^\top \psub
    + f - 2 b^{(M)}  + h^\top B h \nn
    &=
    (\psub - \tilde{B}^{-1} \tilde{b})^\top \tilde{B} (\psub - \tilde{B}^{-1} \tilde{b})
    + f - 2 b^{(M)} - \tilde{b}^\top \tilde{B}^{-1} \tilde{b} + b^\top B^{-1} b \quad
    \Bym{h^\top B h = b^\top B^{-1} b} \per
\end{align}
From the above argument, 
the density $\calN(\tilde{B}^{-1} b\x \tilde{B}^{-1})$ is the Gaussian distribution of $g_t(p)$ 
on $\{p \in \R^M : \sum_{i=1}^M p_i = 1 \}$.
Therefore, 
the 
$p = [\psub^\top\x 1-\sum_{i=1}^{M-1} (\psub)_i]^\top$
for $\psub \sim \calN(\tilde{B}^{-1} b\x \tilde{B}^{-1})$
is supported over 
$\{p \in \R^M : \sum_{i=1}^M p_i = 1 \}$.

If the sample $\psub$ from $\calN(\tilde{B}^{-1} b\x \tilde{B}^{-1})$ is in $\calP_{M-1}$,
then we can obtain the last element $p^{(M)}$ by $p^{(M)} = 1 - \sum_{i=1}^{M-1} (\psub)_i$.
Otherwise, the probability that $\psub$ is the first $M-1$ elements of the sample from $g_t(p)$ is zero,
and hence, $[\psub^\top, p^{(M)}]^\top$ cannot be a sample from $g_t(p)$.
Therefore, sampling $\ptilt$ from $g_t(p)$ and Algorithm~\ref{alg:sample_p_exact} are equivalent.

%% file: tex_src/proof/relation_tspm_bpm.tex
In this appendix,
we discuss the relation between TSPM-Gaussian and BPM-TS~\citep{Vanchinathan14BPM}.

\noindent \textbf{Underlying Feedback Structure. }
Here, we discuss the underlying feedback structure behind TSPM-Gaussian and BPM-TS.

We first consider the underlying feedback structure behind BPM-TS.
In the following, we see that the feedback structure
\begin{align}
    y(t) 
    = 
    S_{i(t)} p + S_{i(t)} \ep \com \; \epsilon \sim \calN(0, \idmat_M)
\end{align}
induces the posterior distribution in BPM-TS.
Under this feedback structure,
we have $y(t) \sim \calN(S_{i(t)} p, S_{i(t)} S_{i(t)}^\top)$.

When we take the prior distribution $\pipr{p}$ as $\calN(0, \sigma_0^2 \idmat_M)$, 
the posterior distribution  for the opponent's strategy parameter can be written as
\begin{align}
    &
    \pipr{ p \relmiddle| \ev{i(s), y(s)}_{s=1}^t }  \nn
    &\propto
    \pipr{p} \prod_{s=1}^t \pipr{y(s) \relmiddle| i(s)\x p}  \nn
    &=
    \pipr{p} \prod_{s=1}^t \P_{y \sim  \calN(S_{i(s)} p, S_{i(s)} S_{i(s)}^\top)} \left\{ y = y(s) \right\}  \nn
    &=
    \exp \prn[\Big]{ -\frac{p^\top p}{2\sigma_0^2} }
    \prod_{s=1}^t \exp\prn[\Big]{  -\frac12 (y(s) - S_{i(s)}p)^\top (S_{i(s)} S_{i(s)}^\top)^{-1} (y(s) - S_{i(s)}p) }  \nn
    &=
    \exp
    \crl[\Big]{
        -\frac12
        \prn[\Big]{
            p^\top \prn[\Big]{\frac{1}{\sigma_0^2} \idmat_M + \sum_{s=1}^t S_{i(s)}^\top (S_{i(s)} S_{i(s)}^\top)^{-1} S_{i(s)} } p
        }  \nn
        &\qquad\qquad-
        2
        \prn[\Big]{
            \sum_{s=1}^t y(s)^\top (S_{i(s)} S_{i(s)}^\top)^{-1} S_{i(s)} p
        }
        +
        \text{(a term independent of $p$)}
    }  \nn
    &\propto
    \exp \crl[\Big]{ -\frac12 (p^\top \BtBPM p - 2 \btBPM^\top p) }  \nn
    &\propto
    \exp \crl[\Big]{ -\frac12 (p - {\BtBPM}^{-1} \btBPM)^\top \BtBPM (p - \BtBPM^{-1} \btBPM) } \com
\end{align}
where
\begin{align}
    \BtBPM
    &=
    \frac{1}{\sigma_0^2} \idmat_M + \sum_{s=1}^t S_{i(s)}^\top (S_{i(s)} S_{i(s)}^\top)^{-1} S_{i(s)}
    =
    B_{t-1}^{\BPM} + S_{i(t)}^\top (S_{i(t)} S_{i(t)}^\top)^{-1} S_{i(t)} \com \\
    \btBPM
    &=
    \sum_{s=1}^t S_{i(s)}^\top (S_{i(s)} S_{i(s)}^\top)^{-1} y(s)
    =
    b_{t-1}^{\BPM} + S_{i(t)}^\top (S_{i(t)} S_{i(t)}^\top)^{-1} y(t) \per
\end{align}
Therefore, the posterior distribution $\pipr{p \mid \ev{i(s), y(s)}_{s=1}^t}$ is
\begin{align}
    \frac{1}{\sqrt{(2\pi)^M | \BtBPM^{-1} | }} 
    \exp\crl[\Big]{ -\frac12 (p - \BtBPM^{-1} \btBPM)^\top \BtBPM (p - \BtBPM^{-1} \btBPM) } \per
\end{align}
and this distribution indeed corresponds to the posterior distribution 
in BPM-TS~\citep{Vanchinathan14BPM} with $\BtBPM = \Sigma_t^{-1}$.

Using the same argument, we can confirm that the feedback structure
\begin{align}
    y_t 
    =
    S_i p + \ep \com \; \epsilon \sim \calN(0, \idmat_M) \per
\end{align}
induces
\begin{align}
	\bar{g}_t(p)
	\coloneqq
	\frac{1}{\sqrt{(2\pi)^M|B_{t}^{-1}|}}
	\exp\prn[\Big]{
	-
	\frac12\normx{p-B_t^{-1} b_t}_{B_{t}}^2
	} \com
\end{align}
which corresponds to the posterior distribution for TSPM in linear partial monitoring.

\noindent \textbf{Covariances in TSPM-Gaussian and BPM-TS. }
In the linear partial monitoring, TSPM assumes noise with covariance $I_M$, 
which is compatible with the fact that the discrete setting can be regarded as linear PM with $I_M$-sub-Gaussian noise. 
On the other hand, BPM-TS assumes covariance $S_i S_i^\top$,
and in general $I_M \preceq S_i S_i^\top$ holds.
Therefore, BPM-TS assumes unnecessarily larger covariance, which makes learning slow down.

%% file: tex_src/proof/preliminaries_for_regret_analysis.tex
In this appendix, we give some technical lemmas, which are used for the derivation of the regret bound
in Appendix~\ref{sec:proof_of_regret_upper_bound}.
Here, we write $X \succeq Y$ to denote $X - Y \succeq 0$.
For $a, b \in \R$, let $a \wedge b$ be $a$ if $a \leq b$ otherwise $b$,
and $a \vee b$ be $b$ if $a \leq b$ otherwise $a$.
We use $h(a) \coloneqq \P_{X\sim\chi^2_M}\left\{ X \geq a \right\}$ to
evaluate the behavior of the posterior samples,
where $\chi^2_M$ is the chi-squared distribution with $M$ degree of freedom.

\subsection{Basic Lemmas}
\begin{fact}[Moment generating function of squared-Gaussian distribution]
    \label{fact:mgf_sqGauss}
    Let $X$ be the random variable following the standard normal distribution.
    Then, the moment generating function of $X^2$ is
    $\Expect{\exp(\xi X^2)} = (1-2\xi)^{-1/2}$ for $\xi < 1/2$.
\end{fact}

\begin{lemma}[Chernoff bound for chi-squared random variable]
    \label{lem:chisq_chernoff}
    Let $X$ be the random variable following the chi-squared distribution with $k$ degree of freedom.
    Then, for any $a \geq 0$ and $0 \le \xi < 1/2$,
    \begin{align}
        \pr{X \geq a}
        \leq
        \e^{-\xi a} (1 - 2\xi)^{-\frac{k}{2}}.
    \end{align}
\end{lemma}
\begin{proof}
    By Markov's inequality,
    the LHS can be bounded as
    \begin{align}
        \pr{X \geq a}
        &=
        \pr{\sum_{i=1}^k X_i^2 \geq a} \; (X_1, \ldots, X_k \iid \calN(0, 1))  \nn
        &=
        \pr{\exp\prn[\Big]{\xi \sum_{i=1}^k X_i^2} \geq \exp(\xi a)}   \nn
        &\leq
        \e^{-\xi a} \prn[\Big]{ \Expect{\e^{\xi X_1^2}} }^k
        \By{Markov's ineq.} \nn
        &=
        \e^{-\xi a} (1-2\xi)^{-\frac{k}{2}}
        \By{Fact~\ref{fact:mgf_sqGauss}}
        \com
    \end{align}
    which completes the proof.
\end{proof}

\subsection{Property of Strong Local Observability}
Recall that $\Delta_{i} = ( L_i - L_1 )^\top p^* > 0$ for $i \in [N]$, which is the difference of the expected loss of actions $i$ and $1$.
For this define
\begin{align} 
    \epsilon 
    \coloneqq 
    \left( \frac{1}{2\sqrt{A}} \min_{i \neq 1} \frac{\Delta_{i}}{\nrm{z_{1,i}}} \right)
    \wedge
    \left( \min_{p\in\calC_1^c} \frac43 \nrm{p - p^*} \right)
    \com
    \label{eq:def_epsilon}
\end{align}
which is used throughout the proof of this appendix and Appendix~\ref{sec:proof_of_regret_upper_bound}.
The following lemma provides the key property of the strong local observability condition.
\begin{lemma}
    \label{lem:loc_prop}
    For any partial monitoring game with strong local observability and
    $p \in \R^M$,
    any of the conditions 1--3 in the following is not satisfied:
    \begin{enumerate}
        \item $L_1^\top p > L_k^\top p$ \; (Worse action $k$ looks better under $p$.)
        \item $\nrm{S_1 p - S_1 p^*} \leq \epsilon$  
        \item $\nrm{S_k p - S_k p^*} \leq \epsilon$  \per 
    \end{enumerate}
\end{lemma}
\begin{proof}
    We prove by contradiction.
    Assume that there exists $p \in \R^M$ such that
    conditions 1--3 are simultaneously satisfied.

    Now, by the conditions 2 and 3, we have
    \begin{align}
        \begin{split}
            | S_1 p - S_1 p^* | &\preceq \epsilon \onemat_{A} \com \\
            | S_k p - S_k p^* | &\preceq  \epsilon \onemat_{A} \per
        \end{split}
    \end{align}
    Here, $|\cdot|$ is the element-wise absolute value,
    and $\preceq$ means that the inequality $\leq$ holds for each element.
    Therefore,
    \begin{align}
        \left|
        \left(
           \begin{array}{c}
               S_1 \\
               S_k
           \end{array}
        \right)
        (p - p^*)
        \right|
        \preceq \epsilon \onemat_{2A} \per
        \label{eq:element_wise_cond}
    \end{align}
    On the other hand, by the strong local observability condition,
    for any $k \neq 1$,
    there exists $z_{1,k} \neq 0 \in \R^{2A}$ 
    such that 
    \begin{align}
         (L_1 - L_k)^\top
         =
         z_{1,k}^\top \left(
            \begin{array}{c}
                S_1 \\
                S_k
            \end{array}
            \right) \per
            \label{eq:loc}
    \end{align}
    Now, we have
    \begin{align}
        &
        z_{1,k}^\top
        \left(
           \begin{array}{c}
               S_1 \\
               S_k
           \end{array}
        \right)
        (p - p^*)  \nn
        &\leq
        \nrm{z_{1,k}}
        \left\|
            \left(
               \begin{array}{c}
                   S_1 \\
                   S_k
               \end{array}
            \right)
            (p - p^*)
        \right\|
        \By{Cauchy-Schwarz ineq.}  \nn
        &\leq
        \sqrt{2A} \epsilon \nrm{z_{1,k}}
        \By{Eq.~\eqref{eq:element_wise_cond}} \com
        \label{eq:loc_prop_upper}
    \end{align}
    and
    \begin{align}
        &
        z_{1,k}^\top \left(
           \begin{array}{c}
               S_1 \\
               S_k
           \end{array}
           \right)
           (p - p^*)  \nn
        &=
        (L_1 - L_k)^\top (p - p^*)
        \By{Eq.~\eqref{eq:loc}}  \nn
        &=
        (L_1 - L_k)^\top p + (L_k - L_1)^\top p^*  \nn
        &\geq
        \Delta_{k}
        \By{Condition 1 \& def. of $\Delta_{k}$}  \per
        \label{eq:loc_prop_lower}
    \end{align}
    Therefore, from~\eqref{eq:loc_prop_upper} and~\eqref{eq:loc_prop_lower}, we have
    \begin{align}
        \Delta_{k}
        \leq
        \sqrt{2A} \epsilon \nrm{z_{1,k}}  \per
    \end{align}
    This inequality does not hold for all $k \neq 1$ 
    for the predefined value of $\ep$, since we have
    \begin{align}
        \epsilon \le \frac{1}{2\sqrt{A}} \min_{k \neq 1} \frac{\Delta_{k}}{\nrm{z_{1,k}}} \per
    \end{align}
    Therefore, the proof is completed by contradiction.
\end{proof}
\begin{remark}
    The similar result holds when the optimal action $1$ is
    replaced with action $j \neq k$ such that $\Delta_{j,k} \coloneqq (L_j - L_k)^\top p^* > 0$
    by taking $\ep$ satisfying
    \begin{align}
        \epsilon \le \frac{1}{2\sqrt{A}} \min_{j \neq k: \Delta_{j,k} > 0} \frac{\Delta_{j,k}}{\nrm{z_{j,k}}} \per
    \end{align}
\end{remark}

From Lemma~\ref{lem:loc_prop}, we have the following corollary.
\begin{corollary}
    \label{cor:loc_prop_cor1}
    For any $p\in\R^M$ satisfying $p \in \Ci$ and $\nrm{S_1 p - S_1 p^*} \leq \epsilon$,
    we have
    \begin{align}
        \nrm{S_i p - S_i p^*} > \ep \per
    \end{align}
\end{corollary}
\begin{proof}
    Note that $p \in \Ci$ is equivalent to $(L_1 - L_i)^\top p^* > 0$ for any $i \neq 1$.
    Therefore, the result directly follows from Lemma~\ref{lem:loc_prop}.
\end{proof}

The next lemma is the property of Mahalanobis distance corresponding to $\gbart(p)$.
\begin{lemma}
    \label{lem:Mahalanobis_concentration}
    Define
    $\calT_i \coloneqq \{p \in \R^M : \nrm{ S_i p - S_i p^* } > \ep\}$.
    Assume that
    $N_i(t) \geq n_i$,
    $\nrm{S_i \phatt - S_i p^*} \le \ep/4$.
    Then, for any $0 \le \xi < 1/2$ 
    \begin{align}
        h\prn[\bigg]{ \inf_{p\in\calT_i} \nrm{ B_t^{1/2} (p - \phatt ) }^2}
        \leq
        \exp \prn[\Big]{ -\frac{9}{16} \xi \epsilon^2 n_i } (1 - 2\xi)^{-M/2} \per
    \end{align}

\end{lemma}
\begin{proof}
    To bound the LHS of the above inequality,
    we bound $\nrm{ B_t^{1/2} (p - \phatt )}^2$ from below for $p \in \calT_i$.
    Using the triangle inequality and the assumptions, we have
    \begin{align}
        \nrm{ S_i (p - \phatt) }
        &\ge
        \nrm{ S_i p - S_i p^* }  - \nrm{ S_i \phatt - S_i p^* }  \nn
        &>
        \ep - \ep/4
        >
        0 \per
        \label{eq:diff_p_phat}
    \end{align}
    Therefore, we have
    \begin{align}
        \nrm{ B_t^{1/2} (p - \phatt )}^2
        &\ge
        \sum_{k\in[N]} N_k(t) \nrm{ S_k (p - \phatt) }^2
        \By{def. of $B_t$} \nn
        &\geq
        n_i \nrm{ S_i (p - \phatt) }^2
        \Sincem{N_i(t) \ge n_i} \nn
        &>
        \frac{9}{16} \epsilon^2 n_i
        \By{Eq.~\eqref{eq:diff_p_phat}} \per
    \end{align}
    By the Chernoff bound for a chi-squared random variable in Lemma~\ref{lem:chisq_chernoff},
    we now have
    \begin{align}
        h(a)
        \leq
        \e^{-\xi a} (1 - 2\xi)^{-M/2} \com
    \end{align}
    for any $a \geq 0$ and $0 \le \xi < 1/2$.
    Hence, 
    using the fact that $\nrm{ B_t^{1/2} (p - \phatt ) }^2$ follows the chi-squared distribution with $M$ degree of freedom,
    we have
    \begin{align}
        h\prn[\bigg]{\inf_{p\in\calT_i} \nrm{ B_t^{1/2} (p - \phatt ) }^2}
        &\leq
        h\prn[\Big]{\frac{9}{16} \epsilon^2 n_i} \nn
        &\leq
        \exp \prn[\Big]{ -\frac{9}{16} \xi \epsilon^2 n_i } (1 - 2\xi)^{-M/2} \com
    \end{align}
    which completes the proof.
\end{proof}

\subsection{Statistics of Uninterested Actions}
For any $k\neq i$ and $n_k\in[T]$,
define
\begin{align}
    \Znk &\coloneqq n_k \nrm{\qknk - S_k p^* }^2 \com  \\
    \Zbsli &\coloneqq \sum_{k \neq i} \max_{n_k\in[T]} \Znk \per
\end{align}
In this section, we bound $\Expect{\Zbsli}$ from above.
Note that $\Zbsli$ is independent of the randomness of Thompson sampling.

\begin{lemma}[Upper bound for the expectation of $\Zbsli$]
    \label{lem:expectation_Z}
    \begin{align}
        \Expect{\Zbsli} 
        \le
        4N \prn[\Big]{\log T + \frac{A}{2} \log 2 + 1 }  \per
    \end{align}
\end{lemma}
\begin{proof}
    Recall that in linear partial monitoring,
    the feedback $y(t) \in \R^A$ for action $k$ is given as
    \begin{align}
        y_t = S_k p^* + \epsilon \com \; \epsilon \sim \calN(0, I_A)
    \end{align}
    at round $t\in[T]$, 
    Therefore, $y(t) - S_k p^* \sim \calN(0, I_A)$.
    Since $\qknk = \frac{1}{n_k} \sum_{s\in[T]:i(s)=k} y(s)$ for any $n_k\in[T]$,
    we have
    \begin{align}
        q_{k, n_k} - S_k p^* = \frac{1}{n_k} \sum_{s\in[T]:i(s)=k} (y(s) - S_k p^*) \sim \calN(0, I_A/n_k) \per
    \end{align}
    Therefore,
    \begin{align}
        \sqrt{n_k} (\qknk - S_k p^*) \sim \calN(0, I_A) \com
    \end{align}
    and thus
    \begin{align}
        n_k \nrm{ \qknk - S_k p^* }^2 = \nrm{ \sqrt{n_k} (\qknk - S_k p^*) }^2 \sim \chi^2_A \per
    \end{align}
    Therefore, for any $0 \le \xi < 1/2$,
    \begin{align}
        \Expect{\max_{n_k\in[T]} \Znk}
        &=
        \int_0^\infty \pr{ \max_{n_k\in[T]} \Znk \geq x } \d x  \nn
        &\leq
        \int_0^\infty \left[ 1 \wedge T \cdot \pr{ Z_1 \geq x } \right] \dx
        \By{the union bound}  \nn
        &\leq
        \int_0^\infty \left[ 1 \wedge T \cdot \e^{-\xi x} (1-2\xi)^{-\frac{A}{2}} \right] \dx
        \By{$Z_1 \sim \chi^2_(A)$ and Lemma~\ref{lem:chisq_chernoff}}  \nn
        &=
        \int_0^{x^*} \d x
        +
        \int_{x^*}^\infty  T \cdot \e^{-\xi x} (1-2\xi)^{-\frac{A}{2}} \dx  \nn
        &\leq
        x^*
        +
        T \cdot \int_{x^*}^\infty \e^{-\xi x} (1-2\xi)^{-\frac{A}{2}} \dx  \nn
        &=
        x^*
        +
        T (1-2\xi)^{-\frac{A}{2}} \left[ -\frac{1}{\xi} \e^{-\xi x} \right]_{x^*}^\infty  \nn
        &=
        \frac{1}{\xi} \crl[\Big]{\log T - \frac{A}{2} \log(1-2\xi) + 1 } \com
    \end{align}
    where
    $x^* \coloneqq \frac{1}{\xi} \ev{ \log T  - \frac{A}{2} \log(1-2\xi) }$.
    Therefore, taking $\xi = 1/4$, we have
    \begin{align}
        \Expect{\Zbsli}
        &=
        \Expect{ \sum_{k \neq i} \max_{n_k\in[T]} \Znk }  \nn
        &\leq
        \sum_{k \neq i} \Expect{ \max_{n_k\in[T]} \Znk}  \nn
        &\leq
        (N-1) \frac{1}{\xi} \crl[\Big]{\log T - \frac{A}{2} \log(1-2\xi) + 1 }  \nn
        &\leq
        4N \prn[\Big]{\log T + \frac{A}{2} \log 2 + 1 }  \com
    \end{align}
    which completes the proof.
\end{proof}

\subsection{Mahalanobis Distance Process}
Discussions in this section are essentially very similar to~\citet[Lemma 11]{Abbasi11improved},
but their results are not directly applicable and we give the full derivation
for self-containedness.
To maximize the applicability
here we only assume sub-Gaussian noise rather than a Gaussian one.

Let
$\ep_t$
be zero-mean
$1$-sub-Gaussian random variable,
which satisfies
\begin{align}
	\Expect{\e^{\lambda^{\top} \ep_t}} \le \e^{-\frac{\normx{\lambda}^2}{2}}
\end{align}
for any $\lambda\in\R^M$.
\begin{lemma}\label{lem:subgaussian}
	For any vector $v\in\R^M$ and positive definite matrix $V\in\R^{M\times M}$
	such that $V\succ I$,
	\begin{align}
		\ExpectX{\ep_t}{\e^{\frac{\normx{\ep_t+v}_{V^{-1}}^2}{2}}}
		&\le
		\frac{\sqrt{|V|}}{\sqrt{|V-I|}}
		\e^{\frac{1}{2}
		v^{\top}(V-I)^{-1}v
		}\per
	\end{align}
\end{lemma}

\begin{proof}
	For any $x\in\R^M$
	\begin{align}
		\ExpectX{\lambda\sim \calN(0,V^{-1})}{\e^{\lambda^{\top}x}}
		&=
		\e^{\frac{\normx{x}_{V^{-1}}^2}{2}}\per
	\end{align}
	Therefore, by letting
	$x=\ep_t+v$
	we see that
	\begin{align}
		\e^{\frac{\normx{\ep_t+v}_{V^{-1}}^2}{2}}
		&=
		\ExpectX{\lambda\sim \calN(0,V^{-1})}{\e^{\lambda^{\top}(\ep_t+v)}} \per
	\end{align}
	As a result, by the definition of sub-Gaussian random variables,
	we have
	\begin{align}
		\ExpectX{\ep_t}{\e^{\frac{\normx{\ep_t+v}_{V^{-1}}^2}{2}}}
		&=
		\ExpectX{\lambda\sim \calN(0,V^{-1})}{
		    \ExpectX{\ep_t}{\e^{\lambda^{\top}(\ep_t+v)}}
		}
		\nn
		&=
		\ExpectX{
		    \lambda\sim \calN(0,V^{-1})}{\e^{\lambda^{\top}v}\ExpectX{\ep_t}{\e^{\lambda^{\top}\ep_t}}
		}
		\nn
		&\le
		\ExpectX{\lambda\sim \calN(0,V^{-1})}{\e^{\lambda^{\top}v}\e^{\normx{\lambda}^2/2}}
		\nn
		&=
		\frac{1}{(2\pi)^{d/2}\sqrt{|V^{-1}|}}\int
		\e^{\lambda^{\top}v}\e^{\normx{\lambda}^2/2}\e^{-\normx{\la}_{V}^2/2}
		\d \la
		\nn
		&=
		\frac{1}{(2\pi)^{d/2}\sqrt{|V^{-1}|}}\int
		\e^{-\frac{1}{2}
		\pax{
		\la^{\top}(V-I)\la-2v^\top\la
		}
		}\d \la\nn
		&=
		\frac{\sqrt{|V-I|}}{(2\pi)^{d/2}\sqrt{|V^{-1}||V-I|}}\int
		\e^{-\frac{1}{2}
		\pax{
		(\la-(V-I)^{-1}v)^{\top}(V-I)(\la-(V-I)^{-1}v)
		-
		v^{\top}(V-I)^{-1}v
		}
		}\d \la\nn
		&=
		\frac{\sqrt{|V|}}{\sqrt{|V-I|}}
		\e^{\frac{1}{2}
		v^{\top}(V-I)^{-1}v
		}\per
	\end{align}
\end{proof}

\begin{lemma}\label{lem:martingale}
	\begin{align}
		\Ex{\exp\pax{
		\frac12\pax{
		\normx{\hat{p}_{t}-p^*}_{B_{t}}^2
		-
		\normx{\hat{p}_{t-1}-p^*}_{B_{t-1}}^2
		}}
		\relmiddle| 
		\hat{p}_{t-1}\x B_{t-1}\x S_{i(t-1)}}
		&\le
		\sqrt{\frac{|B_{t}|}{|B_{t-1}|}}\per
	\end{align}
\end{lemma}

\begin{proof}
    Let $Z_t \coloneqq -\la p^*+\sum_{s=1}^{t} S_{i(s)}^\top \ep_s$,
    and we have
    \begin{itemize}
    	\item $B_{t}=\lambda I+\sum_{s=1}^{t}S_{i(s)}^{\top}S_{i(s)}$,
    	\item $b_{t}=\sum_{s=1}^{t} S_{i(s)}^\top y(s)=B_t p^*+Z_t$,
    	\item $\hat{p}_{t}=B_t^{-1}b_t=p^*+B_t^{-1}Z_t$.
    \end{itemize}
	
	In the following we omit the conditioning on
	$(\hat{p}_{t-1}\x B_{t-1}\x S_{i(t-1)})$
	for notational simplicity.

	Let us define
	$C_t \coloneqq S_{i(t)}B_{t-1}S_{i(t)}^\top$ and $d_t \coloneqq S_{i(t)}B_{t-1}^{-1}Z_{t-1}=S_{i(t)}(\hat{p}_t-p^*)$.
	Then, using the Sherman-Morrison-Woodbury formula we have
	
	\begin{align}
		\lefteqn{
		\normx{\hat{p}_{t}-p^*}_{B_{t}}^2
		-
		\normx{\hat{p}_{t-1}-p^*}_{B_{t-1}}^2
		}\nn
		&=
		Z_{t}^{\top} B_t^{-1} Z_{t}
		-
		Z_{t-1}^{\top} B_{t-1}^{-1} Z_{t-1}
		\nn
		&=
		(Z_{t-1}^{\top} + \ep_{t}^\top S_{i(t)})
		(B_{t-1}^{-1} - B_{t-1}^{-1} S_{i(t)}^\top (I+S_{i(t)} B_{t-1}^{-1} S_{i(t)}^\top)^{-1} S_{i(t)}B_{t-1}^{-1})
		(Z_{t-1} + S_{i(t)}^\top\ep_{t})
		-
		Z_{t-1}^\top B_{t-1}^{-1} Z_{t-1}
		\nn
		&=
		(Z_{t-1}^{\top}+\ep_{t}^{\top}S_{i(t)})
		B_{t-1}^{-1}
		(Z_{t-1}+S_{i(t)}^\top\ep_{t})
		-
		Z_{t-1}^{\top}B_{t}^{-1}Z_{t-1}
		\nn
		&\quad-
		(Z_{t-1}^{\top} + \ep_{t}^{\top}S_{i(t)})
		B_{t-1}^{-1}S_{i(t)}^\top(I+S_{i(t)} B_{t-1}^{-1} S_{i(t)}^\top)^{-1} S_{i(t)}B_{t-1}^{-1}
		(Z_{t-1} + S_{i(t)}^\top\ep_{t})
		\nn
		&=
		\ep_{t}^{\top}S_{i(t)}
		B_{t-1}^{-1}
		S_{i(t)}^\top\ep_{t}
		+
		2Z_{t-1}^{\top}
		B_{t-1}^{-1}
		S_{i(t)}^\top\ep_{t}
		\nn
		&\quad-
		(Z_{t-1}^{\top}+\ep_{t}^{\top}S_{i(t)})
		B_{t-1}^{-1}S_{i(t)}^\top (I + S_{i(t)}B_{t-1}^{-1}S_{i(t)}^\top)^{-1} S_{i(t)} B_{t-1}^{-1}
		(Z_{t-1}+S_{i(t)}^\top\ep_{t})
		\nn
		&=
		\ep_{t}^{\top}C_t\ep_{t}
		+
		2d_{t}^{\top}
		\ep_{t}
		-
		(d_{t}^{\top}+\ep_{t}^{\top}C_t)
		(I+C_t)^{-1}
		(d_{t}+C_t\ep_{t})
		\nn
		&=
		\ep_{t}^{\top}C_t(I-(I+C_t)^{-1}C_t)\ep_{t}
		+
		2d_{t}^{\top}
		(I-(I+C_t)^{-1}C_t)
		\ep_{t}
		-
		d_{t}^{\top}
		(I+C_t)^{-1}
		d_{t}\nn
		&=
		\ep_{t}^{\top}C_t(I+C_t)^{-1}\ep_{t}
		+
		2d_{t}^{\top}
		(I+C_t)^{-1}
		\ep_{t}
		-
		d_{t}^{\top}
		(I+C_t)^{-1}
		d_{t}
		\nn
		&=
		\normx{\ep_{t} + C_t^{-1}d_t}_{C_t(I+C_t)^{-1}}^2
		-
		d_t^\top(I + C_t)^{-1}C_t^{-1}d_t
		-
		d_{t}^{\top}
		(I+C_t)^{-1}
		d_{t}
		\nn
		&=
		\normx{\ep_{t} + C_t^{-1}d_t}_{C_t(I + C_t)^{-1}}^2
		-
		d_t^\top(I + C_t)^{-1} (I + C_t^{-1})d_t
		\per
	\end{align}
	
	Therefore,
	Lemma \ref{lem:subgaussian} with
	$V \coloneqq \pax{C_t(I+C_t)^{-1}}^{-1}=(I+C_t)C_t^{-1}\x v \coloneqq C_t^{-1}d_t$
	yields
	\begin{align}
		\lefteqn{
		\Ex{\exp\pax{
		\frac12\pax{
		\normx{\hat{p}_{t}-p^*}_{B_{t}}^2
		-
		\normx{\hat{p}_{t-1}-p^*}_{B_{t-1}}^2
		}}}
		}\nn
		&\le
		\frac{\sqrt{|(I+C_t)C_t^{-1}|}}{\sqrt{|(I+C_t)C_t^{-1}-I|}}
		\e^{\frac{1}{2}
		d_t^{\top}C_t^{-1}((I+C_t)C_t^{-1}-I)^{-1}C_t^{-1}d_t
		}
		\e^{
		-\frac12 d_t^\top (I+C_t)^{-1}(I+C_t^{-1})d_t
		}
		\nn
		&\le
		\frac{\sqrt{|(I+C_t)C_t^{-1}|}}{\sqrt{|C_t^{-1}|}}
		\e^{\frac{1}{2}
		d_t^{\top}C_t^{-1}(C_t^{-1})^{-1}C_t^{-1}d_t
		}
		\e^{
		-\frac12 d_t^\top (I+C_t)^{-1}(I+C_t^{-1})d_t
		}
		\nn
		&=
		\sqrt{|(I+C_t)|}
		\nn
		&=
		\sqrt{\frac{|B_{t}|}{|B_{t-1}|}} \com
	\end{align}
	where 
	see, \eg~\citet[Lemma 11]{Abbasi11improved} for the last equality.
\end{proof}

\subsection{Norms under Perturbations}
In the following two lemmas,
we give some analysis of norms under perturbations.

\begin{lemma}\label{lem:minimax}
	Let $A$ be a positive definite matrix.
	Let $a\in\R^d$
	and $\ep>0$ be such that
	$\ep<\normx{a}/3$.
	Then
	\begin{align}
		\min_{x:\normx{x}\le 2\ep}\max_{x':\normx{x'}\le \ep}\brx{(a+x+x')^\top A (a+x+x')}
		&=
		\min_{x'':\normx{x''}\le \ep}
		\brx{(a+x'')^\top A (a+x'')}\per\label{norm_opt1}
	\end{align}
\end{lemma}
\begin{proof}
	By considering the Lagrangian multiplier we see that
	any stationary point
	of the function
	$(a+x'')^\top A (a+x'')$ over
	$\{(x,x'):\normx{x}\le 2\ep\x \normx{x'}\le \ep\}$
	satisfies
	\begin{align}
		&A (a+x+x')-\la_1 x=0\com\nn
		&A (a+x+x')-\la_2 x'=0\com\nn
		&x^\top x =4\ep^2\com\nn
		&x'^\top x' =\ep^2\com
		\label{cond_norm}
	\end{align}
	and therefore
	$\la_1 x=\la_2x'$.
	Considering the last two conditions of \eqref{cond_norm}
	we have
	$\la_2=\pm 2\la_1$, implying that
	\begin{align}
		x'=-(3A-2\la_1 I)Aa\label{cond_lagrange1}
	\end{align}
	or
	\begin{align}
		x'=(A-2\la_1 I)Aa\label{cond_lagrange2}
	\end{align}
	for $\la_1$ satisfying
	$x'^{\top}x'=\ep^2$.

	Note that it holds for any positive definite matrix $B$ that
	\begin{align}
		\diff{^2}{\la^2}a(B+\la I)^{-2}a
		&=
		a(B+\la I)^{-4}a=
		\normx{(B+\la I)^{-2}a}^2\com
	\end{align}
	which is positive almost everywhere, meaning that
	$a(B+\la I)^{-2}a$ is strictly convex with respect to $\la\in \R$.
	Therefore,
	there exists at most two $\la'_1$'s satisfying
	\eqref{cond_lagrange1} and $x'^\top x'=\ep^2$, and
	there exists at most two $\la'_1$'s satisfying
	\eqref{cond_lagrange2} and $x'^\top x'=\ep^2$.
	In summary,
	there at most four stationary points
	of
	$(a+x'')^\top A (a+x'')$ over
	$\{(x,x'):\normx{x}\le 2\ep\x \normx{x'}\le \ep\}$.

	On the other hand,
	two optimization problems
	\begin{align}
		\min_{x:\normx{x}\le 2\ep}\min_{x':\normx{x'}\le \ep}\brx{(a+x+x')^\top A (a+x+x')}
		&=
		\min_{x'':\normx{x''}\le 3\ep}
		\brx{(a+x'')^\top A (a+x'')}
	\end{align}
	and
	\begin{align}
		\max_{x:\normx{x}\le 2\ep}\max_{x':\normx{x'}\le \ep}\brx{(a+x+x')^\top A (a+x+x')}
		&=
		\max_{x'':\normx{x''}\le 3\ep}
		\brx{(a+x'')^\top A (a+x'')}
	\end{align}
	can be easily solved by
	an elementary calculation
	and the optimal values are equal to those
	corresponding to \eqref{cond_lagrange1}.

	Therefore, the optimal solutions of the two minimax problems
	\begin{align}
		\max_{x:\normx{x}\le 2\ep}\min_{x':\normx{x'}\le \ep}\brx{(a+x+x')^\top A (a+x+x')}\label{lagrange_maxmin}
	\end{align}
	and
	\begin{align}
		\min_{x:\normx{x}\le 2\ep}\max_{x':\normx{x'}\le \ep}\brx{(a+x+x')^\top A (a+x+x')}\label{lagrange_minmax}
	\end{align}
	correspond to two points corresponding to \eqref{cond_lagrange2}.

	We can see again from an elementary calculation that the optimal solutions for
	two optimization problems
	\begin{align}
		&\min_{x'':\normx{x''}\le \ep}\brx{(a+x'')^\top A (a+x'')}\nn
		&\max_{x'':\normx{x''}\le \ep}\brx{(a+x'')^\top A (a+x'')}
	\end{align}
	have the same necessary and sufficient conditions as \eqref{cond_lagrange2}
	and we complete the proof by noticing that \eqref{lagrange_maxmin} is less than \eqref{lagrange_minmax}.
\end{proof}

\begin{lemma}\label{norm_opt2}
	Let
	$A\succeq nS_1^\top S_1$ be a positive-definite matrix
	with minimum eigenvalue at least $\lambda>0$.
    Then, for any $\phat\in\R^d$
	and $\ep>0$ satisfying
	$\ep< \normx{\phat-p^*}/3$,
	\begin{align}
		\normx{\phat-p^*}_{A}^2-
		\inf_{p: \normx{p-p^*}\le 2\ep}\sup_{p': \normx{p'-p}\le \ep}
		\normx{p'-\phat}_A^2
		&\ge
		\ep\sqrt{n\la}
		\normx{S_1 (\phat-p^*)}\per
	\end{align}
\end{lemma}
\begin{proof}
	Let $a=\phat-p^*$.
	By Lemma \ref{lem:minimax}, we have
	\begin{align}
		\lefteqn{
		\inf_{p: \normx{p-p^*}\le 2\ep}\sup_{p': \normx{p'-p}\le \ep}
		\normx{p'-\phat}_A^2
		}\nn
		&=
		\inf_{x: \normx{x}\le 2\ep}\sup_{x':\normx{x'}\le \ep}
		\normx{a+x+x'}_A^2
		\nn
		&=
		\inf_{x: \normx{x}\le \ep}
		\normx{a+x}_A^2\per
	\end{align}
	Now define $\calS_{\ep',A}=\{x:\normx{x}_A\le \ep'\}$.
	Then,
	we see that
	$\calS_{\ep \sqrt{\la},A}\subset \{x:\normx{x}\le \ep\}$.
	Therefore, an elementary calculation
	using the Lagrange multiplier technique shows
	\begin{align}
		\inf_{x:\normx{x}\le \ep}
		\normx{p'-\phat}_A^2
		&\le
		\inf_{x\in \calS_{\ep\sqrt{\la},A}}
		\normx{p-\phat}_A^2
		\nn
		&=
		\pax{\normx{a}_A-\ep\sqrt{\la}}^2\per
	\end{align}
	As a result, we see that
	\begin{align}
		\normx{p^*-\phat}_{A}^2-
		\inf_{p: \normx{p-p^*}\le 2\ep}\sup_{p': \normx{p'-p}\le \ep}
		\normx{p'-\phat}_A^2
		&\ge
		\normx{a}_A^2-
		\pax{\normx{a}_A-\ep\sqrt{\la}}^2\nn
		&=
		\ep\sqrt{\la}
		\pax{\normx{a}_A+\normx{a}_A-\ep\sqrt{\la}}
		\nn
		&\ge
		\ep\sqrt{\la}
		\pax{\normx{a}_A+\normx{a}\sqrt{\la}-\ep\sqrt{\la}}
		\nn
		&=
		\ep\sqrt{\la}
		\pax{\normx{a}_A+\sqrt{\la}(\normx{a}-\ep)}
		\nn
		&\ge
		\ep\sqrt{\la}
		\normx{a}_A\nn
		&\ge
		\ep\sqrt{n\la}
		\normx{S_1 a}\per
	\end{align}
\end{proof}

For the subsets of $\R^n$, $\calX$ and $\calY$,
let $\calX + \calY \coloneqq \ev{x + y : x\in\calX\x y\in\calY }$ be the Minkowski sum,
and let $B^n_r(p)$ be the $n$-dimensional Euclidian ball of radius $r$ at point $p \in \R^n$
(the superscript $n$ can be omitted when it is clear from context).
We also let $\eppr$ be
\begin{align}\label{def:ep_prime}
    \ep^\prime \coloneqq
    \frac
        {\ep}
        {\prn[\Big]{16 \max_{i\in[N]} \nrm{S_i} } 
            \vee 
        \prn[\Big]{\frac{1}{\sqrt{A}} \max_{i\in[N]} \frac{\nrm{L_i - L_1}}{\nrm{z_{1,i}}} }
        } \com
\end{align}
which is also used throughout the proof of this appendix and Appendix~\ref{sec:proof_of_regret_upper_bound} as $\ep$ in~\eqref{eq:def_epsilon}.
\begin{theorem}\label{thm_delta}
    Let $\ep''\in(0,\ep)$ be a constant for $\ep$ defined in~\eqref{eq:def_epsilon}.
	Let 
	$\phat \in \Ck + B_{\ep'}^d(0)$ be satisfying
	$\normx{S_k(\phat - p^*)}\le \ep'' $.
	Then, there exists
	$\de>0$
	satisfying
	for any 
	$n \ge 0$
	and
	$A\succeq n S_{1}^\top S_{1}+\la I$
	that
	\begin{align}
		\normx{p^*-\phat}_{A}^2-
		\inf_{p: \normx{p-p^*}\le 2\ep}\sup_{p': \normx{p'-p}\le \ep}
		\normx{p'-\phat}_A^2
		&\ge
		\ep\de\sqrt{\la n}
		\per
	\end{align}
\end{theorem}

\begin{proof}
    Recall that $\ep'' < \ep  \le \min_{p\in\calC_1^c}\normx{p-p^*}/3$.
	It is enough from Lemma~\ref{norm_opt2}
    to prove that 
	\begin{align}
		\delta 
		\coloneqq 
		\min_{
		    \phat \in \{p\in\Ck + B^d_{\eppr}(0) : \normx{S_k(p-p^*)}\le \ep'' \}
		}
		\normx{S_1(\phat - p^*)}
		\label{def_delta}
	\end{align}
	is positive. 

	We prove by contradiction and
	the proof is basically same as that of Lemma~\ref{lem:loc_prop}
	but more general in the sense that
	the condition on $\phat$ is not $\phat\in\Ck$ but $\phat \in \Ck + B^d_{\eppr}(0)$.
    Assume that $\delta = 0$, that is,
	there exists $\phat \in \Ck + B^d_{\eppr}(0)$ 
	satisfying 
	$\normx{S_k(p-p^*)}\le \ep''\}$
	and
	$\nrm{S_1 (\phat - p^*)} = 0$.
	Note that $\nrm{S_1 (\phat - p^*)} = 0$ implies $\nrm{S_1 (\phat - p^*)} \le \ep''$.
	Therefore, we now have following conditions on $\phat$:
	\begin{itemize}
	    \item $\phat \in \Ck + B^d_{\eppr}(0)$
	    \item $\nrm{S_1 (\phat - p^*)} \le \ep''$
	    \item $\nrm{S_k (\phat - p^*)} \le \ep''$ .
	\end{itemize}

    Following the same argument as the proof of Lemma~\ref{lem:loc_prop}, 
    we have
    \begin{align}
        z_{1,k}^\top
        \left(
           \begin{array}{c}
               S_1 \\
               S_k
           \end{array}
        \right)
        (\phat - p^*)  
        \leq
        \sqrt{2A} \ep'' \nrm{z_{1,k}}  \per
        \label{eq:E_lb}
    \end{align}
On the other hand, since
    $\phat\in\Ck + B^d_{\eppr}(0)$ 
    we can take $\pbar\in\Ck$ such that $\nrm{\phat - \pbar} \le \eppr$.
    Hence,
    \begin{align}
        z_{1,k}^\top \left(
           \begin{array}{c}
               S_1 \\
               S_k
           \end{array}
           \right)
           (\phat - p^*) 
        &=
        (L_1 - L_k)^\top (\phat - p^*)   \nn
        &=
        - (L_k - L_1)^\top (\phat - \pbar)
        +
        (L_1 - L_k)^\top \pbar 
        + 
        (L_k - L_1)^\top p^*  
        \nn
        &\geq
        - (L_k - L_1)^\top (\phat - p^*)
        +
        \Delta_{k}  \per
        \By{$\pbar\in\Ck$ and def. of $\Delta_k$}
        \label{eq:E_ub}
    \end{align}
    From \eqref{eq:E_lb} and \eqref{eq:E_ub}, we have
    \begin{align}
        \Delta_{k} - (L_k - L_1)^\top (\phat - p^*) 
        \le 
        \sqrt{2A} \ep'' \nrm{z_{1,k}} \per
        \label{eq:E_ineq}
    \end{align}
    Now, the left hand side of \eqref{eq:E_ineq} is bounded from below as
    \begin{align}
        \Delta_{k} - (L_k - L_1)^\top (\phat - \pbar)
        &\ge
        \Delta_{k} - \nrm{L_k - L_1} \nrm{\phat - \pbar}  \nn
        &\ge
        \Delta_{k} - \nrm{L_k - L_1} \eppr  \nn
        &=
        \Delta_{k} 
        - 
        \nrm{L_k - L_1} 
        \frac{\ep}{\frac{1}{\sqrt{A}} \max_i \frac{\nrm{L_1 - L_i}}{\nrm{z_{1,i}}} }
        \nn
        &=
        \Delta_{k} 
        - 
        \nrm{L_k - L_1} 
        \frac{
            \frac{1}{2\sqrt{A}} \min_i \frac{\Delta_i}{\nrm{z_{1,i}}} 
        }{
            \frac{1}{\sqrt{A}} \max_i \frac{\nrm{L_1 - L_i}}{\nrm{z_{1,i}}} 
        }
        \nn
        &\ge
        \Delta_{k} - \Delta_k / 2  \per
    \end{align}
    On the other hand, using the definition of $\ep''$, 
    the right hand side of \eqref{eq:E_ineq} is bounded from above as
    \begin{align}
        \sqrt{2A} \ep'' \nrm{z_{1,k}} < \Delta_k / 2 \per       
    \end{align}
    Therefore, the proof is completed by contradiction.
\end{proof}

\subsection{Exit Time Analysis}
We next consider the exit time.
Let
$\calA_{t}$ be an event deterministic given $\calF_t$,
and $\calB_t$ be a random event
such that if $\calB_t$ occurred then $\calA_{t'}$ never occurs for
$t'=t+1,t+2,\dots$.
Let $P_t\x t=1,2,\dots,T$, be a stochastic process satisfying
$P_t\le \pr{\calB_t|\calF_t}$ a.s.~and
$P_t^{-1}$ is a supermartingale with respect to the filtration induced by $\calF_t$.

\begin{theorem}\label{thm_stop}
	Let $\tau$ be the stopping time defined as
	\begin{align}
		\tau=
		\begin{cases}
			\min\{t\in [T]:\calA_{t}\}&\mbox{\rm if $\calA_t$ occurs for some $t\in[T]$.}\\
			T+1& \mbox{\rm otherwise}.
		\end{cases}\label{def_tau_stop}
	\end{align}
	Then we almost surely have
	\begin{align}
		\Ex{\sum_{t=1}^{T}\ind{\calA_{t}}\relmiddle| \calF_{\tau}}
		\le
		\begin{cases}
			P_\tau^{-1}&\tau\le T,\\
			0 &\tau=T+1.
		\end{cases}
	\end{align}
\end{theorem}

We prove this theorem based on the following lemma.
\begin{lemma}\label{lem:stop}
	Let $(Q_i)_{i=1}^\infty \subset [0,1]$ be an arbitrary stochastic process
	such that $(Q_i^{-1})_{i=1}^\infty$ is a supermartingale with respect to a filtration $(\calG_i)_{i=1}^\infty$.
	Then, for any $\calG_0\subset \calG_1$,
	\begin{align}
		\Ex{
		\sum_{i=1}^{T}
		\prod_{j=1}^i(1-Q_j)
		\relmiddle| \calG_0
		}
		\le
		\Ex{Q_1^{-1}|\calG_0}-1
		\quad \mathrm{a.s.}
	\end{align}
\end{lemma}

\begin{proof}
	Let
	\begin{align}
		{N}_{k}((Q_i,\calG_i)_{i=1}^{\infty}\x \calG_0)
		&=
		\Ex{
		\sum_{i=1}^{k}
		\prod_{j=1}^i(1-Q_j)
		\relmiddle| \calG_0
		}\nn
		\overline{N}_{k}((Q_i,\calG_i)_{i=1}^{\infty}\x \calG_0)
		&=
		\Ex{
		\sum_{i=1}^{\infty}
		\prod_{j=1}^i(1-Q_j)
		\relmiddle| \calG_0
		}\quad \mbox{where $Q_j=Q_k$ for $j>k$.}
	\end{align}

	We show $\overline{N}_{k}((Q_i,\calG_i)_{i=1}^{\infty}\x \calG_0)\le \E[Q_1^{-1}|\calG_0]-1$ a.s.
	for any $(Q_i\x \calG_i)_{i=1}^\infty$, $\calG_0\subset \calG_1$ and $k\in\N$
	by induction.
	First,
	for $k=1$ the statement holds since
	\begin{align}
		\overline{N}_{1}((Q_i,\calG_i)_{i=1}^{\infty}\x \calG_0)
		&=
		\Ex{
		\sum_{i=1}^{\infty}
		\prod_{j=1}^i(1-Q_1)
		\relmiddle| \calG_0
		}\nn
		&=
		\Ex{
		Q_1^{-1}-1
		\relmiddle| \calG_0
		}
		\nn
		&=
		\Ex{Q_1^{-1}\relmiddle|\calG_0} - 1
	\end{align}

	Next, assume that the statement holds
	for all $(Q_i\x \calG_i)_{i=1}^k$, $\calG_0\subset \calG_1$ and $k \le k_0$.
	Then, we almost surely have
	\begin{align}
		\overline{N}_{k_0 + 1}((Q_i,\calG_i)_{i=1}^{\infty}\x \calG_0)
		&=
		\Ex{
		(1-Q_1)
		\Ex{
		1+\sum_{i=2}^{\infty}
		\prod_{j=2}^i(1-Q_j)\relmiddle| \calG_1
		}\relmiddle| \calG_0
		}\nn
		&=
		\Ex{
		(1-Q_1)
		(1+\overline{N}_{k_0}((Q_i,\calG_i)_{i=2}^{\infty}\x \calG_1))
		\relmiddle|\calG_0}
		\nn
		&\le
		\Ex{
		(1-Q_1)
		\E[Q_2^{-1}\relmiddle|\calG_1]
		\relmiddle|\calG_0}
		\since{assumption of the induction}
		\nn
		&\le
		\Ex{
		Q_1^{-1}\relmiddle|\calG_0
		}-1
		\since{$Q_i^{-1}$ is a supermartingale.}
	\end{align}

	We obtain the lemma from
	\begin{align}
		\Ex{
		\sum_{i=1}^{k}
		\prod_{j=1}^i(1-Q_j)
		\relmiddle| \calG_0
		}
		=
        {N}_{k}((Q_i,\calG_i)_{i=1}^{\infty}\x \calG_0)
		\le
        \overline{N}_{k}((Q_i,\calG_i)_{i=1}^{\infty}\x \calG_0)
		\quad \mbox{a.s.}
	\end{align}
\end{proof}

\begin{proof}[Proof of Theorem \ref{thm_stop}]
	The statement is obvious for the case
	$\tau=T+1$ and we consider the other case in the following.

	Let $\tau_i$ be the time of the
	$i$-th occurrence of $\calA_{t}$.
	More formally, we define $\tau_i$ as the stopping time
	$\tau_1=\tau$ and
	\begin{align}
		\tau_{i+1}=
		\begin{cases}
			\min\brx{t\in[T]: \sum_{t'=1}^T \ind{\calA_{t'}}=i+1}&\sum_{t'=1}^T \ind{\calA_{t'}}\ge i+1,\\
			\tau_{i}+1&\mbox{otherwise.}
		\end{cases}
	\end{align}

	Then $(P'_i)=(P_{\tau_i})$ is a stochastic process
	measurable by
	the filtration induced by $(\calF_{i}')=(\calF_{\tau_i})$.
	By Lemma \ref{lem:stop} we obtain
	\begin{align}
		\Ex{\sum_{t=1}^{T}\ind{\calA_{t}} \relmiddle| \calF_{\tau}}
		&=
	    \Ex{\sum_{n=1}^T \ind{\sum_{t=1}^T \ind{\calA_{t}} \ge n \relmiddle| \calF_{\tau}}} \nn
	    &\le
	    1 + \Ex{\sum_{n=2}^T \ind{\sum_{t=1}^T \ind{\calA_{t}} \ge n \relmiddle| \calF_{\tau}}} \nn
		&\le
		1+
		\Ex{
		\sum_{i=1}^{T}
		\prod_{j=1}^i(1-P'_j)
		\relmiddle|
		\calF_{1}'
		}\nn
		&\le
		1+
		\Ex{
		(P_1')^{-1}|
		\calF_{1}'
		}-1\nn
		&=
		P_{\tau}^{-1}\per
	\end{align}
\end{proof}

%% file: tex_src/proof/theory_regret_upper_bound.tex
In this appendix, we give the proof of Theorem~\ref{thm:regret_upper_bound}.
Note that the cells are defined for the decomposition of $\R^M$, not $\calP_M$.
In other words, the cell $\Ci$ is here defined as
$\Ci = \ev{p \in \R^M : \text{action $i$ is optimal}}$.
For the linear setting, the empirical feedback distribution $\qit$ and $\qin$ are defined as
\begin{align}
    \qit &\coloneqq \frac{1}{N_i(t)} \sum_{s\in[t-1]: i(s)=i} y(s) \com \\
    \qin &\coloneqq \text{the value of $\qit$ after taking action $i$ for $n$ times.}
\end{align}
Recall that  $\phatt = B_t^{-1}b_t$, which is the mode of $\gbart(p)$.

\subsection{Regret Decomposition}
Here, we break the regret into several terms.
For any $i\in[N]$, we define events
\begin{align}
    \Ai &\coloneqq \ev{ \nrm{S_i \phatt - S_i p^* } \leq \frac{\epsilon}{4} }\com  \\
    \Atili &\coloneqq \ev{ \nrm{ S_i \ptilt - S_i p^* } \leq \epsilon } \per
\end{align}
We first decompose the regret as
\begin{align}
\label{eq:reg_decompos1}
    \Regret(T)
    &=
    \sumT \Delta_{i(t)}
    \nn
    &\le
    \sumT
    \prn[\Big]{
        \Delta_{i(t)}
        \ind{\Atilone}
        +
        \max_{j\in[N]}
        \Delta_j
        \ind{\Atilonecomp}
    }
    \nn
    &=
    \sum_{i\neq 1}
    \sumT
    \Delta_i\ind{i(t) = i\x\Atilone}
    +
    \max_{j\in[N]}
    \Delta_j
    \sumT
    \ind{\Atilonecomp}
    \nn
    &\le
    \sum_{i\neq 1}
    \Delta_i\sumT
    \prn[\bigg]{
        \underbrace{\ind{ i(t)=i\x \Atilone\x \Ai}}_{\text{(A)}}
        +
        \underbrace{\ind{i(t)=i\x \Aicomp}}_{\text{(B)}}
    }
    +
    \max_{j\in[N]}
    \Delta_j
    \sumT
    \ind{\Atilonecomp}
    \per
\end{align}

To decompose the last term,
we define the following notation.
We define for any $i \in [N]$
\begin{align}
  P_i(t)
  \coloneqq
  \pr{\ptilt\in \Ci \relmiddle| \calF_t}\per
\end{align}
We also define 
\begin{align}
    \calC_{i,t} \coloneqq \Ci \cap B_{\eppr}(\phatt) \com
\end{align}
where $\eppr$ is defined in~\eqref{def:ep_prime},
and
\begin{align}
    \bar{i}_t
    \coloneqq
    \argmax_{i\in [N]}\pr{\ptilt\in \calC_{i,t} \relmiddle| \calF_t} \per
\end{align}
We define $\pact$ as an arbitrary point in $\calC_{\bar{i}_t,t}$.
Then, we define
\begin{align}
  \Aaci
  \coloneqq
  \ev{\normx{S_i\pact - S_i p^*}\le \frac{\ep}{8}} \per
\end{align}

Using these notations,
the last term in~\eqref{eq:reg_decompos1}
can be decomposed as
\begin{align}
    \ind{\Atilonecomp}
    &\leq
    \sum_{k=1}^N \ind{\pact\in\Ck\x \Atilonecomp}  \nn
    &=
    \sum_{k=1}^N \ind{\pact\in\Ck\x \Aackcomp\x \Atilonecomp}
    +
    \sum_{k=1}^N \ind{\pact\in\Ck\x \Aack\x \Atilonecomp}  \nn
    &\leq
    \underbrace{\sum_{k=1}^N \ind{\pact\in\Ck\x \Aackcomp}}_{\text{(C)}}
    +
    \underbrace{\ind{\pact\in\calC_1\x \bar{\calA}_1(t)\x \Atilonecomp}}_{\text{(D)}}
    +
    \underbrace{\sum_{k=2}^N \ind{\pact\in\Ck\x \Aack}}_{\text{(E)}}  \per
\end{align}
We will bound the expectation of each term in the following and complete the proof of Theorem~\ref{thm:regret_upper_bound} as
\begin{align}
    \Expect{\Regret(T)}  
    &=
    \sum_{i\neq1} \Delta_i
    \left(
    \Order\left( \frac{1}{\ep^2} \log T \right)  
    + 
    \Order\left(\frac{N}{\ep^2} \log T \right)
    \right)  \nn
    &\quad+
    \max_{j\in[N]} \Delta_j
    \left(
    \sum_{k=1}^N \Order\left(\frac{N M}{\ep^2} \log T\right)
    +
    \Order(1)
    +
    \sum_{k=2}^N \Order(1)  
    \right) \nn
    &=
    \Order\left(
    \max\left\{
        \frac{N \sum_{i\in[N]} \Delta_i}{\ep^2},
        \frac{N^2 M \max_{i\in[N]} \Delta_i}{\ep^2}
    \right\}
    \log T \right) \nn
    &=
    \Order\left(
        \frac{A N^2 M \max_{i\in[N]} \Delta_i}{\Lambda^2}
    \log T \right) \com
\end{align}
where the last transformation follows from the definition of $\ep$ in~\eqref{eq:def_epsilon}.

\subsection{Analysis for Case (A)}
\begin{lemma}
    \label{lem:case_A}
    For any $i \neq 1$,
    \begin{align}
        \Expect{\sumT \ind{i(t)=i\x \tilde{\calA}_1(t)\x \calA_i(t)} }
        \le
        \frac{64}{9 \ep^2} \log T
        +
        2^{M/2}  \per
    \end{align}
\end{lemma}

To prove Lemma~\ref{lem:case_A}, we prove the following lemma
using Corollary~\ref{cor:loc_prop_cor1} and Lemma~\ref{lem:Mahalanobis_concentration}.
\begin{lemma}
    \label{lem:caseA_2}
    For any $0 \le \xi < 1/2$,
    \begin{align}
        \pr{ \ptilt \in \calV_i \relmiddle| \calA_i(t)\x N_i(t) > n_i}
        &\leq
        \exp \prn[\Big]{ -\frac{9}{16} \xi \epsilon^2 n_i } (1 - 2\xi)^{-M/2} \com
        \label{eq:lem_caseA_2}
    \end{align}
    where $\calV_i \coloneqq \ev{p \in \Ci: \nrm{S_1 p - S_1 p^*} \leq \epsilon}$.
\end{lemma}
\begin{proof}
    Since $\ptilt \sim \calN(\phatt\x B_t^{-1})$ for $\phatt = B_t^{-1}b_t$,
    the squared Mahalanobis distance $\nrm{B_t^{1/2} (\ptilt - \phatt) }^2$ follows the chi-squared distribution with $M$ degree of freedom.
    Therefore, we have
    \begin{align}
        \label{eq:prob_bound_1}
        \pr{ \ptilt \in \calV_i \relmiddle| \calA_i(t)\x N_i(t) > n_i}
        \leq
        h\prn[\bigg]{ \inf_{p\in\calV_i} \nrm{ B_t^{1/2} (p - \phatt ) }^2} \com
    \end{align}
    where
    $h(a) = \P_{X\sim\chi^2_M}\left\{ X \geq a \right\}$.
    To use Lemma~\ref{lem:Mahalanobis_concentration},
    we check the condition of Lemma~\ref{lem:Mahalanobis_concentration} is indeed satisfied.
    First, it is obvious that the assumptions $N_i(t) \geq n_i$ and $\nrm{S_i \phatt - S_i p^*} < \ep/4$ are satisfied.
    Besides, $p\in\calV_i$ implies $p \in \calT_i = \{p \in \R^M : \nrm{S_i p - S_i p^*} \ge \ep \}$ from Corollary~\ref{cor:loc_prop_cor1}.
    Thus, applying Lemma~\ref{lem:Mahalanobis_concentration} concludes the proof.
\end{proof}

\begin{proof}[Proof of Lemma~\ref{lem:case_A}]
For any $n_i > 0$,
\begin{align}
    &
    \sumT \ind{i(t)=i\x \tilde{\calA}_1(t)\x \calA_i(t)} \nn
    &=
    \sumT \ind{ i(t)=i\x \Atilone\x \calA_i(t)\x N_i(t) \leq n_i}
    +
    \sumT \ind{ i(t)=i\x \Atilone\x \calA_i(t)\x N_i(t) > n_i}  \nn
    &\leq
    n_i + \sumT \ind{ i(t)=i\x \Atilone\x \calA_i(t)\x N_i(t) > n_i} \per
\end{align}
The second term is bounded from above as
\begin{align}
    &
    \Expect{ \sumT \ind{ i(t)=i\x \Atilone\x \calA_i(t)\x N_i(t) > n_i} } \nn
    &=
    \sumT \pr{ i(t)=i\x \Atilone\x \calA_i(t)\x N_i(t) > n_i }  \nn
    &\leq
    \sumT \pr{ i(t)=i\x \Atilone \relmiddle| \calA_i(t)\x N_i(t) > n_i}  \nn
    &=
    \sumT \pr{ i(t)=i\x \Atilone\x \ptilt \in \Ci \relmiddle| \calA_i(t)\x N_i(t) > n_i}
    \Sincem{i(t)=i \text{ implies } \ptilt\in\Ci}  \nn
    &\leq
    \sumT \pr{ \ptilt \in \calV_i \relmiddle| \calA_i(t)\x N_i(t) > n_i} \per
\end{align}
To obtain an upper bound for
$\pr{ \ptilt \in \calV_i \relmiddle| \Ai\x N_i(t) > n_i}$,
we use Lemma~\ref{lem:caseA_2}.
By taking $n_i = \frac{16}{9} \frac{1}{\xi \epsilon^2} \log T$ with $\xi = 1/4$,
we have
\begin{align}
    \Expect{\sumT \ind{i(t)=i\x \tilde{\calA}_1(t)\x \calA_i(t)} }
    &\leq
    n_i + \sumT \pr{ \ptilt \in \calV_i \relmiddle| \calA_i(t)\x N_i(t) > n_i}  \nn
    &\leq
    n_i + \sumT \exp \prn[\Big]{-\frac{9}{16} \xi \epsilon^2 n_i } (1 - 2\xi)^{-M/2}
    \By{Lemma~\ref{lem:caseA_2}} \nn
    &=
    \frac{16}{9} \frac{1}{\xi \epsilon^2} \log T
    +
    (1 - 2\xi)^{-M/2}  \nn
    &=
    \frac{64}{9 \ep^2} \log T
    +
    2^{M/2}  \com
\end{align}
which completes the proof.
\end{proof}

\subsection{Analysis for Case (B)}
\begin{lemma}
    \label{lem:case_B}
    For any $i \neq 1$,
    \begin{align}
        \Expect{\sumT \ind{i(t)=i\x \calA_i^c(t) }}
        \le
        \frac{256N \left(\log T + \frac{A}{2} \log 2 + 1 \right)}{\epsilon^2}
        +
        \frac{16 A^2}{\epsilon^2} 
    \end{align}
\end{lemma}
The regret in this case can intuitively be bounded
because as the round proceeds the event $i(t)=i$ makes $S_i \phatt$ close to $S_i p^*$,
which implies that the expected number of times the event $\calA_i^c(t)$ occurs is not large.

Before going to the analysis of Lemma~\ref{lem:case_B},
we prove useful inequalities between $\nrm{\qit - S_i p^*}$, $\nrm{\qit - S_i \phatt}$, and $\nrm{S_i \phatt - S_i p^*}$.

\begin{lemma}
    \label{lem:useful_1}
    Assume $N_i(t) > 0$.
    Then,
    \begin{align}
        \nrm{\qit - S_i \phatt}^2
        \leq
        \frac{\Zbsli}{N_i(t)} + \nrm{\qit - S_i p^* }^2 \per
    \end{align}
\end{lemma}
\begin{proof}
    Recall that $\phatt$ is the maximizer of $\bar{g}_t(p)$,
    and we have
    \begin{align}
        \phatt
        =
        \argmax_{p\in\R^M} \bar{g}_t(p)
        =
        \argmax_{p\in\R^M} \prod_{i=1}^N \exp\crl[\Big]{ -\frac12 N_i(t) \nrm{ \qit - S_i p }^2 }
        =
        \argmin_{p\in\R^M} \sum_{i=1}^N N_i(t) \nrm{ \qit - S_i p }^2 \per
    \end{align}
    Using this and the definition of $\Zbsli$,
    we have
    \begin{align}
        N_i(t) \nrm{\qit - S_i \phatt}^2
        &\leq
        \sum_{k\in[N]} N_k(t) \nrm{\qkt - S_k \phatt}^2  \nn
        &\leq
        \sum_{k\in[N]} N_k(t) \nrm{\qkt - S_k p^* }^2  \nn
        &\leq
        \Zbsli + N_i(t) \nrm{\qit - S_i p^* }^2 \per
    \end{align}
    Dividing by $N_i(t)$ on the both sides completes the proof.
\end{proof}

\begin{lemma}
    \label{lem:useful_2}
    Assume that $\Aicomp$ and $N_i(t) > 0$ hold.
    Then,
    \begin{align}
        \nrm{\qit - S_i p^*} 
        >
        \frac12 \prn[\bigg]{ \frac{\epsilon}{4} - \sqrt{\frac{\Zbsli}{N_i(t)}} } \per
        \label{eq:useful_2}
    \end{align}
\end{lemma}
\begin{proof}
    By the triangle inequality,
    \begin{align}
        \nrm{\qit - S_i p^*}
        &\geq
        \nrm{S_i \phatt - S_i p^*} - \nrm{\qit - S_i \phatt}  \nn
        &>
        \frac\ep4 - \sqrt{\frac{\Zbsli}{N_i(t)} + \nrm{\qit - S_i p^* }^2}
        \By{$\Aicomp$ and Lemma~\ref{lem:useful_1}}  \nn
        &\geq
        \frac\ep4 - \sqrt{ \frac{\Zbsli}{N_i(t)} } - \nrm{\qit - S_i p^* }
        \By{$\sqrt{x+y} \leq \sqrt{x} + \sqrt{y}$ for $x, y \geq 0$}  \com
    \end{align}
    which is equivalent to~\eqref{eq:useful_2}.
\end{proof}

\begin{proof}[Proof of Lemma~\ref{lem:case_B}]
    We first bound the expectation conditioned on $\Zbsli$,
    and then take the expectation for $\Zbsli$.
    Now,
    \begin{align}
        &
        \Expect{\sumT \ind{i(t)=i\x \calA_i^c(t) } \relmiddle| \Zbsli }  \nn
        &=
        \Expect{\sumT \ind{i(t)=i\x \calA_i^c(t)\x N_i(t) \leq \frac{64\Zbsli}{\ep^2} }  \relmiddle| \Zbsli}  \nn
        &\quad+
        \Expect{\sumT \ind{i(t)=i\x \calA_i^c(t)\x N_i(t) > \frac{64\Zbsli}{\ep^2}}  \relmiddle| \Zbsli} \nn
        &\leq
        \frac{64\Zbsli}{\ep^2}
        +
        \Expect{\sumT \ind{i(t)=i\x \calA_i^c(t)\x N_i(t) > \frac{64\Zbsli}{\ep^2} }  \relmiddle| \Zbsli}
        \Since{$i(t) = i$ for all $t\in[T]$} \per
    \end{align}
    The first term becomes ${256N \left(\log T + \frac{A}{2} \log 2 + 1 \right)}/{\ep^2}$ by taking expectation over $\Zbsli$ using Lemma~\ref{lem:expectation_Z}.
    Then, we bound the second term.
    From Lemma~\ref{lem:useful_2},
    $\calA_i^c(t)$ 
    and 
    $N_i(t) > \frac{64\Zbsli}{\ep^2}$ 
    imply
    $\nrm{\qit - S_i p^*} > \epsilon/16$.
    Therefore,
    \begin{align}
        &
        \Expect{\sumT \ind{i(t)=i\x \calA_i^c(t)\x N_i(t) > \frac{64\Zbsli}{\ep^2} }  \relmiddle| \Zbsli}  \nn
        &\leq
        \Expect{ \sumT \ind{i(t)=i\x \nrm{\qit - S_i p^* } > \frac{\epsilon}{16} } } \nn
        &\leq
        \Expect{ \sumT \ind{i(t)=i\x \bigcup_{y\in[A]} | (\qit)_y - (S_i)_y p^*| > \frac{\epsilon}{16\sqrt{A}}}  }  \nn
        &\leq
        \Expect{ \sum_{y=1}^A \sumT \ind{i(t)=i\x | (\qit)_y - (S_i)_y p^*| > \frac{\epsilon}{16\sqrt{A}}} }  \nn
        &\leq
        \Expect{ \sum_{y=1}^A \sum_{n_i=1}^T \sumT \ind{i(t)=i\x  N_i(t)=n_i\x
        \abs{ (\qit)_y - (S_i)_y p^* } > \frac{\epsilon}{16\sqrt{A}} }  }  \nn
        &=
        \Expect{ \sum_{y=1}^A \sum_{n_i=1}^T \ind{ \bigcup_{t=1}^T \ev{ i(t)=i\x  N_i(t)=n_i\x
        \abs{ (\qit)_y - (S_i)_y p^* } > \frac{\epsilon}{16\sqrt{A}} } } }  \nn
        &\quad
        \Since{The event $\ev{i(t)=i\x N_i(t)=n_i}$ occurs at most once for fixed $n_i$.}  \nn
        &\leq
        \sum_{y=1}^A \sum_{n_i=1}^T \pr{ \abs{ (\qini)_y - (S_i)_y p^* } > \frac{\ep}{4\sqrt{A}} }  \nn
        &\leq
        \sum_{y=1}^A \sum_{n_i=1}^T 2 \exp \prn[\bigg]{-2n_i\left(\frac{\ep}{4\sqrt{A}}\right)^2 }
        \By{Hoeffding's ineq.}  \nn
        &\leq
            2A \sum_{n_i=1}^\infty\exp \prn[\Big]{-\frac{n_i \ep^2}{8A}} \nn
        &=
        2A \frac{1}{\exp \left(\frac{\ep^2}{8A} \right) - 1}  \nn
        &\leq
        2A \frac{1}{ \frac{\ep^2}{8A}  }
        \Bym{\e^x \geq 1 + x}  \nn
        &=
        \frac{16 A^2}{\epsilon^2} \per
    \end{align}
    By summing up the above argument, the proof is completed.
\end{proof}

\subsection{Analysis for Case (C)}
Before going to the analysis of cases (C), (D), and (E),
we recall some notations.
Recall that
\begin{align}
  P_i(t)=
  \pr{\ptilt \in \Ci \relmiddle| \calF_t} \com
\end{align}
$\calC_{i,t}=\Ci \cap B_{\eppr}(\phatt)$,
$\bar{i}_t= \argmax_{i\in [N]}\pr{\ptilt\in \calC_{i,t}|\calF_t}$,
and $\pact$ is an arbitrary point in $\calC_{\bar{i}_t,t}$.
Also recall that 
\begin{align}
  \Aaci
  =
  \ev{\normx{S_i \pact - S_i p^*} \le \frac{\ep}{8}} \per
\end{align}

\begin{lemma}
    \label{lem:case_C}
    For any $i \in [N]$,  
    \begin{align}
        \Expect{\sumT \ind{\pact\in\Ci\x \Aacicomp}} 
        \le
        \frac{N}{p_0} 
        \left(
            \frac{2^5 M \log T}{\ep^2} 
            + 
            \e^{\lambda \nrm{p^*}^2 / 2} \pax{\frac{1}{\lambda T} + \frac{L}{M\lambda}}^{M/2} \frac{1}{1 - \e^{-\ep^2/2^5}}
        \right) \per
    \end{align}
\end{lemma}

Before proving the above lemma, 
we give two lemmas.
\begin{lemma}
\label{lem:ptilt_lower_bound}
    \begin{align}
      \pr{\ptilt\in \calC_{\bar{i}_t} \relmiddle| \calF_t }
      \ge
      \pr{\ptilt\in \calC_{\bar{i}_t,t} \relmiddle| \calF_t }
      \ge
      p_0/N \com
    \end{align}
    where $p_0 \coloneqq 1-h((\lambda \ep^\prime)^2)$.
\end{lemma}
\begin{proof}
    First, we prove
    \begin{align}
        \pr{\ptilt\in \bigcup_{i\in[N]} \calC_{i,t} \relmiddle| \calF_t}
        \ge 
        1-h((\lambda \ep^\prime)^2)
        \per
    \end{align}
    This follows from
    \begin{align}
        \pr{\ptilt \not\in \bigcup_{i\in[N]} \calC_{i,t} \relmiddle| \calF_t}
        &=
        \pr{\ptilt\in B_{\eppr}(\phatt) \relmiddle| \calF_t} \nn
        &\le
        h\left(
            \inf_{p \in \ev{p : \nrm{p - \phatt} > \ep^\prime}}
            \nrm{B_t^{1/2} (p - \phatt)}^2
        \right) \nn
        &\le
        h\prn[\Big]{
            \lambda \nrm{p - \phatt}^2
        } \nn
        &\le
        h((\lambda \ep^\prime)^2)
        \per
    \end{align}
    
    Using the definition of $\bar{i}_t$ completes the proof.
\end{proof}

\begin{lemma}
\label{lem:Abarcomp2Acomp}
    For any $i \in [N]$,
    the event $\Aacicomp$ implies $\nrm{S_i \phatt - S_i p^*} \ge \ep/16$.
\end{lemma}
\begin{proof}
    Using the triangle inequality, we have
    \begin{align}
        \nrm{S_i \phatt - S_i p^*}
        &\ge
        \nrm{S_i \pact - S_i p^*}
        -
        \nrm{S_i \pact - S_i \phatt} \nn
        &\ge
        \ep/8 - \nrm{S_i} \nrm{\pact - \phatt}  \nn
        &\ge
        \ep/8 - \nrm{S_i} \frac{\ep}{16 \max_i \nrm{S_i}} \nn
        &\ge 
        \ep/8 - \ep/16
        =
        \ep/16 \per
    \end{align}
\end{proof}

\begin{proof}[Proof of Lemma~\ref{lem:case_C}]

For any $n_0$, which is specified later, we have
\begin{align}
    &
    \Expect{\sumT \ind{\pact\in\Ci\x \Aacicomp}}  \nn
    &=
    \Expect{\sumT \ind{\pact\in\Ci\x \Aacicomp\x N_i(t) < n_0}} 
    + 
    \Expect{\sumT \ind{\pact\in\Ci\x \Aacicomp\x N_i(t) \ge n_0}}  
\end{align}

The first term can be bounded by 
$(p_0/N)^{-1} \cdot n_0$ 
from Lemma~\ref{lem:ptilt_lower_bound}.
The rigorous proof can be obtained by the almost same argument
as the following analysis of the second term using Theorem~\ref{thm_stop}.

Then, we will bound the second term.
Specifically, we will prove that 
for $n_0=\frac{M\log T}{(\ep/16)^2}$,
\begin{align}
\Expect{\sum_{t=1}^T \ind{\pact\in\Ci\x \Aacicomp\x N_i(t)\ge n_0}}
=
\Order(1)\per
\end{align}

First we have
\begin{align}
    \lefteqn{
    \Expect{\sum_{t=1}^T \ind{\pact\in\Ci\x \Aacicomp\x N_i(t)\ge n_0}}
    }\nn
    &\le
    \sum_{m=n_0}^{\infty}
    \Expect{\sum_{t=1}^T \ind{
    \pact\in\Ci\x \Aacicomp\x
    N_i(t)=m
    }}\per
\end{align}

Let
\begin{align}
    \tau = 
    \min\ev{t: \pact\in\Ci\x \Aacicomp\x N_i(t)=m} \wedge (T+1)
  \end{align}
  be the first time such that $\pact\in\Ci\x \Aacicomp$ and $N_i(t)=m$ occur.
  Letting
  $\calA_t \coloneqq \ev{\pact\in\Ci\x \Aacicomp\x N_i(t)=m}$,
  $\calB_t \coloneqq \ev{ i(t)=i }$ and $P_t \coloneqq p_0/N$ in Theorem~\ref{thm_stop}, 
  we have
  \begin{align}
    \Expect{\sum_{t=1}^T \ind{
    \pact\in\Ci\x \Aacicomp\x N_i(t)=m}
    }
    \le
    \frac{N}{p_0}
    \pr{\tau \le T}\per\label{aic1}
  \end{align}
  Here $\tau \le T$ implies that
  \begin{align}
    \Vert \hat{p}_{\tau}-p^* \Vert_{B_{\tau}}
    &=
    (\hat{p}_{\tau} - p^*)^\top
    \pax{\lambda I+\sum_{j\in[N]}N_j(\tau)S_j^\top S_j}
    (\hat{p}_{\tau} - p^*)
    \nn
    &\ge
    m
    (\hat{p}_{\tau} - p^*)^\top
    \pax{S_i^\top S_i}
    (\hat{p}_{\tau} - p^*)
    \nn
    &=
    m
    \normx{S_i(\hat{p}_{\tau}-p^*)}^2
    \ge m (\ep/16)^2\com
  \end{align}
  where the last inequality follows from Lemma~\ref{lem:Abarcomp2Acomp}.
  Therefore we have
  \begin{align}
    \Expect{\exp(\Vert \hat{p}_{\tau}-p^* \Vert_{B_{\tau}}^2/2)}
    &\ge
    \Expect{\ind{\tau \le T}\exp(\Vert \hat{p}_{\tau}-p^* \Vert_{B_{\tau}}^2/2)}
    \nn
    &
    \ge
    \exp(m(\ep/16)^2/2)
    \pr{\tau \le T}\per\label{aic2}
  \end{align}
  Note that $|B_{\tau}|\le |B_T|\le (1+TL/M)^M$
  for $L=\max_i \sqrt{\mathrm{trace}(S_i^\top S_i)}=\max_i \normx{S_i}_{\mathrm{F}}$
  by Lemma 10 of~\citet{Abbasi11improved},
  where $\normx{\cdot}_{\mathrm{F}}$ is the Frobenius norm.
  Therefore we have
  \begin{align}
    \Expect{
    \exp(\Vert \hat{p}_{\tau}-p^* \Vert_{B_{\tau}}/2)
    }
    &\le
    \Expect{
    \sqrt{|B_\tau|}\cdot \frac{\exp(\Vert \hat{p}_{\tau}-p^* \Vert_{B_{\tau}}^2/2)}{\sqrt{|B_\tau|}}
    }\nn
    &\le
    (1+TL/M)^{M/2}
    \Expect{
    \frac{\exp(\Vert \hat{p}_{\tau}-p^* \Vert_{B_{\tau}}^2/2)}{\sqrt{|B_\tau|}}
    }
    \nn
    &\le
    (1+TL/M)^{M/2}
    \Expect{
    \frac{\exp(\Vert \hat{p}_{0}-p^* \Vert_{B_{0}}^2/2)}{\sqrt{|B_0|}}
    }
    \label{superm}
    \\
    &=
    \pax{1+\frac{TL}{M\lambda}}^{M/2}
    \e^{\lambda \Vert p^* \Vert^2/2}\com
    \label{aic3}
  \end{align}
  where~\eqref{superm} holds since
  $\frac{\exp(\Vert \hat{p}_{t}-p^* \Vert_{B_{t}}^2/2)}{\sqrt{|B_t|}}$ is
  a supermartingale from Lemma~\ref{lem:martingale}.
  Combining~\eqref{aic1},~\eqref{aic2}, and~\eqref{aic3}, we obtain
  \begin{align}
    \sum_{m=n_0}^{\infty}
    \Expect{\sum_{t=1}^T \ind{
    \pact\in\Ci\x \Aacicomp\x N_i(t)=m
    }}
    &\le
    \frac{N}{p_0}
    \pax{\frac1\lambda + \frac{TL}{M\lambda}}^{M/2}
    \e^{\lambda \Vert p^* \Vert^2/2}
    \sum_{m=n_0}^{\infty}
    \e^{-m(\ep/16)^2/2}
    \nn
    &\le
    \frac{N}{p_0}
    \pax{\frac1\lambda+\frac{TL}{M\lambda}}^{M/2}
    \e^{\lambda \Vert p^* \Vert^2/2}
    \frac{\e^{-n_0\ep^2/2}}{1-\e^{-(\ep/16)^2/2}}
    \per
  \end{align}
  By choosing $n_0 = \frac{M \log T}{(\ep/16)^2}$ we obtain the lemma.
\end{proof}

\subsection{Analysis for Case (D)}
\begin{lemma}
    \label{lem:case_D}
    For any $i \in [N]$,
    \begin{align}
        \Expect{ \sumT \ind{\pact\in\Ci\x  \Aaci\x  \Atilicomp} }
        \le
        \frac{48}{9} \frac{M+2}{\ep^2} \frac{N}{p_0} \per
    \end{align}
\end{lemma}
\begin{remark}
    To prove the regret upper bound, it is enough to prove Lemma~\ref{lem:case_D} only for $i = 1$.
    However, for the sake of generality, we prove the lemma for any $i \in [N]$. 
\end{remark}
Before proving Lemma~\ref{lem:case_D},
we give two following lemmas.
\begin{lemma}
\label{lem:acute2vanilla}
    For any $i \in [N]$, the event $\Aaci$ implies $\Ai$.
\end{lemma}
\begin{proof}
    Using the triangle inequality, we have
    \begin{align}
        \nrm{S_i p^* - S_i \phatt}
        &\le
        \nrm{S_i p^* - S_i \pact}
        +
        \nrm{S_i \pact - S_i \phatt}  \nn
        &\le
        \ep/8 + \nrm{S_i} \cdot \frac{\ep}{16 \max_i \nrm{S_i}}
        <
        \ep/4 \com
    \end{align}
    which completes the proof.
\end{proof}

Now, Lemma~\ref{lem:case_D} can be intuitively proven
because from Lemma~\ref{lem:acute2vanilla}, $\Aaci$ implies $\Ai$,
and the events $\Ai$ and $\Atilicomp$ does not simultaneously occur many times.

Let $t = \sigma_1, \dots, \sigma_m$ be the time of the first $m$ times that the event 
$\{\pact\in\Ci\x \Ai\x N_i(t)=n_i\}$ occurred
(not $\{\pact\in\Ci\x \Aaci\x N_i(t)=n_i\}$).
In other words, we define
\begin{itemize}
    \item
    $\sigma_1$ : the first time that $\pact\in\Ci\x \Ai$ and $N_i(t)=n_i$ occurred
    \item
    $\sigma_2$ : the second time that $\pact\in\Ci\x \Ai$ and $N_i(t)=n_i$ occurred
    \item
    \ldots \per
\end{itemize}
Now we prove the following lemma using Lemma~\ref{lem:Mahalanobis_concentration}.
\begin{lemma}
\label{lem:caseD_empin-sampleout}
    For any $0 \le \xi < 1/2$,
    \begin{align}
        \pr{\Atilicomp \relmiddle| \Ai\x \sigma_k = t}
        \leq
        \exp \left(-\frac{9}{16} \xi \epsilon^2 n_i \right) (1 - 2\xi)^{-M/2} \per
        \label{eq:caseD_lem_empin_sampleout}
    \end{align}
\end{lemma}
\begin{proof}
    Recall that $\calT_i = \ev{p \in \R^M : \nrm{S_i p - S_i p^* } > \epsilon}$.
    We follow a similar argument as the analysis for Lemma~\ref{lem:caseA_2}.
    Since $\ptilt \sim \calN(B_t^{-1}b_t, B_t^{-1})$,
    the squared Mahalanobis distance $\nrm{ B_t^{1/2} (p - \phatt )}^2$
    follows the chi-squared distribution with $M$ degree of freedom.
    Hence,
    for $h(a) = \P_{X\sim\chi^2_M}\left\{ X \geq a \right\}$,
    we have
    \begin{align}
        \pr{\Atilicomp \relmiddle| \Aini\x \sigma_k = t}
        \leq
        h\left(\inf_{p\in\calT_i} \nrm{ B_t^{1/2} (p - \phatt )}^2\right) \per
    \end{align}
    Then, Eq.~\eqref{eq:caseD_lem_empin_sampleout} directly follows from Lemma~\ref{lem:Mahalanobis_concentration}.
\end{proof}

\begin{proof}[Proof of Lemma~\ref{lem:case_D}]
From Lemma~\ref{lem:acute2vanilla}, the event $\Aaci$ implies $\Ai$.
Hence, it is enough to derive the upper bound for 
\begin{align}
    \Expect{ \sumT \ind{\pact\in\Ci\x  \Ai\x  \Atilicomp} }
\end{align}
instead of the bound for 
\begin{align}
    \Expect{ \sumT \ind{\pact\in\Ci\x  \Aaci\x  \Atilicomp} } \per
\end{align}

Using Lemma~\ref{lem:caseD_empin-sampleout},
we can bound the term for case (D) from above as
\begin{align}
    &
    \Expect{ \sumT \ind{\pact\in\Ci\x \Ai\x \Atilicomp} }  \nn
    &=
    \Expect{ \sum_{n_i=1}^T \sumT \ind{\Ai\x \Atilicomp\x N_i(t)=n_i} }  \nn
    &=
    \sum_{n_i=1}^T \sumT \pr{\Ai\x \Atilicomp\x N_i(t)=n_i}  \nn
    &=
    \sum_{n_i=1}^T \sumT \sum_{k=1}^T \pr{\Ai\x \Atilicomp\x \sigma_k = t }
    \Since{the event $\ev{\sigma_k = t}$ is exclusive for fixed $n_i$} \nn
    &=
    \sum_{n_i=1}^T
    \sumT \sum_{k=1}^T
    \pr{\Ai\x \sigma_k = t}
    \pr{\Atilicomp \relmiddle| \Ai\x \sigma_k = t}  \nn
    &\leq
    \sum_{n_i=1}^T
    \sumT \sum_{k=1}^T
    \pr{\Ai\x \sigma_k = t}
    C \e^{-n_i \iota}
    \By{Lemma~\ref{lem:caseD_empin-sampleout}} \nn
    &=
    \sum_{n_i=1}^T
    C \e^{-n_i \iota}
    \sumT \sum_{k=1}^T
    \pr{\Ai\x \sigma_k = t}  \nn
    &\leq
    \sum_{n_i=1}^T
    C \e^{-n_i \iota}
    \sumT \sum_{k=1}^T
    \pr{\sigma_k = t}  \nn
    &\leq
    \sum_{n_i=1}^T
    C \e^{-n_i \iota}
    \sum_{k=1}^T
    \pr{\text{$\sigma_k$ exists}}  \nn
    &\leq
    \sum_{n_i=1}^T
    C \e^{-n_i \iota}
    \sum_{k=1}^T
    \left(1 - \frac{p_0}{N}\right)^{k-1}
    \By{$\tilde{p}_{\sigma_s} \not\in \Ci$ for $s = 1, \dots, k-1$}  \nn
    &\leq
    3C \frac{1}{\e^\iota - 1} \frac{N}{p_0} \nn
    &\leq
    \frac{48}{9} \frac{M+2}{\ep^2} \frac{N}{p_0} \com
\end{align}
where $\iota = \frac{9 \xi \epsilon^2}{16}, C = (1-2\xi)^{-\frac{M}{2}}$,
and in the last inequality we select the optimal $\xi$ and use $1 + x \le \e^x$.
\end{proof}

\subsection{Analysis for Case (E)}
\begin{lemma}
    \label{lem:case_E}
    For any $i \neq 1$,
    \begin{align}
        \Expect{\sumT \ind{\pact\in\Ci\x \Aaci} } 
        \le
        \frac{2^{5M/2+7} \Gamma(M/2+1) \e^{\la^2\normx{p^*}^2/2}}{\de^2\ep^{M+2}\la^{M/2+1}} \com
        \label{eq:case_E}
    \end{align}
    where
    $\ep$ is defined in~\eqref{eq:def_epsilon} and satisfies $\ep\le\min_{p\in\calC_1^c}\normx{p-p^*}/3$,
    and
    \begin{align}
        \delta
        \coloneqq
        \min_{\phat: (L_1-L_i)^\top\phat\ge 0\x \normx{S_{i}(\phat-p^*)}\le \ep/8} 
        \normx{S_1(\phat-p^*)}  \per
    \end{align}
\end{lemma}
We prove Lemma~\ref{lem:case_E} using Lemma~\ref{lem:martingale} and Theorem~\ref{thm_stop}.

\begin{remark}
    The upper bound in~\eqref{eq:case_E} goes to infinite
	when we set $\la=0$, that is, a flat prior is used.
	However, this is not the essential effect of the prior but
	just comes from the minimum eigenvalue of $B_1$.
	In fact, we can see from the proof that a similar bound can be obtained for $\lambda=0$
	if we run some deterministic initialization until $B_t$ becomes positive definite.
\end{remark}
	\begin{proof}
		We evaluate each term in the summation using Theorem~\ref{thm_stop} with
		\begin{align}
			\calA_t
			&=
			\{\bar{p}_t\in \calC_i\x  \normx{S_i(\bar{p}_t - p^*)}\le \ep/8\x N_1(t)=n\}\com
			\nn
			\calB_t&=
			\{\tilde{p}_t\in \calC_1\}\per
		\end{align}
		for $n\in[T]$. Recall that
		\begin{align}
			\bar{g}_t(p)
			=
			\frac{1}{\sqrt{(2\pi)^M|B_{t}^{-1}|}}
			\exp\pax{
			-\frac12\normx{p-\hat{p}_t}_{B_{t}}^2
			}
		\end{align}
		is the probability density function
		of $\hat{p}_t$ given $\calF_t=\{B_t,b_t\}$.
		Using $\tau$ defined in~\eqref{def_tau_stop}, it holds for any $\tau \in [T]$
		that
		\begin{align}
			\pr{\calB_{\tau}|\calF_{\tau}}
            &=
			\pr{\tilde{p}_{\tau}\in \calC_1 \relmiddle| \calF_{\tau}}
			\nn
			&=
			\int_{p\in \calC_1}
			\bar{g}_{\tau}(p)
			\d p\nn
			&\ge
			\int_{p: \normx{p-p^*}\le 3\ep}
			\bar{g}_{\tau}(p)
			\d p\nn
			&\ge
			\sup_{p: \normx{p-p^*}\le 2\ep}
			\int_{p': \normx{p'-p}\le \ep}
			\bar{g}_{\tau}(p')
			\d p'
			\label{p_subset}
			\\
			&\ge
			\sup_{p: \normx{p-p^*}\le 2\ep}
			\inf_{p': \normx{p'-p}\le \ep}
			\bar{g}_{\tau}(p')
			\mathrm{Vol}(\{p'':\normx{p''-p}\le \ep\})
			\nn
			&=
			\frac{(\sqrt{\pi}\ep)^{M}}{\Gamma(M/2+1)}
			\sup_{p: \normx{p-p^*}\le 2\ep}
			\inf_{p': \normx{p'-p}\le \ep}
			\bar{g}_{\tau}(p')
			\nn
			&=
			\frac{(\ep/\sqrt{2})^M\sqrt{|B_{\tau}|}}{\Gamma(M/2+1)}
			\exp\left\{
			-\frac12
			\pax{
			\inf_{p: \normx{p-p^*}\le 2\ep}
			\sup_{p': \normx{p'-p}\le \ep}
			\normx{p'-\phat_{\tau}}_{B_{\tau}}^2
			}
			\right\}
			\nn
			&\ge
			\frac{(\ep/\sqrt{2})^M\sqrt{|B_{\tau}|}}{\Gamma(M/2+1)}
			\exp\left\{
			-\frac{\normx{\phat_{\tau}-p^*}_{B_{\tau}}^2-\ep\de\sqrt{\la n}}{2}
			\right\}\com
			\label{prob_lower}
		\end{align}
		where~\eqref{p_subset} follows since
		$\{p: \normx{p-p^*}\le 3\ep\}\supset \{p': \normx{p'-p_0}\le \ep\}$
		for any $p_0$ such that $\normx{p_0-p^*}\le 2\ep$,
		and the last inequality follows from Theorem~\ref{thm_delta}.
		To apply Theorem~\ref{thm_delta},
		we used Lemma~\ref{lem:acute2vanilla}.
		
		Now we define a stochastic process corresponds to~\eqref{prob_lower}
		as
		\begin{align}
			P_t
			&=
			\frac{(\ep/\sqrt{2})^M\sqrt{|B_{t}|}}{\Gamma(M/2+1)}
			\exp\left\{
			-\frac{\normx{\phat_{t}-p^*}_{B_{t}}^2-\ep\de\sqrt{\la n}}{2}
			\right\}\per
		\end{align}
		Then, by Lemma~\ref{lem:martingale},
		\begin{align}
			\E[P_{t+1}^{-1}|\calF_t]
			&\le
			\frac{\Gamma(M/2+1)}{(\ep/\sqrt{2})^M}
			\e^{-\ep\de\sqrt{\la n}/2}
			\Ex{
			\frac{1}{\sqrt{|B_{t+1}|}}
			\Ex{
			\exp\left(
			\frac{\normx{\phat_{t}-p^*}_{B_{t+1}}^2}{2}
			\right)
			\relmiddle|\calF_t\x S_{i(t)}
			}\relmiddle|\calF_t}\nn
			&\le
			\frac{\Gamma(M/2+1)}{(\ep/\sqrt{2})^M}
			\e^{-\ep\de\sqrt{\la n}/2}
			\Ex{
			\frac{1}{\sqrt{|B_{t}|}}
			\exp\left(
			\frac{\normx{\phat_{t}-p^*}_{B_{t}}^2}{2}
			\right)
			\relmiddle| \calF_t
			}
			\nn
			&=
			P_t^{-1}\com
		\end{align}
		which means that $P_t^{-1}$ is a supermartingale.
		Therefore we can apply Theorem \ref{thm_stop} and obtain
		\begin{align}
			\Ex{
			\sum_{t=1}^T
			\ind{\phat_t\in \calC_i\x  \normx{S_i(\pbar_t-p^*)}\le \ep/8\x N_1(t)=n}
			}
			&\le
			\Ex{
			\ind{\tau\le T}
			P_{\tau}^{-1}}\nn
			&\le
			\Ex{
			P_{\tau}^{-1}}\nn
			&\le
			\Ex{
			P_{1}^{-1}}\nn
			&=
			\frac{\Gamma(M/2+1)\e^{\la^2\normx{p^*}^2/2}}{(\ep\sqrt{\la/2})^M}
			\e^{-\ep\de\sqrt{\la n}/2}
			\per
		\end{align}

		Finally we have
		\begin{align}
			\lefteqn{
			\Ex{\sum_{t=1}^T
			\ind{\bar{p}_t\in \calC_i\x  \normx{S_i(\bar{p}_t-p^*)}\le \ep/8}
			}}\nn
            &=
            \sum_{n=1}^T
            \Ex{
            \sum_{t=1}^T
            \ind{\bar{p}_t\in \calC_i\x  \normx{S_i(\bar{p}_t-p^*)}\le \ep/8\x N_1(t)=n}
            } \nn
			&\le
			\frac{\Gamma(M/2+1)\e^{\la^2\normx{p^*}^2/2}}{(\ep\sqrt{\la/2})^M}
			\sum_{n=1}^{\infty}
			\e^{-\ep\de\sqrt{\la n}/2}
			\nn
			&\le
			\frac{\Gamma(M/2+1)\e^{\la^2\normx{p^*}^2/2}}{(\ep\sqrt{\la/2})^M}
			\int_0^{\infty}
			\e^{-\ep\de\sqrt{\la x}/2}\d x
			\nn
			&=
			\frac{\Gamma(M/2+1)\e^{\la^2\normx{p^*}^2/2}}{(\ep\sqrt{\la/2})^M}
			\frac{2}{(\ep\de\sqrt{\la}/2)^2}\Gamma(2)
			\nn
			&=
			\frac{2^{M/2+3}\Gamma(M/2+1)\e^{\la^2\normx{p^*}^2/2}}{\de^2 \ep^{M+2} \la^{M/2+1}}
			\com
		\end{align}
		which completes the proof.
	\end{proof}

%% file: tex_src/proof/dp_easy.tex
In this appendix, we will see a property of dp-easy games.

\begin{proposition} \label{prop:dp_easy}
Consider any dp-easy games with $c > -1$.
Then, any two actions in the game are neighbors.
\end{proposition}

\begin{remark}
In section~\ref{sec:experiments}, 
we considered dp-easy games with $c > 0$, 
but this can be relaxed to $c > -1$ to prove Proposition~\ref{prop:dp_easy}.
\end{remark}

\begin{proof}
Take any two different actions $j, k \in [N]$ such that $j < k$.
From the definition of the loss matrix in dp-easy games,
we have $e_j \in \Cj$ and $e_k \in \Ck$.

First, we will find $\alpha \in [0,1]$ such that
\begin{align}\label{eq:alpha_cond}
    \alpha e_j + (1 - \alpha) e_k \in \Cj \cap \Ck \per
\end{align}

From the definition of the loss matrix, 
the $i$-th element of $L(\alpha e_j + (1-\alpha) e_k) \in \calP_M$ is
\begin{align}\label{eq:lin_comb_loss}
    \begin{cases}
        - i & (1 \le i \le j)  \\ 
        \alpha c + (1-\alpha) \cdot (-i) & (j+1 \le i \le k)  \\
        c & (k < i \le N) \\
    \end{cases}
    \per
\end{align}
It is easy to see that the indices
which give the minimum value in~\eqref{eq:lin_comb_loss} is $j$ or $k$.
Thus, to achieve the condition~\eqref{eq:alpha_cond}, the following should be satisfied,
\begin{align}
    -j = \alpha c + (1 - \alpha) \cdot (-k) \com
\end{align}
which is equivalent to 
\begin{align}
    \alpha = \frac{k-j}{c+k} (\eqqcolon \alpha^*) \per
\end{align}
Note that we have $0 \le \alpha \le 1$ for any $c > -1$.

Next, 
we introduce the following definitions.
\begin{align}
    \pjk 
    &\coloneqq
    \alpha^* e_j + (1-\alpha^*) e_k \in \Cj \cap \Ck  
    \com \\
    \mathrm{Ball}_\ep^{(j,k)} 
    &\coloneqq 
    \ev{ p \in \calP_M : \nrm{p - \pjk} \le \ep} 
    \com  \\
    \Lx
    &\coloneqq
    L (\pjk + x) \in \R^N
    \per
\end{align}

To prove the proposition,
it is enough to prove the following:
there exists $\ep > 0$, $\mathrm{Ball}_\ep^{(j,k)} \subset \Cj \cup \Ck$.

To prove this,
it is enough to prove that,
there exists $\ep > 0$,
\begin{align}\label{eq:loss_diff_delta}
    \min_{x\in\R^M : \nrm{x} \le \ep}
    \min_{i \in [N]\backslash\{j,k\}}
        \left( (\Lx)_i - (\Lx)_j \right) 
        \vee
        \left( (\Lx)_i - (\Lx)_k \right)
    > 0
    \per
\end{align}

We will prove~\eqref{eq:loss_diff_delta} in the following.
Take any $i \in [N]\backslash\ev{j,k}$ and 
\begin{align}
    \ep
    \coloneqq
    \min_{i: 1\le i < j} \frac12 \frac{j-i}{\nrm{L_j-L_i}}
    \wedge
    \min_{i: j < i < k} \frac12 \frac{(1-\alpha^*) (k-i)}{\nrm{L_i - L_k}}
    \wedge
    \min_{i: k < i \le N} \frac12 \frac{c + j}{\nrm{L_j - L_i}}
    \per
\end{align}
Note that the $\ep$ used here is different from the one used in the proof of the regret upper bounds.

\noindent \textbf{Case (A):} When $1 \le i < j$,
using Cauchy–Schwarz inequality, we have
\begin{align}
    \left( (\Lx)_i - (\Lx)_j \right) 
    \vee
    \left( (\Lx)_i - (\Lx)_k \right)
    &\ge
    (\Lx)_i - (\Lx)_j   \nn
    &=
    (-i + L_i^\top x)
    -
    (-j + L_j^\top x)  \nn
    &=
    (j-i)
    -
    (L_j - L_i)^\top x \nn
    &\ge
    (j-i) 
    - 
    \nrm{L_j - L_k} \nrm{x} \nn
    &\ge
    (j-i)
    -
    \ep \nrm{L_j - L_i}   \nn
    &\ge
    \frac12 (j-i)  \nn
    &>
    0 
    \per
\end{align}

The arguments for cases (B) and (C) follow in the similar manner as case (A).

\noindent \textbf{Case (B):} When $j < i < k$,
we have
\begin{align}
    \left( (\Lx)_i - (\Lx)_j \right) 
    \vee
    \left( (\Lx)_i - (\Lx)_k \right)
    &\ge
    (\Lx)_i - (\Lx)_k   \nn
    &=
    \ev{
        \alpha^* c + (1-\alpha^*) \cdot (-i) + L_i^\top x
    }
    -
    \ev{
        \alpha^* c + (1-\alpha^*) \cdot (-k) + L_k^\top x
    } \nn
    &=
    (1-\alpha^*) (k-i)
    - 
    (L_i - L_k)^\top  x \nn
    &\ge
    (1-\alpha^*) (k-i)
    - 
    \ep (L_i - L_k)^\top \nn    
    &\ge
    \frac12 (1-\alpha^*) (k-i)  \nn
    &>
    0 
    \per
\end{align}

\noindent \textbf{Case (C)}: When $k < i \le N$,
we have
\begin{align}
    \left( (\Lx)_i - (\Lx)_j \right) 
    \vee
    \left( (\Lx)_i - (\Lx)_k \right)
    &\ge
    (\Lx)_i - (\Lx)_j   \nn
    &=
    (c + L_i^\top x)
    -
    (-j + L_j^\top x)  \nn
    &\ge
    c + j 
    - 
    \nrm{L_j - L_i} \nrm{x}  \nn
    &\ge
    c + j 
    - 
    \ep \nrm{L_j - L_i}  \nn
    &\ge
    \frac12 (c+j)  \nn
    &>
    0 
    \per
\end{align}
Summing up the argument for cases (A) to (C),
the proof is completed.
\end{proof}

%% file: tex_src/proof/all_experiments.tex
Here we give the specific values of the opponent's strategy
used in Section~\ref{sec:experiments}
and show the extended experimental results for performance comparison.
Table~\ref{tb:opponent_strategy_val} summarizes the values of opponent's strategy 
used in this appendix and Section~\ref{sec:experiments}.
Figure~\ref{fig:result_regret_all} shows the empirical comparison of the proposed algorithms against the benchmark methods,
and Figure~\ref{fig:result_n_rejected_all} shows the number of the rejected times.
We can see the same tendency as Section~\ref{sec:experiments}, that is, TSPM performs the best and the number of rejections does not increase with the time step $t$.

\begin{table}[h]
\centering
\small
\begin{center}
\captionof{table}{\label{tb:opponent_strategy_val} The values of the opponent's strategy.}
\begin{tabular}{cll}
    \toprule
    \# of outcomes $M$ & opponent's strategy $p^*$ \\
    \hline
    $2$ & $[0.7, 0.3]$  \\
    $3$ & $[0.5, 0.3, 0.2]$  \\
    $4$ & $[0.3, 0.3, 0.3, 0.1]$  \\
    $5$ & $[0.2, 0.3, 0.3, 0.1, 0.1]$  \\
    $6$ & $[0.2, 0.2, 0.3, 0.1, 0.1, 0.1]$  \\
    $7$ & $[0.2, 0.2, 0.3, 0.1, 0.1, 0.05, 0.05]$  \\
    \bottomrule
\end{tabular}
\end{center}
\end{table}

\begin{figure}[htb]
    \begin{center}
    \setlength{\subfigwidth}{.32\linewidth}
    \addtolength{\subfigwidth}{-.32\subfigcolsep}
    \begin{minipage}[t]{\subfigwidth}
        \centering
        \subfigure[dp-easy, $N=M=2$]{\includegraphics[scale=0.4]{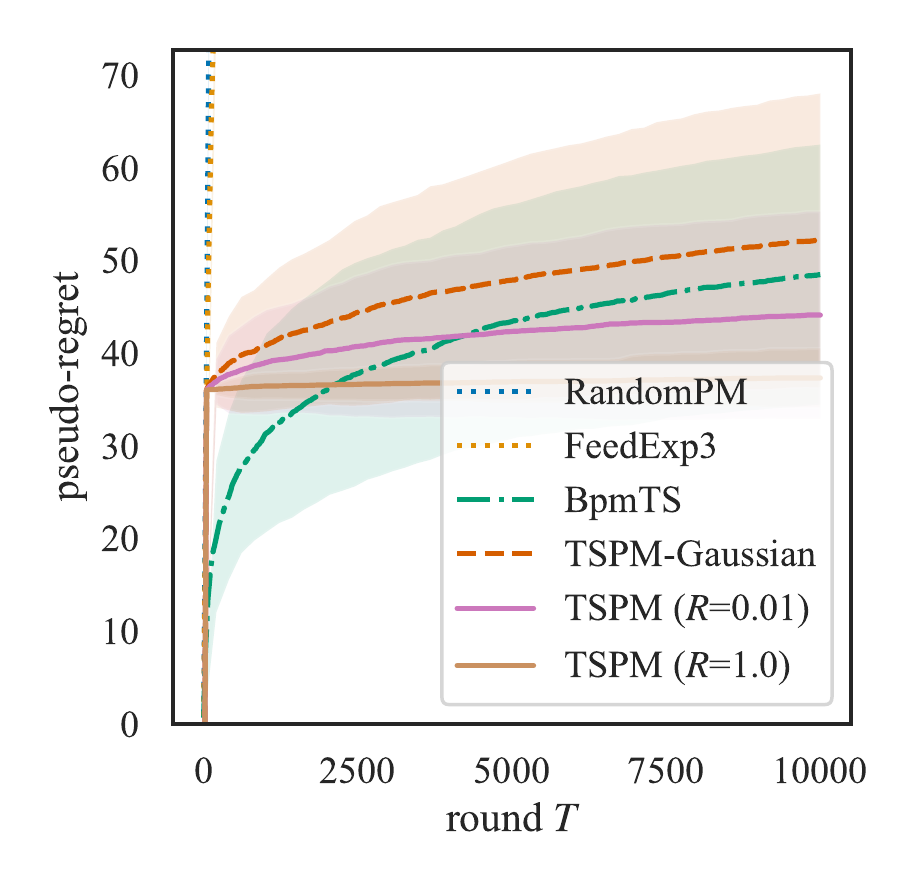}
        }
        \vspace{-0.501em}
    \end{minipage}\hfill
    \begin{minipage}[t]{\subfigwidth}
        \centering
        \subfigure[dp-easy, $N=M=3$]{\includegraphics[scale=0.4]{fig/dp-easy_na_3_no_3.pdf}
        }
        \vspace{-0.501em}
    \end{minipage}\hfill
    \begin{minipage}[t]{\subfigwidth}
        \centering
        \subfigure[dp-easy, $N=M=4$]{\includegraphics[scale=0.4]{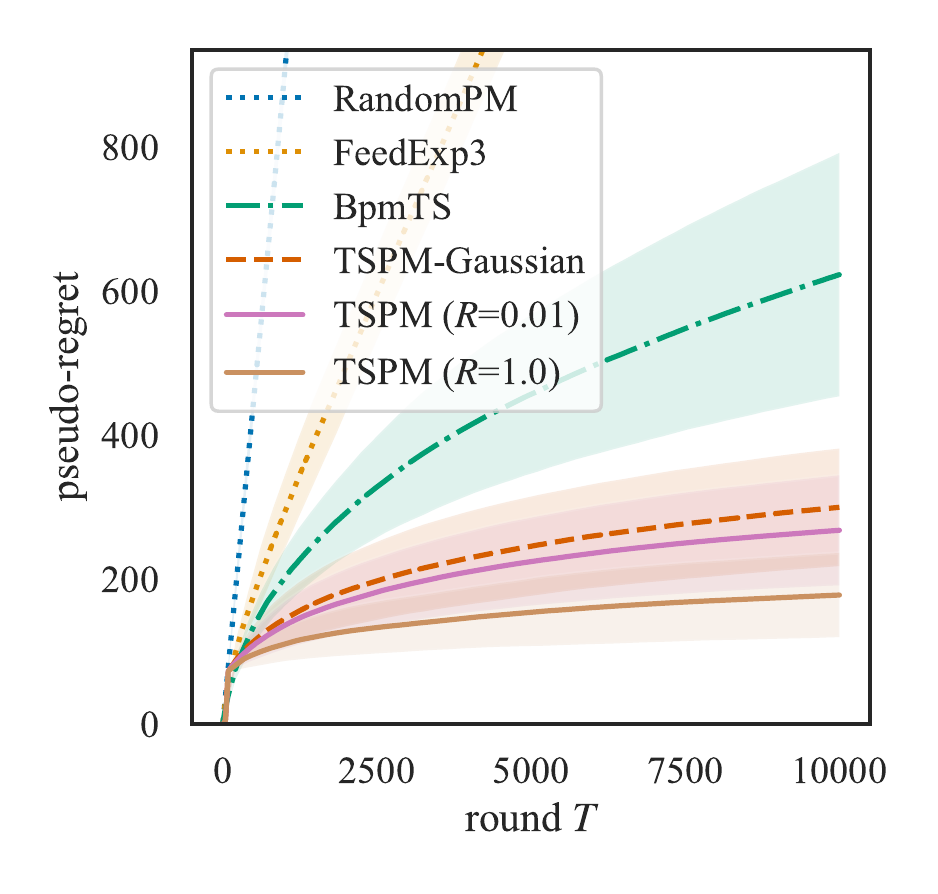}
        }
        \vspace{-0.501em}
    \end{minipage}\hfill
    \begin{minipage}[t]{\subfigwidth}
        \centering
        \subfigure[dp-easy, $N=M=5$]{\includegraphics[scale=0.4]{fig/dp-easy_na_5_no_5.pdf}
        }
        \vspace{-0.501em}
    \end{minipage}\hfill
    \begin{minipage}[t]{\subfigwidth}
        \centering
        \subfigure[dp-easy, $N=M=6$]{\includegraphics[scale=0.4]{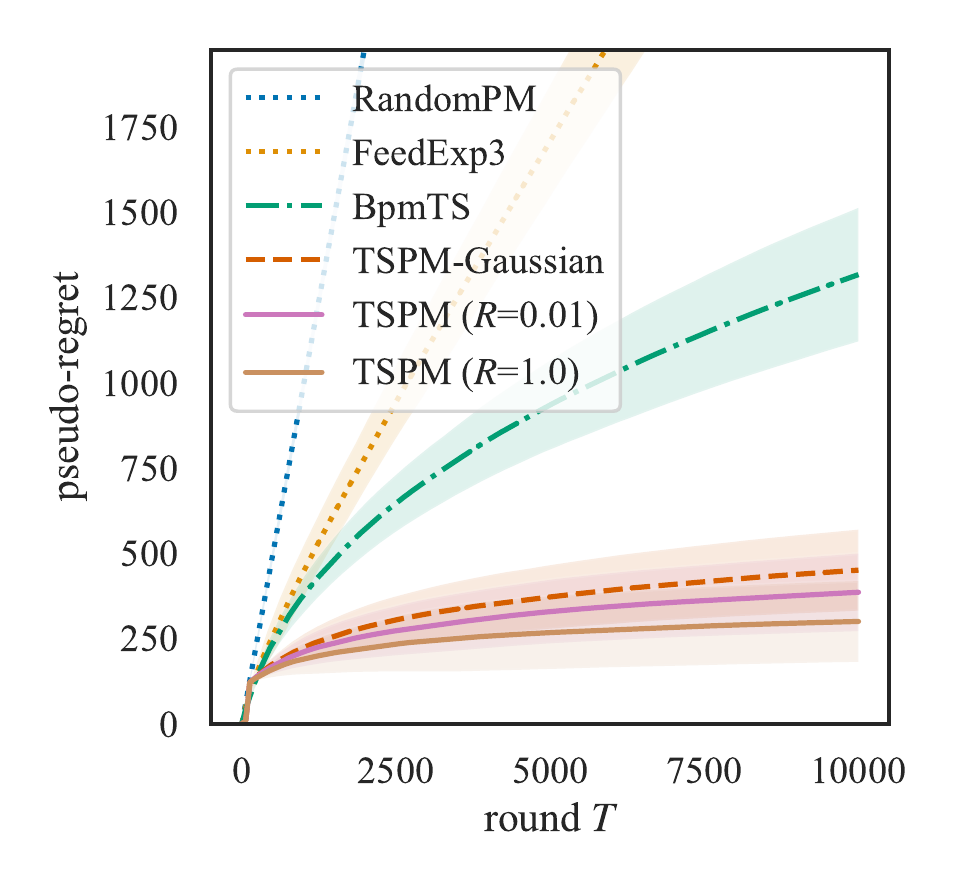}
        }
        \vspace{-0.501em}
    \end{minipage}
    \begin{minipage}[t]{\subfigwidth}
        \centering
        \subfigure[dp-easy, $N=M=7$]{\includegraphics[scale=0.4]{fig/dp-easy_na_7_no_7.pdf}
        }
        \vspace{-0.501em}
    \end{minipage}
    \begin{minipage}[t]{\subfigwidth}
        \centering
        \subfigure[dp-hard, $N=M=2$]{\includegraphics[scale=0.4]{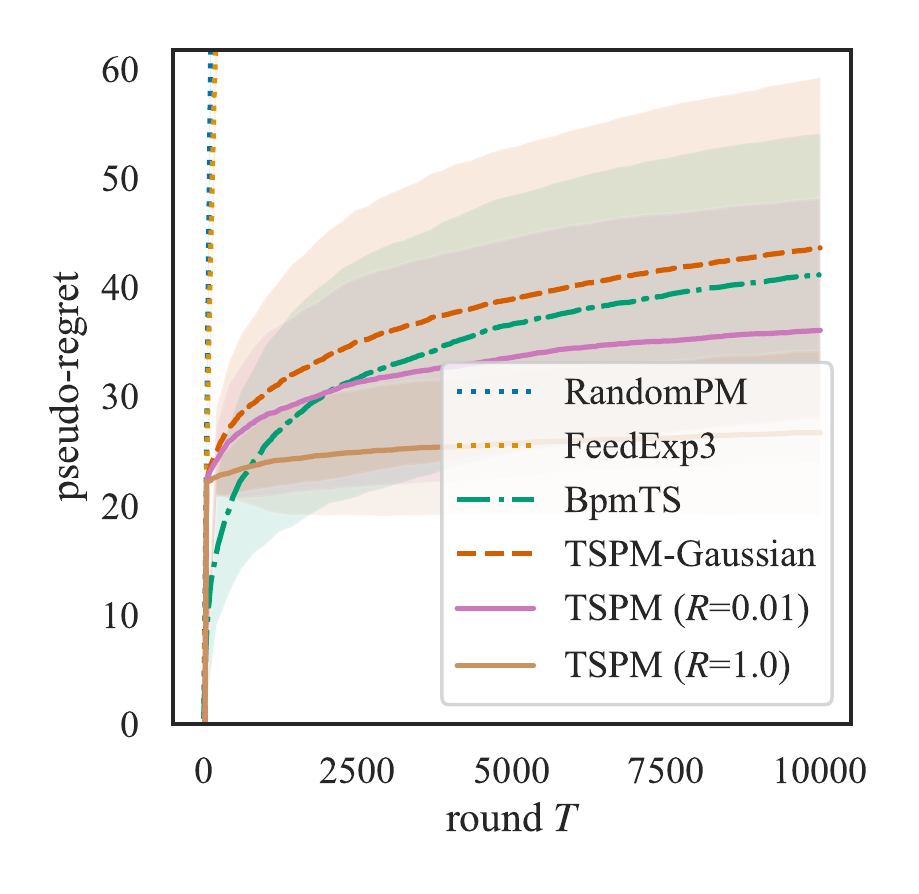}
        }
        \vspace{-0.501em}
    \end{minipage}\hfill
    \begin{minipage}[t]{\subfigwidth}
        \centering
        \subfigure[dp-hard, $N=M=3$]{\includegraphics[scale=0.4]{fig/dp-hard_na_3_no_3.pdf}
        }
        \vspace{-0.501em}
    \end{minipage}\hfill
    \begin{minipage}[t]{\subfigwidth}
        \centering
        \subfigure[dp-hard, $N=M=4$]{\includegraphics[scale=0.4]{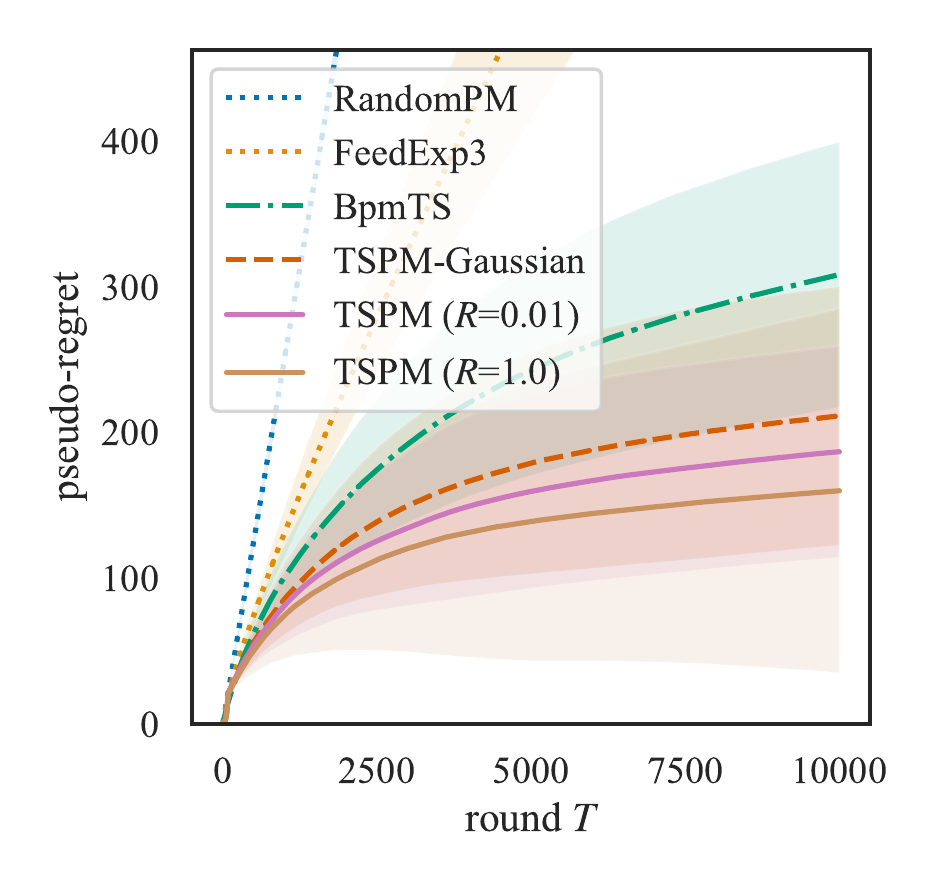}
        }
        \vspace{-0.501em}
    \end{minipage}\hfill
    \begin{minipage}[t]{\subfigwidth}
        \centering
        \subfigure[dp-hard, $N=M=5$]{\includegraphics[scale=0.4]{fig/dp-hard_na_5_no_5.pdf}
        }
        \vspace{-0.501em}
    \end{minipage}\hfill
    \begin{minipage}[t]{\subfigwidth}
        \centering
        \subfigure[dp-hard, $N=M=6$]{\includegraphics[scale=0.4]{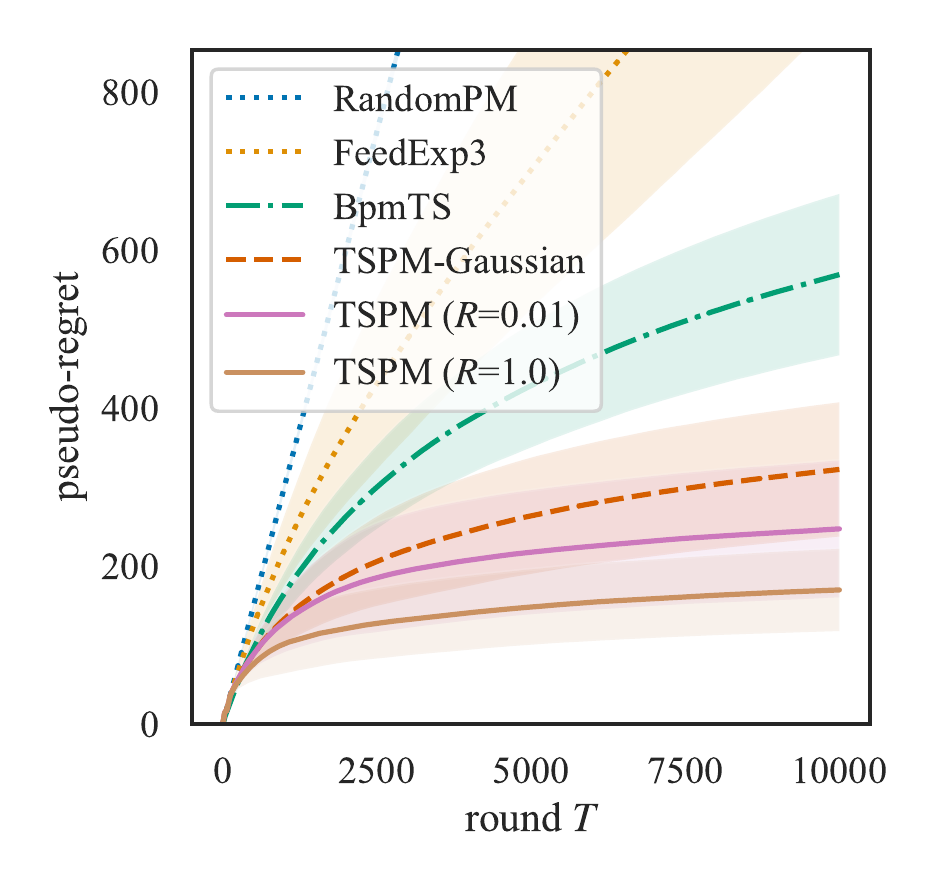}
        }
        \vspace{-0.501em}
    \end{minipage}
    \begin{minipage}[t]{\subfigwidth}
        \centering
        \subfigure[dp-hard, $N=M=7$]{\includegraphics[scale=0.4]{fig/dp-hard_na_7_no_7.pdf}
        }
        \vspace{-0.501em}
    \end{minipage}
    \end{center}
    \caption{
        Regret-round plots of the algorithms.
        The solid lines indicate the average over $100$ independent trials.
        The thin fillings are the standard error.
    }\label{fig:result_regret_all}
\end{figure}%

\setlength\abovecaptionskip{0.2pt}
\begin{figure}[t!]
    \begin{center}
    \setlength{\subfigwidth}{.32\linewidth}
    \addtolength{\subfigwidth}{-.32\subfigcolsep}
    \begin{minipage}[t]{\subfigwidth}
        \centering
        \subfigure[dp-easy, $N=M=2$]{\includegraphics[scale=0.4]{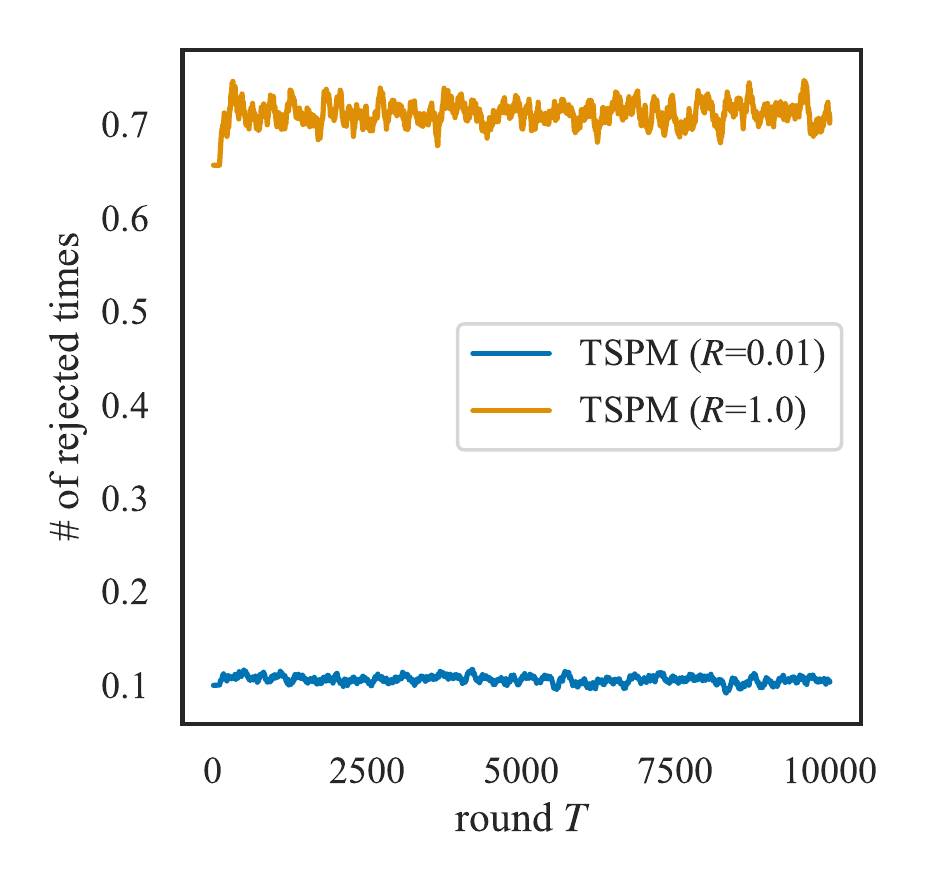}
        }
        \vspace{-0.501em}
    \end{minipage}\hfill
    \begin{minipage}[t]{\subfigwidth}
        \centering
        \subfigure[dp-easy, $N=M=3$]{\includegraphics[scale=0.4]{fig/n_rejected_dp-easy_na_3_no_3.pdf}
        }
        \vspace{-0.501em}
    \end{minipage}\hfill
    \begin{minipage}[t]{\subfigwidth}
        \centering
        \subfigure[dp-easy, $N=M=4$]{\includegraphics[scale=0.4]{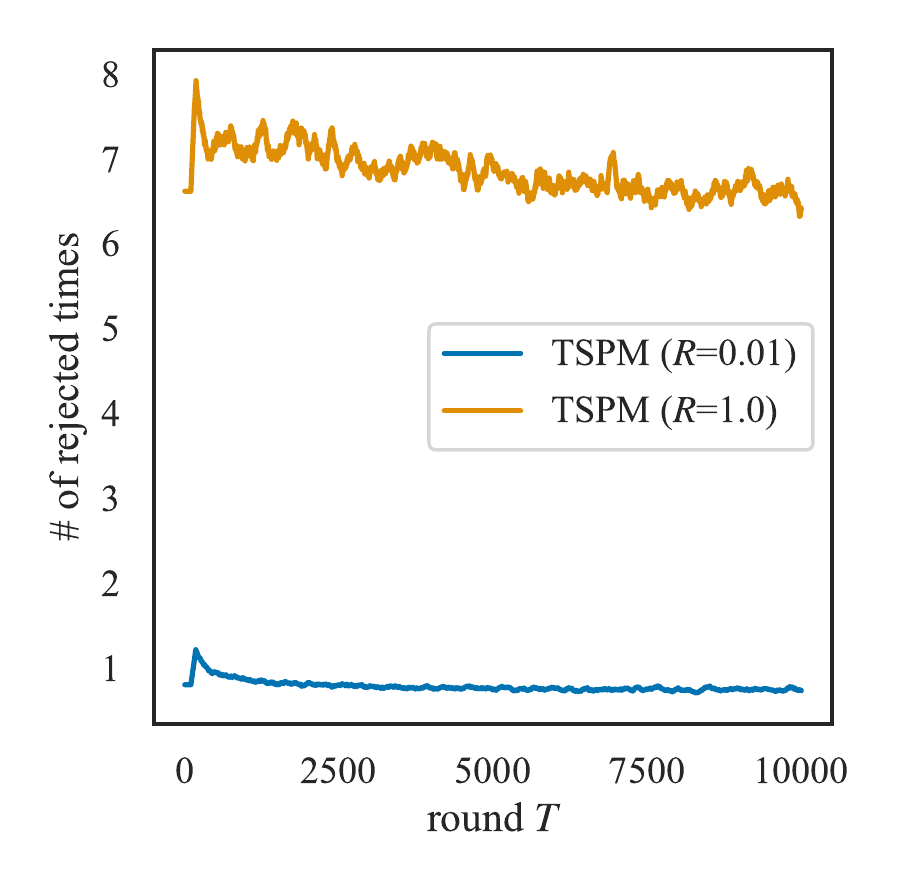}
        }
        \vspace{-0.501em}
    \end{minipage}\hfill
    \begin{minipage}[t]{\subfigwidth}
        \centering
        \subfigure[dp-easy, $N=M=5$]{\includegraphics[scale=0.4]{fig/n_rejected_dp-easy_na_5_no_5.pdf}
        }
        \vspace{-0.501em}
    \end{minipage}\hfill
    \begin{minipage}[t]{\subfigwidth}
        \centering
        \subfigure[dp-easy, $N=M=6$]{\includegraphics[scale=0.4]{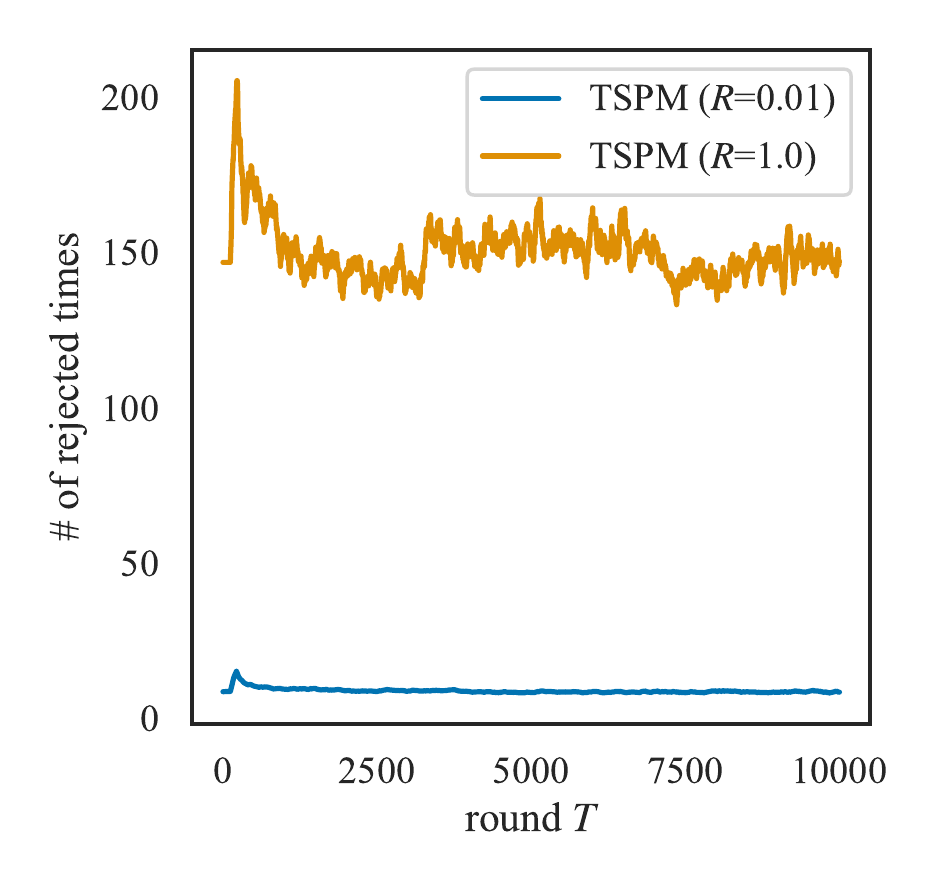}
        }
        \vspace{-0.501em}
    \end{minipage}\hfill
    \begin{minipage}[t]{\subfigwidth}
        \centering
        \subfigure[dp-easy, $N=M=7$]{\includegraphics[scale=0.4]{fig/n_rejected_dp-easy_na_7_no_7.pdf}
        }
        \vspace{-0.501em}
    \end{minipage}
    \begin{minipage}[t]{\subfigwidth}
        \centering
        \subfigure[dp-hard, $N=M=2$]{\includegraphics[scale=0.4]{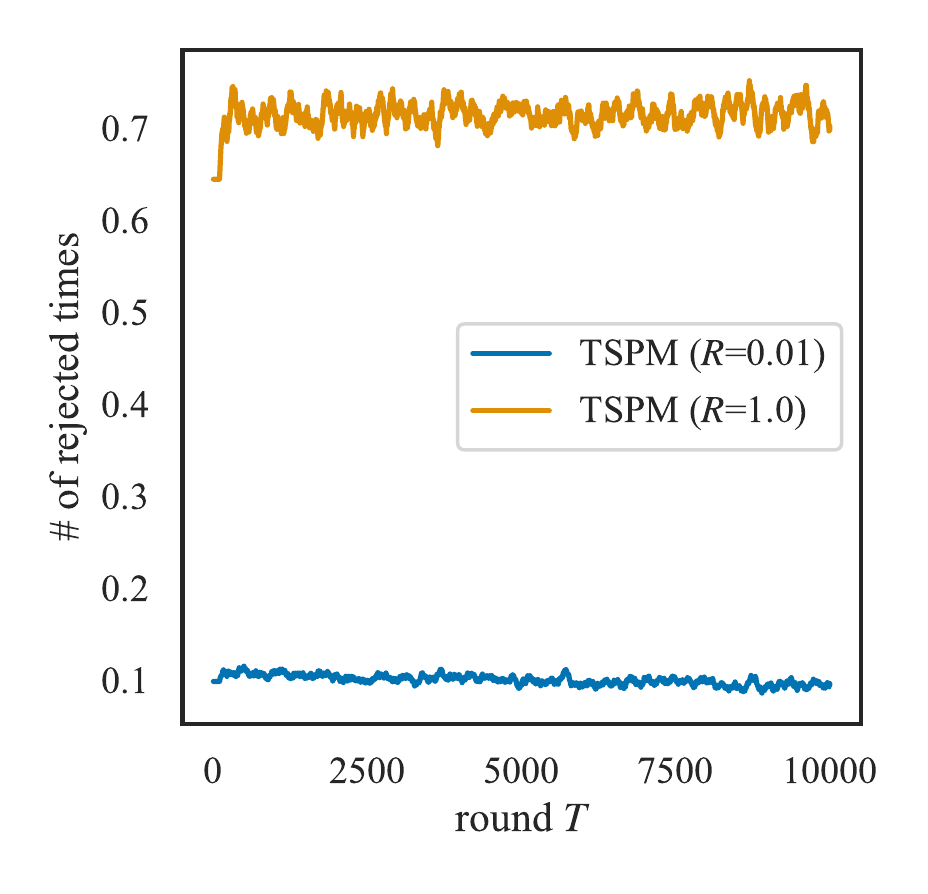}
        }
        \vspace{-0.501em}
    \end{minipage}\hfill
    \begin{minipage}[t]{\subfigwidth}
        \centering
        \subfigure[dp-hard, $N=M=3$]{\includegraphics[scale=0.4]{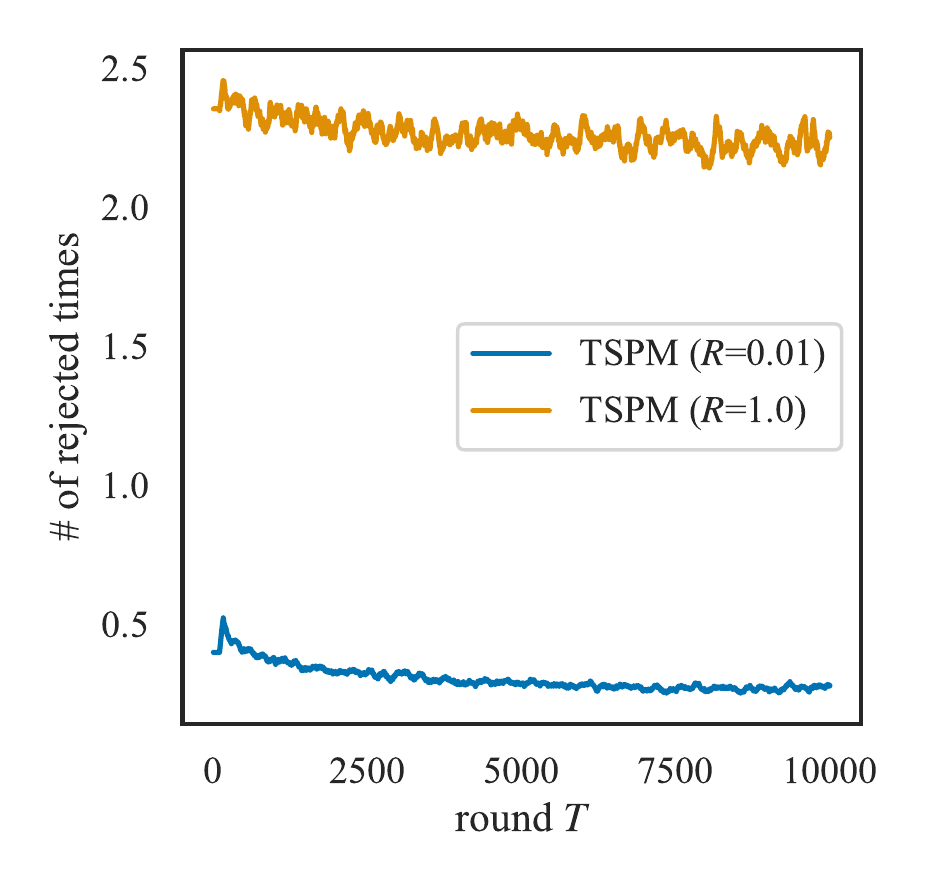}
        }
        \vspace{-0.501em}
    \end{minipage}\hfill
    \begin{minipage}[t]{\subfigwidth}
        \centering
        \subfigure[dp-hard, $N=M=4$]{\includegraphics[scale=0.4]{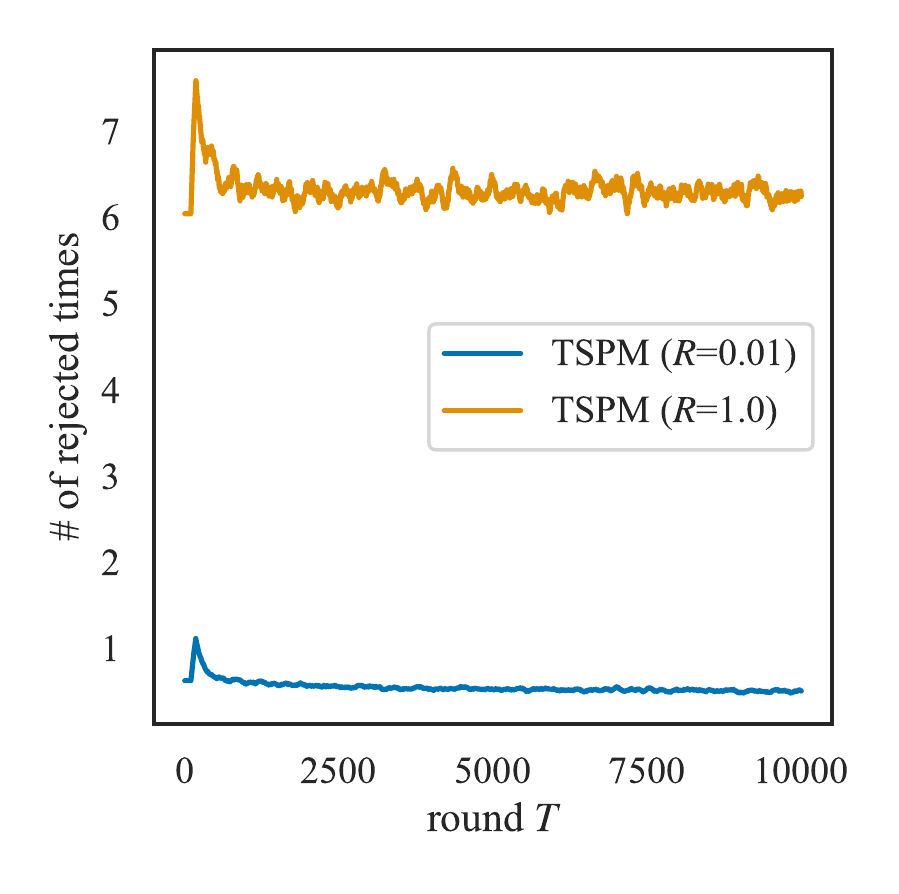}
        }
        \vspace{-0.501em}
    \end{minipage}\hfill
    \begin{minipage}[t]{\subfigwidth}
        \centering
        \subfigure[dp-hard, $N=M=5$]{\includegraphics[scale=0.4]{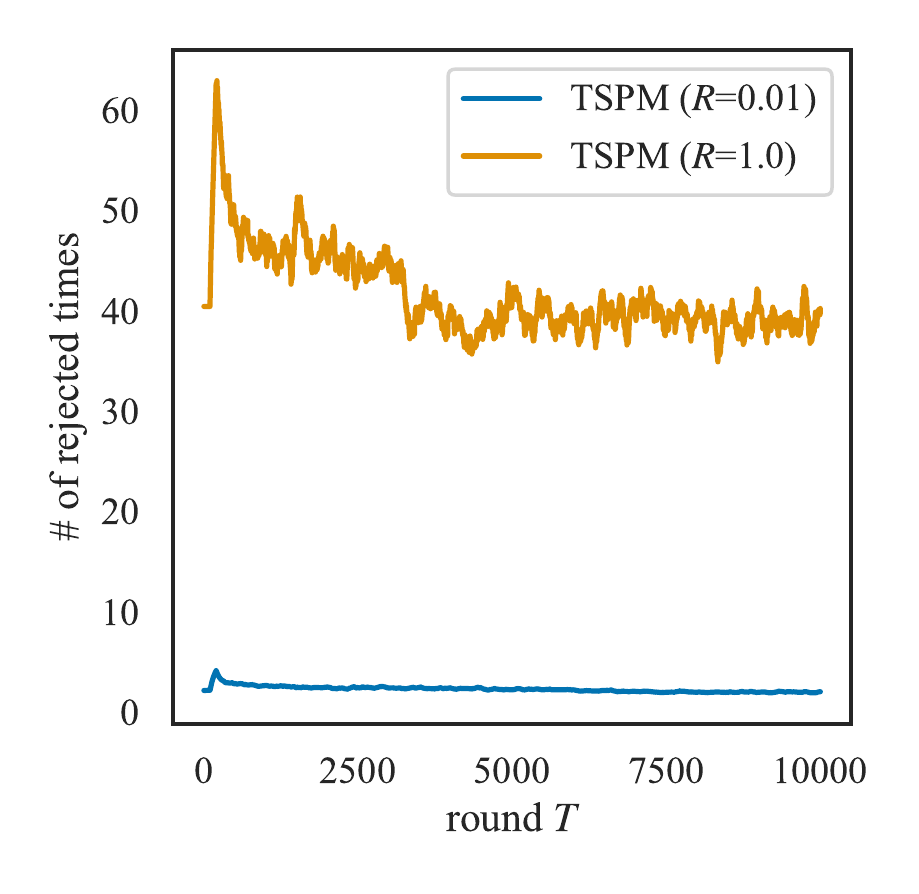}
        }
        \vspace{-0.501em}
    \end{minipage}\hfill
    \begin{minipage}[t]{\subfigwidth}
        \centering
        \subfigure[dp-hard, $N=M=6$]{\includegraphics[scale=0.4]{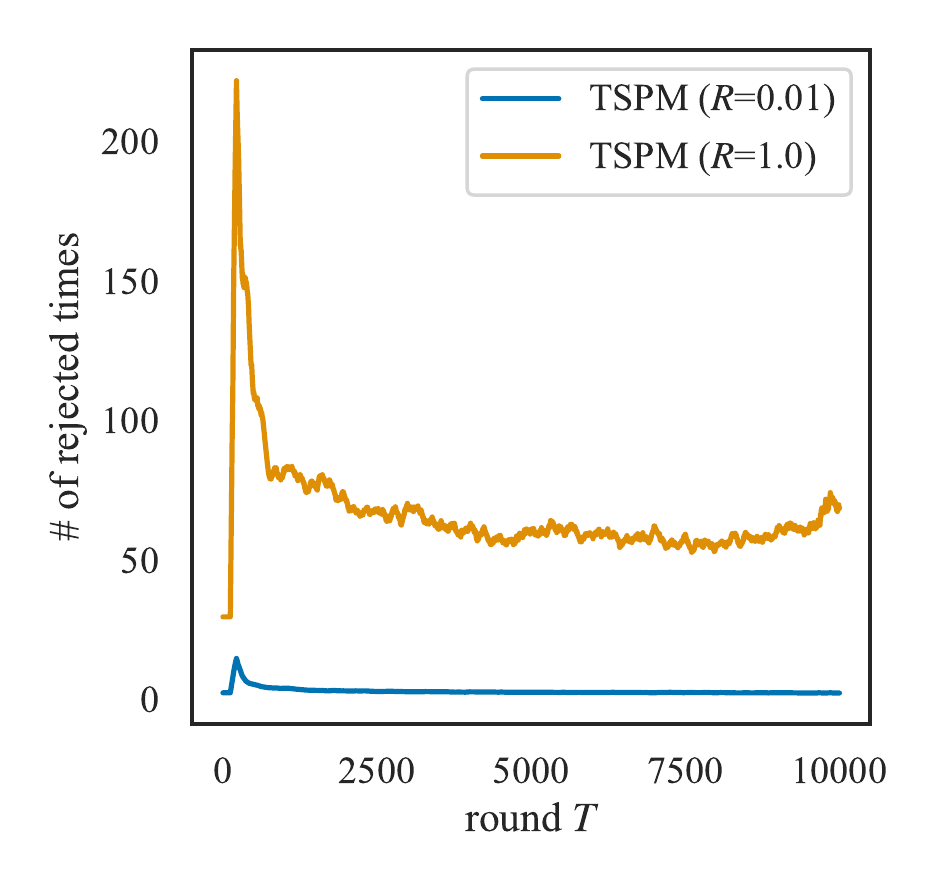}
        }
        \vspace{-0.501em}
    \end{minipage}\hfill
    \begin{minipage}[t]{\subfigwidth}
        \centering
        \subfigure[dp-hard, $N=M=7$]{\includegraphics[scale=0.4]{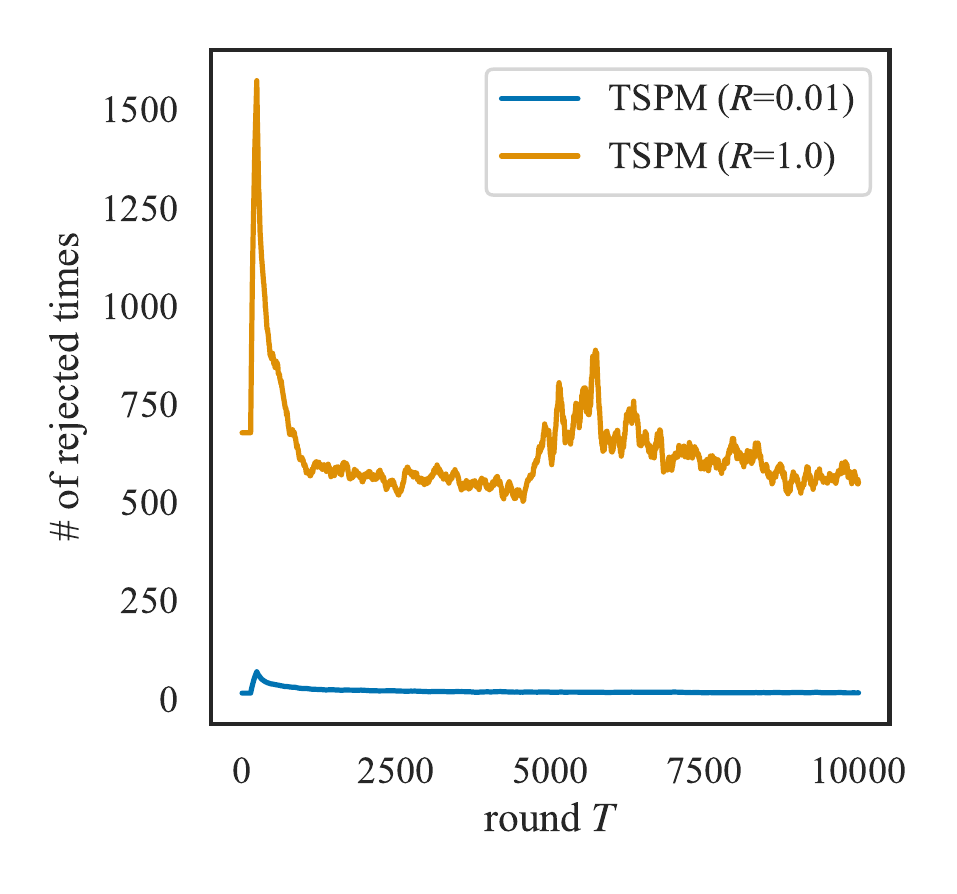}
        }
        \vspace{-0.501em}
    \end{minipage}\hfill
    \end{center}
    \caption{
        The number of rejected times by the accept-reject sampling.
        The solid lines indicate the average over $100$ independent trials after taking moving average with window size $100$.
    }\label{fig:result_n_rejected_all}
\end{figure}%